\documentclass{article}

\usepackage[numbers]{natbib}
\usepackage[preprint]{neurips_2025}
\usepackage[utf8]{inputenc} 
\usepackage{hyperref}       
\usepackage{url}            
\usepackage{booktabs}       
\usepackage{amsfonts}       
\usepackage{nicefrac}       
\usepackage{microtype}      
\usepackage{xcolor}         
\usepackage{tikz}
\usepackage{pgfplots}
\usepackage{graphicx}
\usepackage{subcaption}
\usepackage{multirow}
\usepackage{booktabs}
\usepackage{amsmath, amssymb, amsthm}
\newif\ifarxiv
\arxivtrue

\newcommand{\appref}[2]{\ifarxiv Appendix~\ref{#1}\else Appendix~#2\fi}

\newtheorem{theorem}{Theorem}[section]
\newtheorem{proposition}[theorem]{Proposition}
\newtheorem{lemma}[theorem]{Lemma}
\newtheorem{corollary}[theorem]{Corollary}
\newtheorem{definition}[theorem]{Definition}
\newtheorem{assumption}[theorem]{Assumption}

\pgfplotsset{compat=1.18}

\usetikzlibrary{calc,positioning,arrows.meta,decorations.pathreplacing,intersections,angles,quotes}

\usepackage{algorithm}
\usepackage{algpseudocode}

\title{Fundamental Limitations of Favorable Privacy–Utility Guarantees for DP-SGD}

\author{%
  Murat Bilgehan Ertan\thanks{Affiliated with Vrije Universiteit Amsterdam.} \\
  CWI Amsterdam \\
  Amsterdam, Netherlands \\
  \texttt{bilgehan.ertan@cwi.nl} \\
  \And
  Marten van Dijk\footnotemark[\value{footnote}] \\
  CWI Amsterdam \\
  Amsterdam, Netherlands \\
  \texttt{mevd@cwi.nl} \\
}

\begin{document}
\maketitle

\begin{abstract}
Differentially Private Stochastic Gradient Descent (DP-SGD) is the dominant paradigm for private training, but its fundamental limitations under worst-case adversarial privacy definitions remain poorly understood. We analyze DP-SGD in the $f$-differential privacy framework, which characterizes privacy via hypothesis-testing trade-off curves, and study shuffled sampling over a single epoch with $M$ gradient updates. We derive an explicit suboptimal upper bound on the achievable trade-off curve. This result induces a geometric lower bound on the separation $\kappa$, which is the maximum distance between the mechanism's trade-off curve and the ideal random-guessing line. Because a large separation implies significant adversarial advantage, meaningful privacy requires small $\kappa$. However, we prove that enforcing a small separation imposes a strict lower bound on the Gaussian noise multiplier $\sigma$, which directly limits the achievable utility. In particular, under the standard worst-case adversarial model, shuffled DP-SGD must satisfy
\[
  \textstyle
  \sigma \ge \frac{1}{\sqrt{2\ln M}}
  \quad\text{or}\quad
  \kappa \ge \frac{1}{\sqrt{8}}\!\left(1-\frac{1}{\sqrt{4\pi\ln M}}\right),
\]
thus cannot simultaneously achieve strong privacy and high utility. Although the noise threshold vanishes asymptotically as $M \to \infty$, the convergence is extremely slow. Even for practically relevant numbers of updates the required noise magnitude remains substantial. We further show that the same limitation extends to Poisson subsampling up to constant factors. Our experiments confirm that the noise levels implied by this bound lead to significant accuracy degradation at realistic training settings, thus showing a bottleneck in DP-SGD under standard worst-case adversarial assumptions.
\end{abstract}

\maketitle

\section{Introduction}
The study of differentially private stochastic gradient descent (DP-SGD)~\cite{DBLP:conf/ccs/AbadiCGMMT016} has long been motivated by the need to reconcile data privacy with the accuracy of deep neural networks. However, despite significant theoretical and empirical progress, the tension between privacy and utility remains unresolved. In particular, the degradation of training stability and test accuracy due to injected DP noise fundamentally conflicts with the requirement to add sufficient noise for rigorous differential privacy guarantees.

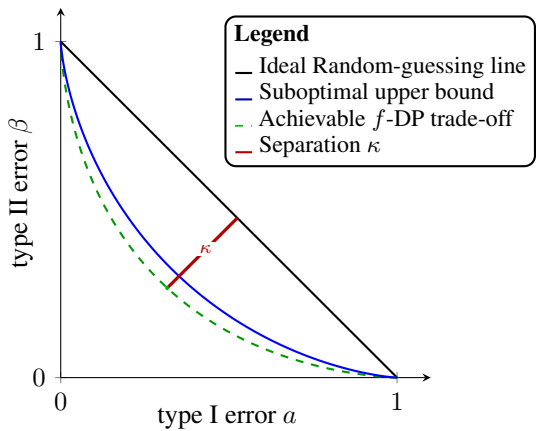
\begin{figure}[tb]
\centering
\begin{tikzpicture}

  \pgfmathdeclarefunction{Phi}{1}{%
  \pgfmathsetmacro{\z}{#1/1.4142135623730951}%
  \pgfmathsetmacro{\pcoef}{0.3275911}%
  \pgfmathsetmacro{\ttmp}{1/(1+\pcoef*abs(\z))}%
  \pgfmathsetmacro{\cOne}{0.254829592}%
  \pgfmathsetmacro{\cTwo}{-0.284496736}%
  \pgfmathsetmacro{\cThree}{1.421413741}%
  \pgfmathsetmacro{\cFour}{-1.453152027}%
  \pgfmathsetmacro{\cFive}{1.061405429}%
  \pgfmathsetmacro{\poly}{((((\cFive*\ttmp+\cFour)*\ttmp+\cThree)*\ttmp+\cTwo)*\ttmp+\cOne)*\ttmp}%
  \pgfmathsetmacro{\erfapprox}{1 - \poly*exp(-\z*\z)}%
  \pgfmathsetmacro{\PhiA}{0.5*(1 + ( \z>=0 ? \erfapprox : -\erfapprox ))}%
  \pgfmathparse{\PhiA}%
}
\pgfmathdeclarefunction{gdpX}{2}{\pgfmathparse{1 - Phi(#1)}}
\pgfmathdeclarefunction{gdpY}{2}{\pgfmathparse{   Phi(#1 - #2)}}
  \pgfmathsetmacro{\muU}{1.10} 
  \pgfmathsetmacro{\muL}{0.90} 
  \pgfmathsetmacro{\tstar}{0.80} %
  \pgfmathsetmacro{\pU}{3.00} 
  \pgfmathsetmacro{\pL}{3.60} 

\pgfmathsetmacro{\astar}{ gdpX(\tstar, \muL) }  %

  \pgfmathsetmacro{\yP}{    gdpY(\tstar, \muL) }  %

  \def\shiftc{0.06}

  \pgfmathsetmacro{\xQ}{(1 + \astar - \yP)/2}
  \pgfmathsetmacro{\yQ}{1 - \xQ}

\pgfmathsetmacro{\Px}{gdpX(\tstar, \muL)}
\pgfmathsetmacro{\Py}{gdpY(\tstar, \muL)}

  \pgfmathsetmacro{\Qx}{(1 + \Px - \Py)/2}
  \pgfmathsetmacro{\Qy}{1 - \Qx}
  
  \begin{axis}[
      width=8.5cm,height=6.5cm,
      xmin=0,xmax=1.10,
      ymin=0,ymax=1.10,
      axis lines=left,
      xlabel={type I error $a$},
      ylabel={type II error $\beta$},
      xlabel style={at={(axis description cs:0.45,-0.05)},anchor=north},
      ylabel style={at={(axis description cs:-0.05,0.5)},anchor=south},
      xtick={0,1},
      ytick={0,1},
      domain=0:1,
      samples=400,
      clip=false,
      axis equal image=true,
    ]

    \addplot[name path=guess,thick,black] {1 - x};

\addplot[
  name path=lower,
  thick, green!60!black, dashed,
  domain=-4:4, samples=400
] ({gdpX(x,\muU)}, {gdpY(x,\muU)});

\addplot[
  name path=upper,
  thick, blue,
  domain=-4:4, samples=400
] ({gdpX(x,\muL)}, {gdpY(x,\muL)});

    \path[name path=diag] (axis cs:0,0) -- (axis cs:1.05,1.05);
    \def\shift{0.05}
    \path[name path=diagKappa] (axis cs:0,-\shift) -- (axis cs:1.05,1.05-\shift);

    \path[name intersections={of=upper and diag, by=Ustar}];
    \path[name intersections={of=lower and diagKappa, by=Lstar}];

    \coordinate (LstarQ) at (axis cs:
        {(1 + \pgfkeysvalueof{/pgfplots/xmin} + 0.18)/2},
        {1 - (1 + \pgfkeysvalueof{/pgfplots/xmin} + 0.18)/2});

    \path[name path=kappaline] (Lstar) -- ++(1,1);
    \path[name intersections={of=kappaline and guess, by=Qkappa}];

    \draw[very thick,red!70!black] (Lstar) -- (Qkappa);
    \fill[green!70!black] (Lstar) circle (0.006);
    \node[red!70!black,fill=white,inner sep=1pt] at ($ (Lstar)!0.55!(Qkappa) $)
      {\scriptsize $\kappa$};
    \node[anchor=north east,
          align=left,
          draw=black,
          fill=white,
          rounded corners,
          thick,
          font=\footnotesize,
          xshift=-1mm, yshift=-1mm]
          at (rel axis cs:1,1) {%
      \textbf{Legend}\\[2pt]
      \tikz[baseline={(current bounding box.center)}]
          {\draw[black!70!black, thick] (0,0) -- (0.05,0);}~Ideal Random-guessing line\\
      \tikz[baseline={(current bounding box.center)}]
          {\draw[blue!70!black, thick] (0,0) -- (0.05,0);}~Suboptimal upper bound\\
      \tikz[baseline={(current bounding box.center)}]
          {\draw[green!70!black, dashed, line width=0.6pt] (0,0) -- (0.05,0);}~Achievable $f$-DP trade-off\\
      \tikz[baseline={(current bounding box.center)}]
          {\draw[red!70!black, thick] (0,0) -- (0.05,0);}~Separation $\kappa$
    };

\end{axis}
\end{tikzpicture}
\caption{
Trade-off view of privacy in the $f$-DP framework~\cite{DBLP:journals/corr/abs-1905-02383}. The black line shows the ideal random-guessing trade-off between type~I and type~II errors. The red segment~$\kappa$ denotes the \emph{maximum} distance between the achievable $f$-DP trade-off and the ideal limit.
}
\label{fig:early-landscape}

\end{figure}

Existing analyses of DP-SGD have primarily relied on the \emph{Poisson subsampling} model~\ifarxiv\cite{DBLP:conf/ccs/AbadiCGMMT016,DBLP:conf/nips/BirrellEBP24,pmlr-v97-zhu19c,DBLP:journals/corr/abs-1908-10530}\fi, in which each record is independently included in a mini-batch with a fixed probability. This assumption, introduced by Abadi et al.~\cite{DBLP:conf/ccs/AbadiCGMMT016}, enables privacy accounting techniques like Rényi Differential Privacy~\cite{DBLP:conf/csfw/Mironov17} and privacy-loss distributions for numerically tight composition~\cite{DBLP:conf/aistats/KoskelaJH20,DBLP:conf/ccs/0001M18}. However, it is primarily a modeling convenience. In practice, modern deep learning systems do not sample examples independently but instead shuffle the entire dataset and partition it into fixed-size mini-batches each epoch. This shuffling-based batching is far more efficient computationally but breaks the independence assumptions exploited in Poisson analyses, making theoretical guarantees harder to derive. As a result, most practical DP-SGD implementations train using shuffled batches while reporting privacy parameters as if Poisson subsampling were used~\cite{DBLP:journals/jair/PonomarevaHKXDMVCT23,DBLP:conf/icml/ChuaGK0MSZ24,DBLP:conf/nips/ChuaGK0MSZ24}. This persistent mismatch between analytical assumptions and real implementations motivates our study of \emph{random shuffling} as a more realistic sampling mechanism for understanding DP-SGD’s privacy behavior.

Building on these observations, our analysis is carried out entirely within the $f$-DP framework~\cite{DBLP:journals/corr/abs-1905-02383}, which provides a complete hypothesis-testing description of privacy and allows us to reason directly about the trade-off curves induced by shuffled DP-SGD.  We focus on the single-epoch shuffling regime under the standard \emph{worst-case adversarial model}, in which each data record appears in exactly one batch per epoch. In the standard worst-case adversarial model, the adversary may use the observable noisy batch updates together with any auxiliary information (such as batch size) available to infer the contribution of the differing data record; the formal model is given in Section~\ref{sec:adversarial-model}. This setting eliminates the independence structure that underpins Poisson-based analyses and introduces fundamental dependencies in the adversary’s observation model. We explicitly formalize this \emph{adversarial observation model}, which is typically left implicit in prior work, and analyze its implications through the geometry of the $f$-DP trade-off curve. Finally, while our primary results concern shuffled DP-SGD, we show via a rigorous mixture argument that the same limitations extend, up to constant factors, to Poisson subsampling. We additionally derive an asymptotic separation bound for Poisson subsampling within the $\mu$-GDP framework~\cite{DBLP:journals/corr/abs-1905-02383} which relies on asymptotic convergence without an explicit rate. Together, these results yield a \emph{unified worst-case picture} covering both the theoretical (Poisson) and practical (shuffling) sampling schemes, and demonstrate that our findings propagate to the privacy parameters reported in standard DP-SGD practice.

Before presenting our main claim, we briefly recall how privacy is parameterized in DP-SGD.  In the classical $(\varepsilon,\delta)$-DP formulation, a randomized algorithm is private when its output distributions on two neighboring datasets are nearly
indistinguishable to any adversary; smaller values of $\varepsilon$ therefore correspond to stronger privacy guarantees.  In DP-SGD, this effect is achieved by two mechanisms: each per-example gradient is clipped to an $\ell_2$ radius $C$, which bounds its maximum possible influence, and Gaussian noise with standard deviation $C\sigma$ is added to every batch update.  The clipping constant $C$ controls sensitivity, while the noise multiplier $\sigma>0$ is the dominant driver of the resulting privacy parameters $(\varepsilon,\delta)$.  Smaller $\sigma$ improves test accuracy but necessarily weakens privacy, whereas larger $\sigma$ strengthens privacy at the expense of utility. 

In the $f$-DP framework~\cite{DBLP:journals/corr/abs-1905-02383}, an adversary attempting to distinguish two neighboring datasets forms a test with \emph{false positive rate} (FP)~$\alpha$ and \emph{false negative rate} (FN)~$\beta$.  The resulting \emph{trade-off function} $f(\alpha)$ gives the minimum achievable FN error among all tests with FP error at most $\alpha$.  Thus, $f$ captures the fundamental FP--FN trade-off faced by the strongest possible adversary.  The ideal privacy baseline is the random-guessing line $\beta = 1-\alpha$, corresponding to zero leakage. Section~\ref{subsec:fdp} provides further background on the $f$-DP formalism.

A key geometric quantity in this picture is the \emph{separation} between the trade-off curve and the random-guessing line (perfect privacy). We define
\[
\operatorname{sep}(f)
\;=\;
\max_{\alpha\in[0,1]}\;
\min_{\gamma\in[0,1]}
\bigl\|\,(\alpha,f(\alpha))-(\gamma,1-\gamma)\,\bigr\|_2,
\]
that is, the \emph{maximum} distance between a point on the trade-off curve $(\alpha,f(\alpha))$ and the ideal line $(\gamma,1-\gamma)$. When $f$ is symmetric and convex, this separation is attained at the fixed point $\hat{a}$ satisfying $f(\hat{a})=\hat{a}$, and we can write $\kappa \;=\; \operatorname{sep}(f) \;=\; (1-2\hat{a})/{\sqrt{2}}.$
Smaller values of $\operatorname{sep}(f)$ (and hence smaller~$\kappa$) correspond to stronger privacy, as the entire trade-off curve lies closer to random guessing. This privacy interpretation will be formalized and analyzed rigorously later.

This notion of separation follows the $f$-DP interpretation introduced by Dong et~al.~\cite{DBLP:journals/corr/abs-1905-02383}, applied to deep learning by Bu et~al.~\cite{DBLP:journals/corr/abs-1911-11607}, and later formalized by Zhu et~al.~\cite{DBLP:conf/IEEEares/ZhuTPDC25}. Figure~\ref{fig:early-landscape} provides an early geometric view of this hypothesis-testing landscape for shuffled DP-SGD. The vertical segment~$\kappa$ illustrates the separation $\operatorname{sep}(f)$ for the corresponding trade-off curve. In the remainder of the paper, we show that under the standard worst-case adversarial model, \emph{our explicit upper bound on the shuffled trade-off curve implies a lower bound on $ \kappa = \operatorname{sep}(f) $}.
Combining these elements yields our main conceptual consequence, an unavoidable \emph{or} constraint for one-epoch shuffled DP\text{-}SGD with $M$ rounds. Under the standard worst-case adversarial model, the mechanism must satisfy:
\[
\sigma \;\ge\; \frac{1}{\sqrt{2\ln M}}
\qquad\textbf{or}\qquad
\kappa \;\ge\; \frac{1}{\sqrt{8}}\!\left(1 - \frac{1}{\sqrt{4\pi\ln M}}\right)
\]
Both cannot be made small simultaneously. A large noise multiplier $\sigma$ hurts accuracy, while a large separation $\kappa$ indicates substantial privacy leakage. While the quantity $1/\sqrt{\ln M}$ vanishes asymptotically as $M\to\infty$, this decay is extremely slow. Even for astronomically large values of $M$, the required noise level remains far from zero. Our lower bound thus translates into a privacy–utility tension. In the shuffled setting and under the standard worst-case adversary of the DP framework, DP-SGD cannot operate in a regime where both privacy and utility are simultaneously strong. As our analysis confirms, this tension is not an artifact of shuffling; the same lower bound applies to Poisson subsampling (up to a constant factor), suggesting that the bottleneck originates in the worst-case adversarial framework, which both sampling schemes share. We complement our theoretical results with experiments illustrating the practical severity of this trade-off. In particular, we show that the noise levels implied by our lower bound already lead to substantial degradation in test accuracy under realistic training budgets. These empirical results confirm that the theoretical limitations we identify also manifest concretely at finite parameter values encountered in practice.
In classical terms, this threshold has an immediate $(\varepsilon,\delta)$-DP consequence.
Whenever $\sigma < 1/\sqrt{2\ln M}$, any $(\varepsilon,\delta)$-DP guarantee reported for the
same mechanism must satisfy
$\varepsilon \ge \ln\bigl((1+\kappa\sqrt2-2\delta)/(1-\kappa\sqrt2)\bigr)$ for the
corresponding separation lower bound $\kappa$ (Corollary~\ref{cor:eps-delta-main}). Already
at the one-epoch baseline with $\delta = 1/N$ and $N = 10^8$ this forces
$\varepsilon \gtrsim 1$, and standard privacy accountants report $\varepsilon \gg 10$ at
realistic multi-epoch budgets (Section~\ref{sec:experiments}).

Our result is a \emph{converse}. Accounting-based analyses fix $\sigma$ and ask how
private the mechanism is. We ask how small $\sigma$ can be if worst-case privacy is to
remain nontrivial. Numerical privacy-loss lower bounds for shuffled DP-SGD are known at
specific parameter
instances~\cite{DBLP:conf/icml/ChuaGK0MSZ24,DBLP:conf/nips/ChuaGK0MSZ24}, but to our
knowledge none gives a \emph{closed-form} bound on the minimum admissible $\sigma$ as a
function of~$M$. We emphasize that the limitation we prove is not inherent to DP-SGD as an algorithm, but rather to the adversarial assumption framework under which its privacy is analyzed. In particular, the standard definition implicitly grants the adversary full visibility into every component of the stochastic updates, making certain forms of privacy leakage unavoidable. Our main contributions are as follows:

\begin{itemize}
    \item We analyze single-epoch shuffled DP-SGD under the standard worst-case adversarial model in the $f$-DP framework and derive an explicit suboptimal upper bound on its trade-off curve.
    \item This yields a geometric lower bound on the separation~$\kappa$, exposing a fundamental privacy--utility trade-off: the noise multiplier~$\sigma$ and separation cannot be driven below explicit lower bounds at the same time. Then, via a mixture argument, we show that Poisson subsampling inherits the same limitation up to a constant factor, providing a unified worst-case characterization of both sampling schemes.
    \item We translate our $f$-DP bounds into consequences for classical
    $(\varepsilon,\delta)$-DP, obtaining a closed-form lower bound on the minimal
    admissible~$\varepsilon$ for fixed~$\delta$ (Corollary~\ref{cor:eps-delta-main}),
    and verify the implication independently with the numerically tight shuffled
    privacy accountant of Chua et al.~\cite{DBLP:conf/nips/ChuaGK0MSZ24}
    (Section~\ref{sec:experiments}).
    \item Empirically, we validate the predicted threshold behavior, showing that the noise levels implied by our bounds already cause substantial accuracy degradation across models, datasets, and batch sizes.
\end{itemize}

\paragraph{Organization.}
The remainder of the paper is structured as follows. Section~\ref{sec:related} reviews the most relevant prior work. Section~\ref{sec:preliminaries} introduces the differential privacy framework, the sampling mechanisms used in practice, and the BatchSGD update model. Section~\ref{sec:adversarial-model} formalizes the worst-case adversarial observation model that underlies our analysis. In Section~\ref{sec:hypothesis-testing}, we introduce a suboptimal hypothesis testing framework for DP-SGD and define the separation metric $\kappa$ formally. Next, Section~\ref{sec:limitations} develops the geometric tools needed to analyze this metric, establishes our main separation lower bound for shuffled DP-SGD, and extends the result to Poisson subsampling via a mixture argument and yielding a unified worst-case picture for both sampling schemes. Section~\ref{sec:experiments} presents empirical illustrations of our theoretical findings, and Sections~\ref{sec:discussion} and~\ref{sec:conclusion} conclude with broader implications and future directions.

\section{Related Work}
\label{sec:related}

Differential Privacy (DP)~\ifarxiv\cite{DBLP:conf/ccs/AbadiCGMMT016, DBLP:conf/tcc/DworkMNS06, DBLP:journals/fttcs/DworkR14, DBLP:conf/eurocrypt/DworkKMMN06, DBLP:journals/cacm/Dwork11}\fi has emerged as a principled framework for preserving individual information in data analysis and model training. Among its applications, Differentially Private Stochastic Gradient Descent (DP-SGD)~\cite{DBLP:conf/ccs/AbadiCGMMT016} is the most widely used privacy-preserving method in deep learning, injecting calibrated noise into gradient updates to satisfy a target privacy budget. It has also been studied for various machine learning domains~\ifarxiv\cite{10.1145/2660267.2660348,tang2025privatefinetuninglargelanguage,de2022unlockinghighaccuracydifferentiallyprivate,chen2021gswgangradientsanitizedapproachlearning,dockhorn2023differentiallyprivatediffusionmodels,anil2021largescaledifferentiallyprivatebert,he2022exploringlimitsdifferentiallyprivate,igamberdiev2024dpnmtscalabledifferentiallyprivatemachine,yue2023synthetictextgenerationdifferential,mckenna2025scalinglawsdifferentiallyprivate}\fi, and adopted in real-world systems from the U.S. Census~\cite{10.1145/3219819.3226070} to Microsoft telemetry~\cite{ding2017collectingtelemetrydataprivately} and Google’s recent initiative to train large language models from scratch under DP~\cite{sinha2025vaultgemmadifferentiallyprivategemma}. It is now supported by \ifarxiv major privacy libraries such as TensorFlow Privacy, JAX Privacy, and Opacus~\cite{tensorflow2015-whitepaper,jax-privacy2022github,opacus}\fi, which provide practical implementations and privacy accounting tools for large-scale training.

Analytically, the privacy guarantees of DP-SGD have been refined through composition frameworks such as Rényi Differential Privacy (RDP)~\cite{DBLP:conf/csfw/Mironov17}  and Privacy-Loss Distributions (PLD)~\ifarxiv\cite{DBLP:journals/popets/SommerMM19,DBLP:conf/aistats/KoskelaJH20,DBLP:conf/ccs/0001M18}\fi, which enable tighter accounting of cumulative privacy loss across training iterations. The classical $(\varepsilon,\delta)$-DP~\cite{DBLP:conf/tcc/DworkMNS06,DBLP:conf/eurocrypt/DworkKMMN06} formulation summarizes privacy using two scalar parameters, whereas the emerging $f$-DP framework~\ifarxiv\cite{DBLP:journals/corr/abs-1905-02383,DBLP:journals/corr/abs-1911-11607,DBLP:conf/nips/WangS0SS23}\fi    describes privacy through the entire \emph{trade-off curve} of optimal type~I/type~II errors, providing a complete and geometrically interpretable characterization of indistinguishability. Importantly, the trade-off curve is the most informative DP representation. As shown in \cite{DBLP:journals/corr/abs-1905-02383} (Appendix~A and~B), all standard DP notions, including $(\varepsilon,\delta)$-DP and divergence-based relaxations, can be \emph{derived} from it. Every $f$-DP trade-off curve induces a corresponding family of $(\varepsilon,\delta)$-DP guarantees, and conversely each $(\varepsilon,\delta)$-DP guarantee imposes a linear constraint on the trade-off curve~\cite{DBLP:journals/corr/abs-1905-02383}. This relationship allows techniques developed for $(\varepsilon,\delta)$-DP to be imported into the $f$-DP framework. This motivates our decision to conduct all analysis directly in the $f$-DP framework.

Privacy amplification by shuffling, introduced by Erlingsson et al.~\cite{DBLP:conf/soda/ErlingssonFMRTT19} and later refined by Feldman et al.~\cite{DBLP:conf/focs/FeldmanMT21, DBLP:conf/soda/FeldmanMT23}, showed that randomly permuting locally privatized reports can strengthen privacy guarantees in the central model. However, this amplification applies to locally randomized mechanisms rather than to DP-SGD’s centralized setting.

The most closely related analysis of shuffled DP-SGD is due to Chua et
al.~\cite{DBLP:conf/icml/ChuaGK0MSZ24,DBLP:conf/nips/ChuaGK0MSZ24}, who study the mechanism
under their \emph{Adaptive Batch Linear Queries} (ABLQ) framework. For the shuffled
sampler, their results provide \emph{lower bounds on the privacy loss} $\varepsilon$ for
single- and multi-epoch training. In the $f$-DP viewpoint, these correspond to
\emph{upper bounds} on the achievable trade-off function, the same direction as our
analysis, in which an explicit upper bound on the trade-off curve yields a lower bound on
the separation~$\kappa$. Importantly, they also show that random shuffling in DP-SGD leads
to \emph{weaker provable privacy guarantees} than Poisson subsampling, highlighting a
discrepancy with the common practice of reporting Poisson-based parameters for shuffled
implementations~\cite{DBLP:journals/jair/PonomarevaHKXDMVCT23}. Our analysis agrees in
direction, and differs in form and scope. Their lower bounds are given by analytic expressions
that must be evaluated, and inverted, numerically for each parameter setting, whereas our
max-statistic construction yields a simple \emph{closed-form} upper bound on the shuffled
trade-off curve and, in turn, an explicit threshold $\sigma \ge 1/\sqrt{2\ln M}$ as a
function of $M$ alone. Moreover, the separation formulation applies uniformly over all
$f$-DP families, transfers to Poisson subsampling via an explicit mixture argument
(Section~\ref{subsec:poisson-mixture}), and translates to closed-form
$(\varepsilon,\delta)$ consequences (Corollary~\ref{cor:eps-delta-main}).

Concurrently, a growing line of work proposes allocation-style samplers that retain the
implementability of shuffling while recovering Poisson-like accounting. Examples include the Balls-and-Bins
sampler~\cite{DBLP:conf/aistats/ChuaGHKKLMSZ25} and the $k$-out-of-$t$ allocation scheme, in which each record is assigned to $k$ of the $t$ rounds, introduced as balanced iteration subsampling~\cite{DBLP:conf/icml/DongCO25} (random allocation~\cite{DBLP:conf/nips/ShenfeldF25}), with analytic guarantees and efficient privacy-loss-distribution accounting~\cite{DBLP:journals/corr/abs-2602-17284} and a two-limit optimality result~\cite{lessrandommoreprivate}. These are achievability
analyses, and they are complementary to ours. At $k=1$ these samplers place each record in
exactly one uniformly random round, so the tightly dominating pair identified for the
Balls-and-Bins sampler in~\cite[Thm.~3.1]{DBLP:conf/aistats/ChuaGHKKLMSZ25} coincides, after rescaling
by $\sigma$, with the hypothesis pair we analyze
in~\eqref{eq:likelihood-h0}--\eqref{eq:likelihood-h1}, so the dichotomy of
Theorem~\ref{thm:shuffling} applies to them unchanged under the same assumptions
 (which does not imply that their trade-off curves coincide with $f_{\mathrm{shuf}}$).

The random-shuffling regime remains the standard sampling strategy in modern deep learning practice. It has been extensively studied from an optimization perspective, where several works analyze its convergence behavior and compare it to sampling with replacement~\ifarxiv\cite{DBLP:journals/jmlr/NguyenTPND21,DBLP:conf/icml/HaochenS19,DBLP:journals/ijon/MengCWML19,DBLP:journals/mp/GurbuzbalabanOP21}\fi.
 Its widespread adoption instead reflects training conventions and practical considerations. Poisson subsampling does not scale well to large datasets because random access is inefficient for data that do not fit in memory, and the resulting variable batch sizes are inconvenient to handle in standard deep learning frameworks~\cite{DBLP:conf/icml/ChuaGK0MSZ24}. As a result, shuffling-based batching remains the default in modern training systems, even though the corresponding privacy analyses continue to rely on Poisson-based assumptions~\cite{DBLP:journals/jair/PonomarevaHKXDMVCT23}.

Apart from privacy accounting and sampling considerations, it is now well established that training deep neural networks with DP-SGD often incurs a substantial loss in utility, with prior work consistently reporting degradation in accuracy and generalization under meaningful privacy guarantees~\ifarxiv\cite{DBLP:journals/corr/abs-2012-07828,de2022unlockinghighaccuracydifferentiallyprivate,DBLP:journals/corr/abs-2201-12328,DBLP:journals/corr/abs-2502-17772,DBLP:journals/corr/abs-2105-07985,DBLP:journals/corr/abs-2111-13895,DBLP:conf/nips/BagdasaryanPS19}\fi. A common insight across these works is that the injected Gaussian noise, controlled by the noise multiplier~$\sigma$, is one of the dominant sources of this degradation, acting as a barrier to effective training.

\section{Preliminaries}
\label{sec:preliminaries}
Differential Privacy (DP) guarantees that a randomized mechanism produces
statistically similar outputs on any two neighboring datasets differing in a
single individual’s record.  Within stochastic gradient descent (SGD), the DP-SGD algorithm~\cite{DBLP:conf/ccs/AbadiCGMMT016} enforces this requirement by first bounding the effect of each example via gradient clipping, and then injecting Gaussian noise to obscure the contribution of any individual.

DP-SGD modifies the standard mini-batch update in two ways:  
(i) each per-sample gradient is clipped to a fixed threshold, bounding sensitivity, and  
(ii) Gaussian noise scaled to this threshold is added to the aggregated gradient before updating the model parameters. The precise mechanism is formalized later in Section~\ref{subsec:batchsgd}. Moreover, the privacy of DP-SGD depends critically on how mini-batches are sampled. We therefore introduce the two canonical sampling mechanisms, namely \emph{Poisson Subsampling} and \emph{Random Shuffling} in the following subsection.

\paragraph{Scope of analysis.}
We analyze the privacy guarantees of a \emph{single epoch} of DP-SGD under the standard worst-case adversarial model (see Section~\ref{sec:adversarial-model}). In this setting the sampler produces index sets $S_1,\ldots,S_M$, each used exactly once to form a mini-batch. Extensions to multiple epochs require additional composition and are beyond the scope of our core lower bound.

\subsection{Sampling Mechanisms in DP-SGD}
\label{sec:sampling}

The sampling mechanism plays a central role in the privacy analysis of DP-SGD because subsampling provably amplifies differential privacy guarantees. When each update is computed on a random mini-batch, a given record influences the mechanism only in the rounds where it is sampled, thereby reducing the likelihood of leaking information about that individual~\cite{DBLP:conf/nips/BalleBG18}. While various sampling strategies have been proposed~\cite{DBLP:journals/jair/PonomarevaHKXDMVCT23}, two have become canonical in both theory and practice: \emph{Poisson subsampling} and \emph{random shuffling}. The former is analytically convenient and therefore ubiquitous in theoretical work; the latter is the standard choice in practical deep learning systems. Both determine how mini-batches are formed from a dataset $d=\{\xi_1,\ldots,\xi_N\}$, but they exhibit fundamentally different statistical dependencies and therefore lead to different privacy-amplification behaviors. We formalize both mechanisms below, with pseudocode in \appref{app:implementation}{B}.

\begin{definition}[Poisson Subsampling]
\label{def:poisson-subsampling}
Given a dataset $d$ of size $N$, Poisson subsampling includes each element independently with probability $q\in(0,1)$ at each round $j=1,\dots,M$.  
This produces random index sets $S_j\subseteq[N]$ with expected size $\mathbb{E}[|S_j|]=qN$, and independence is maintained both across records and across rounds.
\end{definition}

\begin{definition}[Random Shuffling]
\label{def:random-shuffling}
Given a dataset $d$ of size $N$ and a fixed number of rounds $M$, random shuffling samples a uniform permutation $\pi$ of the dataset indices and then forms $M$ disjoint batches of equal size $b=\lfloor N/M\rfloor$ by taking consecutive blocks of $\pi$. Concretely, for each $j=1,\ldots,M$, we have $S_j \;=\; \{\pi_{(j-1)b+1},\ldots,\pi_{jb}\}$. So each record appears in at most one batch per epoch, and the last $N - Mb = N \bmod M$ permuted indices are discarded.
\end{definition}

In the remainder of this work, we focus on the random-shuffling mechanism, as it captures practical DP-SGD implementations and introduces inter-round dependencies that are absent under Poisson subsampling. As we later show, our main limitation results extend to Poisson subsampling up to constant factors, yielding a unified worst-case characterization for both sampling schemes.

\subsection{DP-SGD Mechanism}
\label{subsec:batchsgd}

We formalize the update rule used throughout this work.  
Let $d=\{\xi_1,\ldots,\xi_N\}$ denote the training dataset of size $N$, and let
$S_1,\ldots,S_M\subseteq[N]$ be index sets produced by a sampling mechanism
(Poisson subsampling or random shuffling).  
The corresponding mini-batch at iteration $j$ is $B_j=\{\xi_i : i\in S_j\}$.

Given a loss function $\ell(w;\xi)$, the empirical risk is
\[
L(w)=\frac{1}{N}\sum_{i=1}^N \ell(w;\xi_i),
\]
and the per-sample gradient at iteration $j$ is
\[
g(w_j;\xi_i)=\nabla_w \ell(w_j;\xi_i).
\]
\begin{definition}[Gradient Clipping]
\label{def:clipping}
Given a clipping threshold $C>0$, the clipped version of a per-sample gradient
$g(w;\xi)$ is
\[
[g(w;\xi)]_C
=
g(w;\xi)\cdot
\min\!\left(1,\frac{C}{\|g(w;\xi)\|_2}\right).
\]
Thus $\|[g(w;\xi)]_C\|_2 \le C$ for all $\xi$, ensuring bounded sensitivity of
the aggregated update.
\end{definition}

\begin{definition}[Gaussian Noise Mechanism]
\label{def:gaussian-noise}
Given a clipping threshold $C$ and noise multiplier $\sigma\ge 0$, the DP-SGD
mechanism adds Gaussian noise
\[
\mathbf{z}
\sim
\mathcal{N}\!\bigl(0,(C\sigma)^2 I\bigr),
\]
where $I$ is the identity matrix (i.e., independent Gaussian noise per
coordinate).  
The standard deviation $C\sigma$ scales proportionally with both the clipping
threshold and the noise multiplier.
\end{definition}

At each iteration $j$, the mechanism processes the mini-batch $S_j$ by
computing clipped per-sample gradients, adding Gaussian noise as in
Definition~\ref{def:gaussian-noise}, averaging by the batch size, and then
performing an SGD update step. The Algorithm~\ref{alg:batchsgd} summarizes the procedural implementation, including batch generation. For $\sigma = 0$ and $C = \infty$, this procedure reduces to standard mini-batch SGD; for $\sigma > 0$ with finite $C$, it corresponds to the DP-SGD update of~\cite{DBLP:conf/ccs/AbadiCGMMT016}.

\renewcommand{\thealgorithm}{$\mathcal{B}$}
\begin{algorithm}[tb]
\caption{DP-SGD Mechanism~\cite{DBLP:conf/ccs/AbadiCGMMT016}}
\label{alg:batchsgd}
\begin{algorithmic}[1]
\State \textbf{Input:} Dataset $d$, number of rounds $M$, learning rate $\eta$, clipping bound $C$, noise multiplier $\sigma$ (optional)
\For{each epoch}
    \State Generate mini-batches $\{S_1, \ldots, S_M\}$ using either Poisson subsampling (Def.~\ref{def:poisson-subsampling}) or random shuffling (Def.~\ref{def:random-shuffling})
    \For{each $S_j$}
        \State Compute per-sample gradients $g(w_j;\xi_i)$ for all $i \in S_j$
        \State Clip each gradient:
        \[
        [g(w_j; \xi_i)]_C = g(w_j; \xi_i)\cdot \min\Bigl(1, \frac{C}{\|g(w_j; \xi_i)\|_2}\Bigr)
        \]
        \State Compute clipped batch sum:
        \[
        G_j = \sum_{i \in S_j} [g(w_j; \xi_i)]_C
        \]
        \State Add Gaussian noise and average:
        \[
        \mathbf{z}_j \sim \mathcal{N}\!\bigl(0,(C\sigma)^2 I\bigr), \qquad
        \tilde{G}_j = \frac{1}{|S_j|}\Bigl(G_j + \mathbf{z}_j\Bigr)
        \]
        \State Update model: $w_{j+1} \leftarrow w_j - \eta\,\tilde{G}_j$
    \EndFor
\EndFor
\end{algorithmic}
\end{algorithm}

\subsection{Differential Privacy and \texorpdfstring{$f$-DP}{f-DP}}
\label{subsec:fdp}

We briefly recall the classical $(\varepsilon,\delta)$ formulation of
differential privacy and its refinement through the $f$-DP framework, both of
which rely on a notion of adjacency between datasets.

Two datasets that differ only in the contribution of a single individual are
called \emph{adjacent}, written $d \sim d'$.  This notion can be
formalized in several ways, depending on how one models the presence,
replacement, or removal of a record.  We will introduce the standard adjacency
relations shortly, but for any such choice, differential privacy is defined as follows.

\begin{definition}[Differential Privacy (DP)~\cite{DBLP:conf/eurocrypt/DworkKMMN06}]
\label{def:dp}
A randomized mechanism $\mathcal{M}$ is \emph{$(\varepsilon,\delta)$-DP} with
respect to an adjacency relation~$\sim$ if, for all adjacent datasets
$d \sim d'$ and all measurable events $E$,
\[
\Pr[\mathcal{M}(d)\in E]
\;\le\;
e^{\varepsilon}\,\Pr[\mathcal{M}(d')\in E] + \delta.
\]
\end{definition}

The $f$-DP framework characterizes privacy through the entire \emph{trade-off function} $f(\alpha)$, which captures the best achievable Type~I and Type~II error trade-off when distinguishing $d$ from $d'$. This viewpoint provides a complete geometric representation of privacy, supports tight composition, and matches the Gaussian mechanism central to DP-SGD.  In addition, DP relies on a formal notion of \emph{adjacency} between
datasets. Several adjacency models appear in the literature; we record them
formally below.

\begin{definition}[Substitution Adjacency]
\label{def:substitution-adjacency}
Let $\mathcal{X}$ denote the space of valid data records. Datasets $d$ and $d'$ are
\emph{substitution neighbors} if they have equal cardinality and differ in exactly one
record, i.e., if there exist records $\xi \in d$ and $\xi' \in \mathcal{X}$ with
$\xi' \neq \xi$ such that
\[
d' = \bigl(d \setminus\{\xi\}\bigr) \cup \{\xi'\}.
\]
Thus $|d| = |d'|$, and $\xi'$ takes the place of the differing individual's record $\xi$.
Throughout, $\xi$ and $\xi'$ denote individual data records (elements of $\mathcal{X}$),
whereas mini-batches are denoted $B_j$ (cf.\ Section~\ref{subsec:batchsgd}).
\end{definition}

\begin{definition}[Add/Remove Adjacency]
\label{def:addremove-adjacency}
Datasets $d$ and $d'$ are \emph{add/remove neighbors} if one is obtained from
the other by inserting or deleting a single record. Equivalently, $|d'| = |d| \pm 1$ and
\[
\begin{aligned}
d' &= d \cup \{\xi\} \quad\text{for some } \xi \in \mathcal{X}, \quad\text{or}\\
d' &= d \setminus\{\xi\} \quad\text{for some } \xi \in d.
\end{aligned}
\]
\end{definition}

\begin{definition}[Zero-Out Adjacency~\cite{DBLP:conf/nips/ChuaGK0MSZ24, pmlr-v139-kairouz21b}]
\label{def:zeroout-adjacency}
Let $\mathcal{X}_\perp := \mathcal{X} \cup \{\perp\}$ be the augmented record space, where $\perp$ denotes a ghost record whose contribution to any query or gradient is defined to be zero.

For a dataset $d = (d_1,\ldots,d_N) \in \mathcal{X}_\perp^{\,N}$ and an index $i\in[N]$, let
$d_{-i}$ denote the dataset obtained by removing the $i$-th entry.
Two datasets $d,d' \in \mathcal{X}_\perp^{\,N}$ are said to be \emph{zero-out adjacent} if there exists
an index $i\in[N]$ such that
\[
d_{-i} = d'_{-i},
\qquad
\{d_i,\, d'_i\} = \{\perp,\, \xi\}
\ \text{for some }\xi\in\mathcal{X},
\]
that is, the datasets are identical except at a single position where one
contains a genuine record and the other contains the ghost record~$\perp$.
\end{definition}

When DP-SGD uses Poisson subsampling, each element is included in a batch independently
with probability $q$, so the batch size itself is random. The privacy question is therefore
whether a given record is \emph{present or absent}, which corresponds to
\emph{add/remove adjacency} (Definition~\ref{def:addremove-adjacency}). This is the pairing
adopted in the subsampling-amplification literature~\cite{DBLP:conf/nips/BalleBG18}, and in subsequent $f$-DP and privacy-accounting analyses of DP-SGD~\ifarxiv\cite{DBLP:journals/corr/abs-1911-11607,DBLP:conf/aistats/0005DW22}\fi. Under
add/remove adjacency, the worst-case $\ell_2$-sensitivity of the clipped gradient sum is
at most $C$, since the two neighboring sums differ by a single clipped gradient of norm at most
$C$ (with equality when a record's gradient reaches the clipping threshold), and the
Gaussian mechanism accordingly uses noise of scale $C\sigma$.

In contrast, the original $f$-DP analysis of
DP-SGD~\cite{DBLP:journals/corr/abs-1905-02383}
is carried out under \emph{substitution adjacency}
(Definition~\ref{def:substitution-adjacency}), where the two datasets have equal size and
differ in one record. In this case the worst-case $\ell_2$-sensitivity of the clipped
gradient sum is at most $2C$, the maximum possible distance between two clipped gradients. This
imposes an extra factor of~$2$ on the required Gaussian noise scale in substitution-based
$f$-DP analyses. We stress that sensitivity is a statement about the worst case over
neighboring datasets. It places no requirement on how any individual record's gradient
behaves.

A third notion, \emph{zero-out adjacency}
(Definition~\ref{def:zeroout-adjacency})~\cite{DBLP:conf/nips/ChuaGK0MSZ24, pmlr-v139-kairouz21b},
replaces the differing record with a ghost record $\perp$ whose gradient is zero. The
worst-case sensitivity is then at most $C$ (the distance between a clipped gradient and the
zero vector, attained when the clipping threshold is reached), and, unlike add/remove
adjacency, the total dataset cardinality $N$ is
preserved. The latter property is precisely what our shuffling analysis requires. The
ghost-record coupling keeps the index sets $S_1,\dots,S_M$ identically distributed under
both neighboring datasets, so the two hypotheses differ only in the contribution of a
single position.

For the remainder of this work we therefore analyze DP-SGD under
\textbf{zero-out adjacency}, so that our bounds reflect the tightest possible worst-case
guarantees. We emphasize that zero-out adjacency is an \emph{analysis device}, not an
algorithmic modification. The DP-SGD mechanism (Algorithm~\ref{alg:batchsgd}) is unchanged,
and only the pair of neighboring datasets compared in the privacy analysis changes. All of
our results translate to substitution adjacency at the cost of a factor-of-two increase in
the required noise scale, which would only strengthen our conclusions.

\section{Adversarial Model}
\label{sec:adversarial-model}

A key element in any differential privacy analysis is the choice of adversary and the information they are assumed to observe. In DP-SGD, the adversary’s view is determined by what portion of the stochastic updates (gradients, batch structure, and added noise) are observable. We consider two notions: a \emph{practical}  adversary corresponding to realistic communication constraints, and a \emph{worst-case} adversary required for proving formal differential privacy guarantees.

\subsection{Practical Adversary}
\label{subsec:practical-adversary}

In many distributed or federated training settings, an adversary can observe
only the messages transmitted to the server. These messages are the 
\emph{noisy aggregated gradients} produced by DP-SGD. The following
assumption captures this realistic eavesdropper model.

\begin{assumption}[Practical Adversary]
\label{ass:practical}
In each round $j=1,\ldots,M$, the adversary observes only the
(noisy, averaged) update
\[
\tilde{G}_j
= \frac{1}{|S_j|}
    \left(
        \sum_{i \in S_j} [g(w_j;\xi_i)]_C 
        + \mathcal{N}(0,(C\sigma)^2 I)
    \right),
\]
and has no access to individual gradients or intermediate computations.
\end{assumption}

\subsection{Worst-Case Adversary}
\label{subsec:worst-case-adversary}
To enable standard DP arguments, analyses in the literature traditionally endow the adversary with additional knowledge that is not realistically available but simplifies the analysis under the adjacency models defined in Section~\ref{subsec:fdp}.

We adopt \emph{zero-out adjacency} (Definition~\ref{def:zeroout-adjacency}): neighboring datasets $d$ and $d'$ both have size $N$ and share the first $N-1$ records. They differ only in the $N$-th element. In $d$ it is a ghost record $\perp$ whose gradient is defined to be zero ($[g(w_j; \perp)]_C = 0$), while in $d'$ it is a valid record $\xi_N$.  Then, we give the adversary access to the sum of all clipped gradients \emph{except} the potentially differentiating one (index $N$). Furthermore, we assume the batch size $|S_j|$ is public knowledge (or observable from metadata).

\begin{assumption}[Auxiliary Knowledge Under Zero-out Adjacency]
\label{ass:aux}
Under zero-out adjacency (Definition~\ref{def:zeroout-adjacency}),
for each round $j$ the adversary knows the batch size $|S_j|$ and the
partial sum of clipped gradients excluding the potentially differing record:
\[
G_j^{(-)} 
 = \sum_{i \in S_j \setminus \{N\}} [g(w_j;\xi_i)]_C.
\]
\end{assumption}

Given the observable update $\tilde{G}_j$ from 
Assumption~\ref{ass:practical} and the auxiliary information from 
Assumption~\ref{ass:aux}, the adversary can deterministically recover the noisy
contribution of the differing record.

\begin{proposition}[Isolation of the Individual Contribution]
\label{prop:reconstruction}
Let $\delta_{N\in S_j}$ denote the indicator that $N$ belongs to batch $S_j$. Under Assumptions~\ref{ass:practical} and~\ref{ass:aux}, the adversary can
compute
\[
Z_j := |S_j|\,\tilde{G}_j - G_j^{(-)}.
\]
Consequently, when the dataset is $d$ (ghost record $\perp$ present), the contribution is zero and the adversary reconstructs pure noise:
\begin{equation}
Z_j = \mathcal{N}(0,(C\sigma)^2 I),
\label{eq:diff-term-d}
\end{equation}
and when the dataset is $d'$ (valid record $\xi_N$ present), the reconstruction becomes
\begin{equation}
Z'_j = \delta_{N\in S_j} [g(w_j;\xi_N)]_C 
       + \mathcal{N}(0,(C\sigma)^2 I).
\label{eq:diff-term-d-prime}
\end{equation}
\end{proposition}

\paragraph{Consequence.}
Since the worst-case adversary can mathematically reconstruct these differential
terms from the public observations and the auxiliary knowledge under Assumptions~\ref{ass:practical} and~\ref{ass:aux}, analyzing the
indistinguishability of \eqref{eq:diff-term-d} and \eqref{eq:diff-term-d-prime} is sufficient to bound the privacy leakage of the entire protocol. Hence our analysis may focus entirely on the reconstructed variables $Z_j$ and $Z'_j$ for the remainder of this work. In essence, we take the worst-case adversary in our analysis to be the one that distinguishes $Z_j$ and $Z'_j$, and the present subsection motivates this definition. Other functions of the adversary's view may also carry information about the differing record, and the auxiliary knowledge of Assumption~\ref{ass:aux} already includes the batch cardinality $|S_j|$. Since our results are lower bounds obtained by exhibiting an explicit hypothesis test on $Z_j$ and $Z'_j$, restricting attention to these variables can only understate the adversary, and the resulting bounds remain valid.

\section{Hypothesis Testing and \texorpdfstring{$f$}{f}-DP Characterization}
\label{sec:hypothesis-testing}

Having defined the adversarial view, we now characterize privacy leakage as a statistical
hypothesis testing problem between neighboring datasets $d$ and $d'$. This perspective aligns naturally with the $f$-DP framework and reveals structural limitations of DP-SGD under the standard worst-case adversarial model. Throughout this section and the next we work entirely under random shuffling. The Poisson subsampling case will be derived later, in Section~\ref{subsec:poisson-mixture}.

\subsection{Projection onto a One-Dimensional Statistic}
\label{subsec:projection}

The adversary observes all transmitted model parameters $\{w_j\}$. Under zero-out adjacency, the two datasets differ in the $N$-th record: dataset $d$ contains a ghost record $\perp$ (zero gradient), while dataset $d'$ contains the valid record $\xi_N$. The adversary can compute the clipped gradient of the valid record:
\[
v = [g(w_j;\xi_N)]_C.
\]

Under zero-out adjacency, the DP interpretation requires considering the \emph{worst-case sensitivity} relative to the zero vector. Because the clipped gradient satisfies $\|v\|_2 \le C$, the maximum possible distance between the valid gradient and the ghost gradient is $\|v - 0\|_2 \le C$. Consequently, the noise added to the batch gradient in round $j$ is $\mathcal{N}(0,(C\sigma)^2 I)$ under both datasets. Using terms in Equations~\eqref{eq:diff-term-d} and~\eqref{eq:diff-term-d-prime}, the adversary’s observations in round $j$ are:
\[
Z_j = \mathcal{N}(0,(C\sigma)^2 I)
\qquad\text{(dataset $d$, ghost record),}
\]
\[
Z'_j = \delta_{N\in S_j}\,v + \mathcal{N}(0,(C\sigma)^2 I)
\qquad\text{(dataset $d'$, valid record).}
\]

Because Gaussian noise is rotationally symmetric, all information relevant to distinguishing $d$ (pure noise) from $d'$ (signal plus noise) lies along the direction of the signal vector $v$. Hence, we can project both $Z_j$ and $Z'_j$ onto this direction. Now, invoking the sensitivity-attainment assumption of Theorem~\ref{thm:shuffling} (a record with $\|v\|_2 = C$, the worst case over neighboring datasets) and normalizing by $C$, we obtain a scalar statistic. Conditioned on the round where the record is sampled ($N\in S_j$), the distributions differ by a shift of exactly $1$:
\[
\mathcal{N}(0,\sigma^2)
\quad\text{vs.}\quad
\mathcal{N}(1,\sigma^2).
\]

Finally, dividing by $\sigma$ for convenience produces the canonical form of the
hypothesis test:
\[
\mathcal{N}(0,1)
\quad\text{vs.}\quad
\mathcal{N}(\sigma^{-1},1).
\]
This projection reduces the adversary’s task to a one-dimensional Gaussian hypothesis test, fully capturing the information relevant to determining whether the differentiating record was included in the batch (signal) or replaced by a ghost record (noise). It serves as the basis for the hypothesis-testing formulation presented in the next section.
\subsection{Adversarial Hypothesis Formulation}
\label{subsec:hypothesis-formal}

Let $H_0$ denote the null hypothesis corresponding to dataset $d$ and let $H_1$ denote the alternative corresponding to dataset $d'$.
Over an epoch of $M = N/m$ rounds, the adversary observes:
\[
(x_1, \ldots, x_M)
\sim
\begin{cases}
(\mathcal{N}(0,1))^M, & H_0, \\[3pt]
(\mathcal{N}(u_1\sigma^{-1},1), \ldots, \mathcal{N}(u_M\sigma^{-1},1)), & H_1,
\end{cases}
\]
where $u_j := \delta_{N \in S_j}$ indicates whether the valid sample $\xi_N$ appears in the $j$-th batch. Under $H_0$, no valid record exists, so effectively $u_j = 0$ for all $j$. Under $H_1$, since $M$ divides $N$, the shuffling procedure discards no records (an assumption carried into Theorem~\ref{thm:shuffling}), and the valid record appears in exactly one batch $S_j$, so exactly one $u_j = 1$, with the index $j$ uniformly distributed over $[M]$.

Under the worst-case adversary, the entire sequence $(x_1, \ldots, x_M)$ is observable, so the adversary can compute the full likelihood ratio between $H_0$ and $H_1$.  
The corresponding probability densities are
\begin{align}
P[(x_j)_{j=1}^M | H_0]
  &= \prod_{j=1}^{M} \frac{e^{-x_j^2/2}}{\sqrt{2\pi}}, \label{eq:likelihood-h0}\\[3pt]
P[(x_j)_{j=1}^M | H_1]
  &= \frac{1}{M} \sum_{j=1}^M
     \frac{e^{-(x_j - 1/\sigma)^2/2}}{\sqrt{2\pi}}
     \prod_{\substack{i=1 \\ i \ne j}}^{M} \frac{e^{-x_i^2/2}}{\sqrt{2\pi}} \nonumber\\
  &= \frac{1}{M} \sum_{j=1}^M
     e^{x_j/\sigma - 1/(2\sigma^2)}
     \prod_{i=1}^{M} \frac{e^{-x_i^2/2}}{\sqrt{2\pi}}. \label{eq:likelihood-h1}
\end{align}
\subsection{Application of the Neyman--Pearson Lemma}
\label{sec:np-lemma}
We now recall the classical Neyman--Pearson framework, which provides the optimal hypothesis test distinguishing between datasets $d$ and $d'$ in terms of Type~I and Type~II errors.

Let $\phi$ denote a (possibly randomized) rejection rule for the null hypothesis. Given the observed sequence $(x_j)_{j=1}^M$, the rule $\phi$ outputs the probability in $[0,1]$ with which $H_0$ is rejected, the values $1$ and $0$ being the deterministic case.

The \emph{false positive rate} (Type~I error) of a test $\phi$ is defined as
\[
\alpha(\phi) = \mathbf{E}[\phi \mid d] \in [0,1],
\]
and represents the probability of incorrectly rejecting the null hypothesis when $d$ is true.  
A test of size $\alpha(\phi)$ is said to have level~$a$ if $\alpha(\phi) \le a$.  
The set of all level-$a$ tests is therefore
\[
T_a = \{ \phi' \; : \; \mathbf{E}[\phi' \mid d] = \alpha(\phi') \le a \}.
\]

The \emph{false negative rate} (Type~II error) is defined as
\[
\beta(\phi) = 1 - \mathbf{E}[\phi \mid d'] \in [0,1],
\]
which represents the probability of incorrectly accepting the null hypothesis when $d'$ is true.

Among all tests with size at most $a$, the \emph{uniformly most powerful (UMP)} test is the one that minimizes the Type~II error.  
Formally, $\phi$ is UMP among level-$a$ tests if, for every $\phi' \in T_a$,
\[
\beta(\phi) \le \beta(\phi').
\]
The optimal trade-off between $\alpha$ and $\beta$ can thus be summarized by the \emph{trade-off function}
\[
f(a) = \inf_{\phi'} \{ \beta(\phi') : \alpha(\phi') \le a \}.
\]
This function $f(a)$ corresponds directly to the definition of $f$-DP~\cite{DBLP:journals/corr/abs-1905-02383}.

\noindent
According to the Neyman--Pearson lemma, any UMP test of level $a$ must be a \emph{likelihood-ratio test}.  
That is, there exists a threshold $h \ge 0$, where~$\mathbf{P}[\cdot]$ denotes probability:
\begin{equation}
\begin{array}{lcl}
\phi((x_j)_{j=1}^M) = 1 &\Rightarrow&
\mathbf{P}[(x_j)_{j=1}^M \mid d'] > h \cdot \mathbf{P}[(x_j)_{j=1}^M \mid d],\\[4pt]
\phi((x_j)_{j=1}^M) = 0 &\Rightarrow&
\mathbf{P}[(x_j)_{j=1}^M \mid d'] < h \cdot \mathbf{P}[(x_j)_{j=1}^M \mid d],
\end{array}
\label{NP}
\end{equation}
where the boundary set $\{(x_j)_{j=1}^M : \mathbf{P}[(x_j)_{j=1}^M \mid d'] = h \cdot \mathbf{P}[(x_j)_{j=1}^M \mid d]\}$ has probability zero under both $d$ and $d'$.  
The lemma guarantees that if a test $\phi$ satisfies $\alpha(\phi) = a$, then it is UMP among level-$a$ tests; moreover, all such UMP tests share the same likelihood-ratio form in \eqref{NP} with a common threshold~$h$.

Substituting \eqref{eq:likelihood-h0} and \eqref{eq:likelihood-h1} derived earlier into \eqref{NP} yields the concrete rejection rule
\begin{equation}
\begin{array}{lcl}
\phi((x_j)_{j=1}^M) = 1 &\Rightarrow&
\frac{1}{M} \sum_{j=1}^M e^{x_j/\sigma - 1/(2\sigma^2)} > h,\\[4pt]
\phi((x_j)_{j=1}^M) = 0 &\Rightarrow&
\frac{1}{M} \sum_{j=1}^M e^{x_j/\sigma - 1/(2\sigma^2)} < h.
\end{array}
\label{rejrule}
\end{equation}
In the subsequent sections, we will not use this exact form of~$\phi$, instead we adopt a suboptimal test to derive our bound.
\subsection{Geometric Setup and Notation}
\label{subsec:geometry-setup}
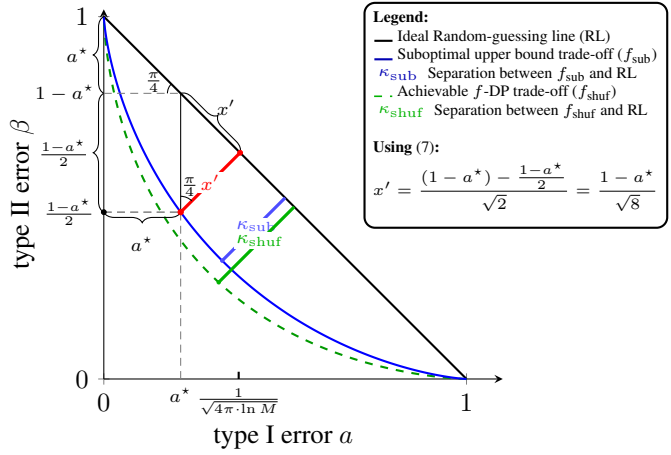
\begin{figure}[tbp]
\centering
\resizebox{0.6\columnwidth}{!}{%
\begin{tikzpicture}

  \pgfmathdeclarefunction{Phi}{1}{%
  \pgfmathsetmacro{\z}{#1/1.4142135623730951}%
  \pgfmathsetmacro{\pcoef}{0.3275911}%
  \pgfmathsetmacro{\ttmp}{1/(1+\pcoef*abs(\z))}%
  \pgfmathsetmacro{\cOne}{0.254829592}%
  \pgfmathsetmacro{\cTwo}{-0.284496736}%
  \pgfmathsetmacro{\cThree}{1.421413741}%
  \pgfmathsetmacro{\cFour}{-1.453152027}%
  \pgfmathsetmacro{\cFive}{1.061405429}%
  \pgfmathsetmacro{\poly}{((((\cFive*\ttmp+\cFour)*\ttmp+\cThree)*\ttmp+\cTwo)*\ttmp+\cOne)*\ttmp}%
  \pgfmathsetmacro{\erfapprox}{1 - \poly*exp(-\z*\z)}%
  \pgfmathsetmacro{\PhiA}{0.5*(1 + ( \z>=0 ? \erfapprox : -\erfapprox ))}%
  \pgfmathparse{\PhiA}%
}
\pgfmathdeclarefunction{gdpX}{2}{\pgfmathparse{1 - Phi(#1)}}
\pgfmathdeclarefunction{gdpY}{2}{\pgfmathparse{   Phi(#1 - #2)}}
  \pgfmathsetmacro{\muU}{1.10} 
  \pgfmathsetmacro{\muL}{0.90} 
  \pgfmathsetmacro{\tstar}{0.80} %
  \pgfmathsetmacro{\pU}{3.00} 
  \pgfmathsetmacro{\pL}{3.60}

\pgfmathsetmacro{\astar}{ gdpX(\tstar, \muL) }  %

  \pgfmathsetmacro{\yP}{    gdpY(\tstar, \muL) }  %

  \def\shiftc{0.06}

  \pgfmathsetmacro{\xQ}{(1 + \astar - \yP)/2}
  \pgfmathsetmacro{\yQ}{1 - \xQ}

\pgfmathsetmacro{\Px}{gdpX(\tstar, \muL)}
\pgfmathsetmacro{\Py}{gdpY(\tstar, \muL)}

  \pgfmathsetmacro{\Qx}{(1 + \Px - \Py)/2}
  \pgfmathsetmacro{\Qy}{1 - \Qx}
  
  \begin{axis}[
      width=8.5cm,height=6.5cm,
      xmin=0,xmax=1.10,
      ymin=0,ymax=1.02,
      axis lines=left,
      xlabel={type I error $a$},
      ylabel={type II error $\beta$},
      xlabel style={at={(axis description cs:0.45,-0.10)},anchor=north},
      ylabel style={at={(axis description cs:-0.15,0.5)},anchor=south},
      xtick={0,1},
      ytick={0,1},
      domain=0:1,
      samples=400,
      clip=false,
      axis equal image,
    ]

    \addplot[name path=guess,thick,black] {1 - x};

\addplot[
  name path=lower,
  thick, green!60!black, dashed,
  domain=-4:4, samples=400
] ({gdpX(x,\muU)}, {gdpY(x,\muU)});

\addplot[
  name path=upper,
  thick, blue,
  domain=-4:4, samples=400
] ({gdpX(x,\muL)}, {gdpY(x,\muL)});

    \addplot[densely dashed,gray] coordinates {(\astar,0) (\astar,\yP)};
    \addplot[only marks,mark=*,mark size=1.0pt,red] coordinates {(\astar,\yP)};
    \node[below] at (axis cs:\astar,0) {\scriptsize $a^\star$};

    \addplot[densely dashed,gray] coordinates {(0,\yP) (\astar,\yP)};
    \draw[decorate,decoration={brace,mirror,amplitude=4pt}]
      (axis cs:0,\yP) -- (axis cs:\astar,\yP) coordinate (L)
      node[midway,below=6pt] {\scriptsize $a^\star$};

    \draw[decorate,decoration={brace, amplitude=4pt}]
      (axis cs:0,\yP) -- (axis cs:0,1-\astar)
      node[midway,left=2pt] {\scriptsize $\tfrac{1-a^\star}{2}$};

    \draw[decorate,decoration={brace, amplitude=4pt}]
      (axis cs:0,1-\astar) -- (axis cs:0,1)
      node[midway,left=2pt] {\scriptsize ${a^\star}$};
      
    \draw[decorate,decoration={brace,mirror, amplitude=4pt}]
      (axis cs:\xQ,\yQ) -- (axis cs:\astar,1-\astar)
      node[pos=0.5,yshift=6pt,xshift=6pt] {\scriptsize $x'$};

    \draw[thick] (axis cs:\astar+0.16,0) -- ++(0,0.02);
    \node[below] at (axis cs:\astar+0.16,0) {\scriptsize $\tfrac{1}{\sqrt{4\pi\cdot\ln{M}}}$};

    \addplot[only marks,mark=*,mark size=0.9pt,red] coordinates {(\xQ,\yQ)};
    \draw[very thick,red] (axis cs:\astar,\yP) -- (axis cs:\xQ,\yQ);
    \node[red,fill=white,inner sep=1pt] at (axis cs:{(\astar+\xQ)/2},{(\yP+\yQ)/2})
      {\scriptsize $x'$};

    \addplot[densely dashed,gray] coordinates {(0,\yP) (\astar,\yP)};
    \addplot[only marks,mark=*,mark size=0.9pt,black] coordinates {(0,\yP)};

    \node[left] at (axis cs:0,\yP) {\scriptsize $\tfrac{1-a^\star}{2}$};

    \addplot[black] coordinates {(\astar,\yP) (\astar,1-\astar)};

    \addplot[densely dashed,gray] coordinates {(0,1-\astar) (\astar,1-\astar)};
    \node[left] at (axis cs:0,1-\astar) {\scriptsize $1-a^\star$};

    \path[name path=diag] (axis cs:0,0) -- (axis cs:1.05,1.05);
    \def\shift{0.05}
    \path[name path=diagKappa] (axis cs:0,-\shift) -- (axis cs:1.05,1.05-\shift);

    \path[name intersections={of=upper and diag, by=Ustar}];
    \path[name intersections={of=lower and diagKappa, by=Lstar}];

    \coordinate (LstarQ) at (axis cs:
        {(1 + \pgfkeysvalueof{/pgfplots/xmin} + 0.18)/2},
        {1 - (1 + \pgfkeysvalueof{/pgfplots/xmin} + 0.18)/2});

    \path[name path=kappaline] (Lstar) -- ++(1,1);
    \path[name intersections={of=kappaline and guess, by=Qkappa}];

    \coordinate (Mid) at (axis cs:0.5,0.5);
    \draw[very thick,blue!70] (Ustar) -- (Mid);
    \fill[blue!70] (Ustar) circle (0.006);
    \node[blue!70,fill=white,inner sep=1pt] at ($ (Ustar)!0.55!(Mid) $)
      {\scriptsize $\kappa_{\mathrm{sub}}$};

    \draw[very thick,green!70!black] (Lstar) -- (Qkappa);
    \fill[green!70!black] (Lstar) circle (0.006);
    \node[green!70!black,fill=white,inner sep=1pt] at ($ (Lstar)!0.55!(Qkappa) $)
      {\scriptsize $\kappa_{\mathrm{shuf}}$};
    \node[anchor=north east,
          align=left,
          draw=black,
          fill=white,
          rounded corners,
          thick,
          font=\tiny,
          xshift=17mm, yshift=1mm]
          at (rel axis cs:1,1) {%
      \textbf{Legend:}\\
      \tikz[baseline={(current bounding box.center)}]
          {\draw[black!70!black, thick] (0,0) -- (0.05,0);}~Ideal Random-guessing line (RL)\\
      \tikz[baseline={(current bounding box.center)}]
          {\draw[blue!70!black, thick] (0,0) -- (0.05,0);}~Suboptimal upper bound trade-off ($f_{\text{sub}}$)\\
      \tikz[baseline=-2ex]
            {\node[text=blue!70!black] {$\kappa_{\mathrm{sub}}$};}~Separation between $f_{\text{sub}}$ and RL\\
      \tikz[baseline={(current bounding box.center)}]
          {\draw[green!70!black, dashed, line width=0.6pt] (0,0) -- (0.05,0);}~Achievable $f$-DP trade-off ($f_{\text{shuf}}$)\\
      \tikz[baseline=-2ex]
            {\node[text=green!70!black] {$\kappa_{\mathrm{shuf}}$};}~Separation between $f_{\text{shuf}}$ and RL\\\\
      \textbf{Using~\eqref{eq:pointwise-sep}:}\\[2pt]
      \(x' = \dfrac{(1-a^\star) - \tfrac{1-a^\star}{2}}{\sqrt{2}}
            = \dfrac{1-a^\star}{\sqrt{8}}\)\\[2pt]
    };

\coordinate (P) at (axis cs:\astar,\yP);   %
\coordinate (Q) at (axis cs:\xQ,\yQ);     %

\coordinate (Phor) at ($(P)+(axis direction cs: 0.40,0)$); %
\coordinate (Qhor) at ($(Q)+(axis direction cs: 0.20,0)$); %

\pic [draw,angle radius=6pt,angle eccentricity=1.05,
      "$\tfrac{\pi}{4}$",
      pic text options={font=\scriptsize, xshift=1pt, yshift=3.8pt, inner sep=1.5pt}]
     {angle = Q--P--L};

\coordinate (S)      at (axis cs:\astar,1-\astar);
\coordinate (Sleft)  at ($(S)+(axis direction cs:-0.15,0)$);      %
\coordinate (Sdiag)  at ($(S)+(axis direction cs:-0.14,0.15)$);    %

\pic [draw,angle radius=5pt,angle eccentricity=1.1,
      "$\tfrac{\pi}{4}$",
      pic text options={font=\scriptsize, xshift=-5.2pt, yshift=1pt, inner sep=0.5pt, yshift=1pt}]
     {angle = Sdiag--S--Sleft};

\end{axis}
\end{tikzpicture}%
}
\caption{Illustrative geometry of the suboptimal and true trade-off functions in our impossibility argument. The figure compares the \emph{random-guessing line} (black), the \emph{suboptimal upper bound} trade-off $f_{\text{sub}}$ (blue), and the  \emph{true $f$-DP trade-off} $f_{\text{shuf}}$ (green, dashed). The red segment $x'$ is the pointwise separation of the suboptimal curve at the analytically tractable point $a^\star$, which forms an \emph{isosceles right triangle} with base angle \(45^\circ\); the blue segment $\kappa_{\mathrm{sub}}$ is the separation of the entire suboptimal curve; and the green segment $\kappa_{\mathrm{shuf}}$ is the separation of the true $f$-DP trade-off under random shuffling. The separations satisfy $x' \le \kappa_{\mathrm{sub}} \le \kappa_{\mathrm{shuf}}$.}
\label{fig:proof_fig}
\end{figure}

We conclude this section by introducing the geometric framework used in our limitation argument. Our analysis is based on the notion of \emph{separation} between an $f$-DP
trade-off curve and the random-guessing line, and on the observation that
separation can be lower bounded using any explicit pointwise upper bound on the
true trade-off curve. 

\begin{definition}[Pointwise separation]
\label{def:pointwise-sep}
For any trade-off curve $g(\alpha)$, its pointwise separation from the
ideal random-guessing line $\beta = 1-\alpha$ is
\begin{equation}
\operatorname{sep}_g(\alpha)
:=
\frac{(1-\alpha)-g(\alpha)}{\sqrt{2}}.
\label{eq:pointwise-sep}
\end{equation}
\end{definition}

\paragraph{Geometric interpretation.}
For any $\alpha\in[0,1]$, the quantity $\operatorname{sep}_g(\alpha)$ in \eqref{eq:pointwise-sep} is exactly the Euclidean distance from the point $(\alpha,g(\alpha))$ to the random-guessing line $\mathcal{L} := \{(\gamma,1-\gamma):\gamma\in[0,1]\}$. In Fig.~\ref{fig:proof_fig}, treating the blue curve as the graph of $g$, the perpendicular dropped from $(\alpha,g(\alpha))$ to $\mathcal{L}$ forms an isosceles right triangle with legs of length $x'$, so that $((1-\alpha)-g(\alpha))^2 = 2(x')^2$, yielding \eqref{eq:pointwise-sep}.

\begin{definition}[Global Separation]
\label{def:global-sep}
The separation of a trade-off curve $g$, denoted $\operatorname{sep}(g)$, is the maximum pointwise separation:
\begin{equation}
\operatorname{sep}(g)
=
\max_{\alpha\in[0,1]} \operatorname{sep}_g(\alpha).
\label{eq:global-sep-from-pointwise}
\end{equation}
Equivalently, using the geometric interpretation above, this corresponds to the maximum Euclidean distance from the curve to the random-guessing line:
\begin{equation}
\operatorname{sep}(g)
=
\max_{\alpha\in[0,1]}
\min_{\gamma\in[0,1]}
\bigl\|(\alpha,g(\alpha))-(\gamma,1-\gamma)\bigr\|_2 .
\label{eq:sep-def}
\end{equation}
where $\gamma$ parameterizes the projection of $(\alpha,g(\alpha))$ onto the random-guessing line $\mathcal{L}$. Equations \eqref{eq:global-sep-from-pointwise} and \eqref{eq:sep-def} therefore define the same quantity.
\end{definition}

We now characterize where the maximum in \eqref{eq:global-sep-from-pointwise}–\eqref{eq:sep-def} is attained for the class of symmetric and convex trade-off curves.

\begin{lemma}[Separation]
\label{lem:sep-fixed-point}
Let $f\colon[0,1]\to[0,1]$ be a symmetric and convex trade-off function in the
$f$-DP framework. Then the maximum in \eqref{eq:sep-def} is attained at the unique
fixed point $\hat{a}\in(0,1/2]$ satisfying $f(\hat{a})=\hat{a}$, and the separation
equals
\begin{equation}
\operatorname{sep}(f)
=
\frac{1-2\hat{a}}{\sqrt{2}}
=:\kappa.
\label{eq:sep-fixed-point}
\end{equation}

Moreover, for any symmetric and convex $f$ that is not identically equal to the
random-guessing line, the separation necessarily satisfies
\begin{equation}
0 \;<\; \kappa \;<\; \frac{1}{\sqrt{2}}.
\label{eq:sep-upperbound-absolute}
\end{equation}
\end{lemma}

\begin{proof}
Because the random-guessing line is the diagonal
$\{(\gamma,1-\gamma)\colon \gamma\in[0,1]\}$, the inner minimization in
\eqref{eq:sep-def} amounts to orthogonally projecting $(\alpha,f(\alpha))$ onto
this line. %
The projection lands at $\gamma = \bigl(\alpha - f(\alpha) + 1\bigr)/2 \in [0,1]$, so the
minimum is the perpendicular distance $\bigl((1-\alpha)-f(\alpha)\bigr)/\sqrt2$, giving the
simplified form in~\eqref{eq:global-sep-from-pointwise}.
\begin{equation}
\operatorname{sep}(f)
=
\max_{\alpha\in[0,1]}
\frac{(1-\alpha)-f(\alpha)}{\sqrt{2}}.
\label{eq:sep-simplified}
\end{equation}

By convexity of optimal trade-off curves
\cite{DBLP:journals/corr/abs-1905-02383,DBLP:journals/corr/abs-1911-11607}, the map
$\alpha \mapsto (1-\alpha)-f(\alpha)$ is concave, so its maximum is attained where
$-1 \in \partial f(\alpha)$, and by symmetry of $f$ about the diagonal that point is the
unique fixed point $\hat{a}$ satisfying $f(\hat{a})=\hat{a}$. The separation is therefore
the distance between $(\hat{a},\hat{a})$ and the midpoint $(\tfrac12,\tfrac12)$ of the
random-guessing line, and substituting into \eqref{eq:sep-simplified} yields
\eqref{eq:sep-fixed-point}.

The absolute bound \eqref{eq:sep-upperbound-absolute} follows directly from
\eqref{eq:sep-fixed-point} and the fact that any valid symmetric convex trade-off curve
with $f(0)>0$ satisfies $0 < \hat{a} < 1/2$
\end{proof}

\paragraph{Remark.}
Separation inherits structural properties from $f$-DP such as post-processing because it is derived from the trade-off curve. However, $\operatorname{sep}(f)$ is \emph{not} a DP metric: it lacks a composition theorem and cannot track privacy under adaptive composition. It is an interpretable geometric property of an $f$-DP curve.
It does, however, admit a direct operational reading. If
$\operatorname{sep}(f) \ge \kappa$, then at the fixed point of the trade-off curve the
worst-case adversary attains $\alpha + \beta = 1 - \sqrt{2}\,\operatorname{sep}(f) \le
1 - \sqrt{2}\,\kappa$, i.e., total error better than random guessing
($\alpha + \beta = 1$) by at least $\sqrt{2}\,\kappa$. Conversely, if
$\operatorname{sep}(f) \le \kappa$, then every test satisfies
$\alpha + \beta \ge 1 - \sqrt{2}\,\kappa$, uniformly over the curve. Separation is not
intended to replace $(\varepsilon,\delta)$. It is the geometric intermediate that our proof
produces, it is uniform over all $f$-DP families, and it translates back to
$(\varepsilon,\delta)$ in closed form (Corollary~\ref{cor:eps-delta-main}).

\paragraph{Geometric preview.}
Figure~\ref{fig:proof_fig} provides a preview of the geometric structure
underlying our limitation argument. The black line corresponds to random guessing, the green curve depicts the true $f$-DP trade-off $f_{\mathrm{shuf}}$ of shuffled DP-SGD, and the blue curve
$f_{\mathrm{sub}}$ represents a particular explicit trade-off curve that
upper-bounds $f_{\mathrm{shuf}}$ pointwise. \textbf{At this stage, the figure should be interpreted illustratively}: the precise definition and derivation of $f_{\mathrm{sub}}$ are deferred to
Section~\ref{subsec:suboptimal-test}.

The figure highlights three separation quantities.
The quantity $x'$ denotes the pointwise separation at a specific point on the
upper-bounding curve, $\kappa_{\mathrm{sub}}$ denotes the global separation of that upper bound, and
$\kappa_{\mathrm{shuf}}$ denotes the separation of the true trade-off curve.
Independently of how the upper bound is constructed, these quantities necessarily
satisfy
\[
x' \;\le\; \kappa_{\mathrm{sub}} \;\le\; \kappa_{\mathrm{shuf}},
\]
whenever one trade-off curve upper bounds another pointwise.

In Section~\ref{subsec:suboptimal-test}, we formally construct the suboptimal
trade-off curve $f_{\mathrm{sub}}$, verify that it upper bounds
$f_{\mathrm{shuf}}$, and make all quantities shown in
Figure~\ref{fig:proof_fig} fully explicit.

\section{Limitations of Favorable Privacy–Utility Trade-offs}
\label{sec:limitations}

\begin{theorem}[Lower Bound on Separation for Shuffled DP-SGD]
\label{thm:shuffling}
Consider one-epoch random shuffling. Let $M \ge 2$ denote the number of rounds, and let $\sigma > 0$ be the noise multiplier.
Assume the worst-case sensitivity $C$ is attained and that $M$ divides $N$, so that
shuffling discards no records.
Let $\kappa_{\mathrm{shuf}}:= \mathrm{sep}(f_{\mathrm{shuf}})$ denote the separation of the true $f$-DP trade-off curve under shuffling defined in~\eqref{eq:kappa-def}.

Then we have the following statement:
\[
\boxed{
\kappa_{\mathrm{shuf}}
\;\ge\;
\frac{1}{\sqrt{8}}
\left(1 - \frac{1}{\sqrt{4\pi\ln M}}\right)\quad\textbf{OR}\quad \sigma \;\ge\; \frac{1}{\sqrt{2\ln M}}
}
\]

\end{theorem}

Thus, unless the assumptions underlying this worst-case analysis are relaxed, improving privacy (i.e., reducing the separation $\kappa$) inevitably forces the noise level $\sigma$ into a regime that degrades utility (Corollary~\ref{cor:eps-delta-main} below, with the full derivation in \appref{sec:app-kappa-to-eps-delta}{E}). This trade-off is reflected in our experimental results and is consistent with empirical observations in private deep learning. In Section~\ref{subsec:poisson-mixture}, we show that the same phenomenon persists under Poisson subsampling (up to constant factors). The proof of Theorem~\ref{thm:shuffling} is given in Section~\ref{subsec:separation-lowerbound}. The $\sqrt{2\ln M}$ scale reflects Gaussian extreme-value behavior in the observation model itself, since the expected maximum of $M$ i.i.d.\ standard normals is at most $\sqrt{2\ln M}$~\cite[Thm.~2.5]{blm2013}, as \appref{app:sqrt-2lnM}{G} explains.

\begin{corollary}[$(\varepsilon,\delta)$-DP implication]
\label{cor:eps-delta-main}
Under the conditions of Theorem~\ref{thm:shuffling}, suppose $\sigma < 1/\sqrt{2\ln M}$ and let
\[
\kappa_M := \frac{1}{\sqrt{8}}\left(1 - \frac{1}{\sqrt{4\pi\ln M}}\right)
\;\le\; \kappa_{\mathrm{shuf}}.
\]
If the same mechanism is reported as $(\varepsilon,\delta)$-DP, then
\[
\varepsilon \;\ge\; \ln\!\left(\frac{1 + \kappa_M\sqrt{2} - 2\delta}{\,1 - \kappa_M\sqrt{2}\,}\right).
\]
In particular, under the one-epoch baseline $\delta = 1/N$ with $N = 10^8$, this already forces
$\varepsilon \gtrsim 1$ across all practically relevant $M$ (e.g., $\varepsilon \ge 0.96$ at
$M = 10^3$, rising to $\varepsilon \ge 1.00$ at $M = 5\times 10^6$, see
the table in \appref{sec:app-kappa-to-eps-delta}{E}). The proof follows from
Lemmas~E.2 and~E.3 in \appref{sec:app-kappa-to-eps-delta}{E}, together with monotonicity of the map
$\kappa \mapsto \varepsilon_{\min}(\kappa;N)$ defined there, which lets us substitute the lower bound
$\kappa_M$ for $\kappa_{\mathrm{shuf}}$.
\end{corollary}

\subsection{Suboptimal Hypothesis Test}
\label{subsec:suboptimal-test}
Instead of using the optimal Neyman–Pearson likelihood ratio test, we consider a \emph{suboptimal} rejection rule that thresholds the \emph{maximum} coordinate of the observation vector $(x_1, \ldots, x_M)$:
\begin{equation}
     \begin{array}{lcr} 
     \phi((x_j)_{j=1}^M)=1 &\Rightarrow& \displaystyle\max_{j=1}^M 
     e^{x_j/\sigma-1/(2\sigma^2)} > h, \\[5pt]
     \phi((x_j)_{j=1}^M)=0 &\Rightarrow& \displaystyle\max_{j=1}^M 
     e^{x_j/\sigma-1/(2\sigma^2)} < h.
     \end{array}
     \label{eq:rejrule-sub}
\end{equation}
This test is equivalent to thresholding the maximum observed statistic:
\begin{equation}
     \begin{array}{lcr}
     \phi((x_j)_{j=1}^M)=1 &\Rightarrow& \max_{j=1}^M x_j > \bar{h}, \\[3pt]
     \phi((x_j)_{j=1}^M)=0 &\Rightarrow& \max_{j=1}^M x_j < \bar{h},
     \end{array}
     \label{eq:rejrule-max}
\end{equation}
where $\bar{h} = \sigma \ln h + 1/(2\sigma)$.  
Although suboptimal, this rule isolates the single coordinate shifted by $1/\sigma$ under $H_1$, enabling a closed-form $f$-DP trade-off.  

\paragraph{False-negative and false-positive rates.}
For $\bar{h} \ge 0$, define
\begin{align}
\bar{\alpha}(\bar{h}) &= \Pr\!\left[\max_{j=1}^M x_j > \bar{h} \mid d\right], \label{eq:alpha-sub}\\
\bar{\beta}(\bar{h})  &= \Pr\!\left[\max_{j=1}^M x_j < \bar{h} \mid d'\right]. \label{eq:beta-sub}
\end{align}

Under $d$ (no shift), all $x_j \sim \mathcal{N}(0,1)$ independently, giving
\[
\bar{\alpha}(\bar{h}) = 1 - \Phi(\bar{h})^M,
\qquad
\bar{\alpha}^{-1}(a) = \Phi^{-1}\!\big((1-a)^{1/M}\big),
\]
where $\Phi$ is the CDF of $\mathcal{N}(0,1)$.

Under $d'$, the coordinates are \emph{not} marginally independent. The joint law is the
uniform mixture over the $M$ possible positions $j^\star$ of the shifted coordinate
(cf.~\eqref{eq:likelihood-h1}). Conditional on any fixed position $j^\star = j$, however,
the coordinates are independent, with $x_j \sim \mathcal{N}(1/\sigma,1)$ and the remaining
$M-1$ coordinates i.i.d.\ $\mathcal{N}(0,1)$. Hence
\[
\bar{\beta}(\bar{h})
= \frac{1}{M}\sum_{j=1}^{M}
\Pr\!\Bigl[\,\max_{i}\, x_i < \bar{h} \,\Bigm|\, j^\star = j\Bigr]
= \Phi(\bar{h}-1/\sigma)\,\Phi(\bar{h})^{M-1},
\]
since, by conditional independence, the conditional probability equals
$\Phi(\bar{h}-1/\sigma)\,\Phi(\bar{h})^{M-1}$ for every position $j$ of the shifted coordinate.
Thus, the induced (suboptimal) trade-off function is equal to $\bar{\beta}(\bar{\alpha}^{-1}(a))$:
\begin{equation}
\label{eq:suboptimal-tradeoff}
f_{\mathrm{sub}}(a) = \Phi\:\!\Big(\Phi^{-1}\big((1-a)^{1/M}\big)-\sigma^{-1}\Big)\cdot(1-a)^{(M-1)/M}.
\end{equation}

Throughout this section, we work in the regime $\sigma \le 1/\sqrt{2\ln M}$, and write
$\sigma = s/\sqrt{\ln M}$ for some $s \le \sqrt{1/2}$.
\begin{lemma}[Bound under $\sigma \le 1/\sqrt{2\ln M}$]
\label{lem:suboptimal-bound}
Let $f_{\mathrm{sub}}(a)$ denote the suboptimal trade-off in~\eqref{eq:suboptimal-tradeoff}.
For any integer $M \ge 2$ and any noise multiplier $\sigma > 0$ satisfying $\sigma \;\le\; {1}/{\sqrt{2\ln M}}$
there exists some $a^\star \le 1 / \sqrt{4\pi \ln M}$ such that
\[
f_{\mathrm{sub}}(a^\star) = \tfrac{1}{2}\,(1-a^\star)^{(M-1)/M}.
\]
\end{lemma}
\begin{proof}
Under the standing assumption $\sigma = s/\sqrt{\ln M}$ with $s \le \sqrt{1/2}$, we choose $a^\star$ such that the inner argument of $\Phi$ in~\eqref{eq:suboptimal-tradeoff} equals zero, i.e.,
\begin{equation}
\Phi^{-1}\!\big((1-a^\star)^{1/M}\big) = \sigma^{-1},
\label{eq:astar-choice}
\end{equation}
Substituting into~\eqref{eq:suboptimal-tradeoff} gives
\[
f_{\mathrm{sub}}(a^\star) = \Phi(0)\cdot(1 - a^\star)^{(M-1)/M} = \frac{1}{2}(1 - a^\star)^{(M-1)/M},
\]
where we used the symmetry of the standard normal distribution, for which $\Phi(0) = \tfrac{1}{2}$.

Applying $\Phi$ to both sides of~\eqref{eq:astar-choice}, and using that $\Phi$ is invertible, yields
$(1 - a^\star)^{1/M} = \Phi(1/\sigma)$, and hence
\begin{equation}
\label{eq:1-astar}
1 - a^\star = \Phi(1/\sigma)^M = (1 - \Phi(-1/\sigma))^M.
\end{equation}

Next, we use the standard Gaussian tail bound
\[
\Phi(-t) < \frac{e^{-t^2/2}}{t\sqrt{2\pi}}, \quad t>0,
\]
and substituting $t = 1/\sigma$, we obtain
\[
\Phi(-1/\sigma) \;\le\; \frac{\sigma e^{-1/(2\sigma^2)}}{\sqrt{2\pi}}.
\]
Exponentiating both sides with $M$ and using~\eqref{eq:1-astar} yields
\[
1 - a^\star \ge \left(1 - \frac{\sigma e^{-1/(2\sigma^2)}}{\sqrt{2\pi}}\right)^M.
\]
Substituting $\sigma = s/\sqrt{\ln M}$ yields
\[
1 - a^\star \ge
\left(1 - \frac{s\,M^{-1/(2s^2)}}{\sqrt{2\pi\ln M}}\right)^M
\ge 1 - \frac{s\,M^{1-1/(2s^2)}}{\sqrt{2\pi\ln M}}.
\]
For $s \le \sqrt{1/2}$, this implies
\[
a^\star \le \frac{1}{\sqrt{4\pi\ln M}}.
\]
This establishes the desired bound.
\end{proof}

\subsection{Separation Lower Bounds for Shuffled DP-SGD}
\label{subsec:separation-lowerbound}
We now instantiate the geometric framework introduced earlier using the analytically tractable suboptimal trade-off curve derived above (cf.~\eqref{eq:suboptimal-tradeoff}). This yields an explicit geometric lower bound on separation. In Figure~\ref{fig:proof_fig}, we consider:
\begin{itemize}
  \item the \textbf{black line} $\beta = 1-\alpha$ (ideal random guessing);
  \item the \textbf{green curve} $f_{\mathrm{shuf}}(\alpha)$, representing the
        true trade-off of shuffled DP-SGD under a worst-case adversary;
  \item the \textbf{blue curve} $f_{\mathrm{sub}}(\alpha)$, an explicit suboptimal
        upper bound derived in~\eqref{eq:suboptimal-tradeoff}.
\end{itemize}

\paragraph{Notation: three separations.}
Using Definition~\ref{def:pointwise-sep} and \eqref{eq:global-sep-from-pointwise},
we introduce the three quantities used in the analysis and in the Figure~\ref{fig:proof_fig}:
\begin{align}
x'
&:= \operatorname{sep}_{f_{\mathrm{sub}}}(a^\star),
\label{eq:xprime-def}
\\[3pt]
\kappa_{\mathrm{sub}}\phantom{'}
&:= \operatorname{sep}\!\left(f_{\mathrm{sub}}\right)
 = \max_{\alpha\in[0,1]} \operatorname{sep}_{f_{\mathrm{sub}}}(\alpha),
\label{eq:x-def}
\\[3pt]
\kappa_{\mathrm{shuf}}
&:= \operatorname{sep}\!\left(f_{\mathrm{shuf}}\right)
 = \max_{\alpha\in[0,1]} \operatorname{sep}_{f_{\mathrm{shuf}}}(\alpha).
\label{eq:kappa-def}
\end{align}

Here, $x'$ is the pointwise separation at the analytically tractable point
$(a^\star,f_{\mathrm{sub}}(a^\star))$, $\kappa_{\mathrm{sub}}$ is the global separation of the suboptimal curve
$f_{\mathrm{sub}}$, and $\kappa_{\mathrm{shuf}}$ is the true separation of the shuffled
trade-off curve $f_{\mathrm{shuf}}$. Then, since $f_{\mathrm{sub}}(\alpha) \ge f_{\mathrm{shuf}}(\alpha)$ for all $\alpha$, we have
\[
\operatorname{sep}_{f_{\mathrm{sub}}}(\alpha)
\;\le\;
\operatorname{sep}_{f_{\mathrm{shuf}}}(\alpha)
\qquad\text{for all }\alpha,
\]
and therefore
\begin{equation}
x' \;\le\; \kappa_{\mathrm{sub}} \;\le\; \kappa_{\mathrm{shuf}}.
\label{eq:sep-chain}
\end{equation}
Thus any lower bound obtained for $x'$ immediately implies the same lower
bound for $\kappa_{\mathrm{shuf}}$, the true separation of the shuffled $f$-DP trade-off.

We now prove Theorem~\ref{thm:shuffling} by combining the ingredients developed in
the previous sections. The proof proceeds by constructing an explicit, analytically tractable upper bound on the $f$-DP trade-off curve under one-epoch shuffling using the suboptimal hypothesis test of Section~\ref{subsec:suboptimal-test}, and then translating this bound into a geometric lower bound on the separation via the framework of Section~\ref{subsec:geometry-setup}.

\begin{proof}[Proof of Theorem~\ref{thm:shuffling}]
Assume $\sigma \le 1/\sqrt{2\ln M}$, and let $f_{\mathrm{sub}}$ denote the induced (suboptimal) trade-off curve defined in~\eqref{eq:suboptimal-tradeoff}.

We now convert Lemma~\ref{lem:suboptimal-bound} into a geometric lower bound that reveals an inherent limitation of the privacy--utility trade-off. In Figure~\ref{fig:proof_fig}, the suboptimal point $(a^\star, f_{\mathrm{sub}}(a^\star))$ lies on the blue curve representing the analytically tractable upper bound:
\[
(a^\star,f_{\mathrm{sub}}(a^\star))
=\Bigl(a^\star,\tfrac12(1-a^\star)^{(M-1)/M}\Bigr).
\]
Let $x'$ be the perpendicular distance from this point to the random-guessing line $\beta=1-\alpha$,
so that $x'$ is given by~\eqref{eq:xprime-def}.

Rewrite
\[
f_{\mathrm{sub}}(a^\star)
=\tfrac12(1-a^\star)^{(M-1)/M}
=\frac{1-a^\star}{2}\,(1-a^\star)^{-1/M}.
\]
By Lemma~\ref{lem:suboptimal-bound}, $a^\star \le 1/\sqrt{4\pi\ln M}$, hence
\begin{equation}
(1-a^\star)^{-1/M}
\le
\Bigl(1-\tfrac{1}{\sqrt{4\pi\ln M}}\Bigr)^{-1/M}.
\label{eq:astar-factor-upper}
\end{equation}
Define
\begin{equation}
\varepsilon_M
\;:=\;
\frac{2}{M\sqrt{4\pi\ln M}}.
\label{eq:eps-def}
\end{equation}
Since $1/\sqrt{4\pi\ln M}<1/2$ for all $M\ge 2$, the elementary inequality $(1-t)^{-1/M}\le 1+\frac{2t}{M}$ for $t\in(0,1/2)$ implies (applying it to the right-hand side of~\eqref{eq:astar-factor-upper} with $t=1/\sqrt{4\pi\ln M}$) that
\begin{equation}
(1-a^\star)^{-1/M}\le 1+\varepsilon_M.
\label{eq:factor-upper-eps}
\end{equation}
Consequently,
\begin{equation}
f_{\mathrm{sub}}(a^\star)\le \frac{1-a^\star}{2}\,(1+\varepsilon_M).
\label{eq:fsub-upper}
\end{equation}

By definition, the perpendicular distance to the line $\beta=1-\alpha$ equals
\[
x'=\frac{(1-a^\star)-f_{\mathrm{sub}}(a^\star)}{\sqrt2}.
\]
Substituting~\eqref{eq:fsub-upper} yields
\begin{equation}
x' \ge \frac{1-a^\star}{\sqrt8}\,(1-\varepsilon_M).
\label{eq:xprime-lb}
\end{equation}
Using again $a^\star \le 1/\sqrt{4\pi\ln M}$ gives
\begin{equation}
x' \ge \frac{1}{\sqrt8}\left(1-\frac{1}{\sqrt{4\pi\ln M}}\right)(1-\varepsilon_M).
\label{eq:xprime-explicit}
\end{equation}

Finally, since $f_{\mathrm{sub}}$ arises from a specific (suboptimal) test while
$f_{\mathrm{shuf}}$ is the infimum over all tests, we have
$f_{\mathrm{sub}}(\alpha) \ge f_{\mathrm{shuf}}(\alpha)$ for all $\alpha$. The monotonicity
of pointwise separation under this ordering, established in~\eqref{eq:sep-chain}, gives
$x' \le \kappa_{\mathrm{sub}} \le \kappa_{\mathrm{shuf}}$.
Therefore,
\begin{equation}
\label{eq:separation-lowerbound-final}
\kappa_{\mathrm{shuf}}
\ge x'
\ge
\frac{1}{\sqrt8}\left(1-\frac{1}{\sqrt{4\pi\ln M}}\right)(1-\varepsilon_M)
\end{equation}
which proves the claimed lower bound.
Note that, although the bound above is explicit, the correction factor $(1-\varepsilon_M)$ converges to $1$ at rate $O(1/M)$; we therefore omit it in the theorem statement and retain it only in the proof for full rigor.
\end{proof}

\paragraph{Interpretation.}
The inequality~\eqref{eq:separation-lowerbound-final} provides a concrete \emph{lower bound} on the geometric separation $\kappa_{\mathrm{shuf}}$ between the true $f$-DP trade-off curve and the random-guessing line. Equivalently, it shows that any mechanism whose noise level satisfies $\sigma < 1/\sqrt{2\ln M}$ must retain at least this amount of separation and therefore cannot make its trade-off curve arbitrarily close to the random-guessing regime. Moreover, although the bound on $\sigma$ decreases as $M$ grows, the decay is extremely slow. Even for very large $M$, the required noise level remains far from zero. For example, assuming a batch size of 256, the standard ImageNet-1k dataset ($\sim$1.3M images, $M \approx 5\times 10^3$) requires $\sigma \gtrsim 0.24$. Scaling to foundation models on LAION-5B~\cite{DBLP:conf/nips/SchuhmannBVGWCC22} ($\sim$5.8B image-text pairs, $M \approx 2.3\times 10^7$), the limit only drops to $\sigma \approx 0.17$. Thus, the condition on $\sigma$ remains substantial across practically relevant values of $M$, revealing a true bottleneck in this setting and simply scaling the dataset is insufficient to bypass this fundamental noise bottleneck.

\subsection{Extension to Poisson Subsampling}
\label{subsec:poisson-mixture}
While our derivation focuses on the shuffled regime which reflects practical implementations, we now show that the same phenomenon extends to Poisson subsampling. A simple mixture argument reveals that Poisson subsampling can be upper bounded by a mechanism that, with some probability, leaks no information at all, and otherwise behaves like the shuffled mechanism analyzed above. This immediately transfers the shuffled lower bound on separation to the Poisson setting up to a constant factor.
We first state the resulting bound for Poisson-sampled DP-SGD, deferring the technical argument to a supporting lemma below.

\begin{theorem}[Lower Bound on Separation for Poisson-sampled DP-SGD]
\label{thm:poisson}
Consider DP-SGD trained for one epoch under \emph{Poisson subsampling} with $M \ge 2$ rounds, sampling rate $q = 1/M$ (i.e., expected batch size $b = N/M$, so that one epoch is processed in expectation), and noise multiplier $\sigma > 0$, under the assumptions of Theorem~\ref{thm:shuffling}.
Let $f_{\mathrm{pois}}$ denote the true $f$-DP trade-off curve under Poisson subsampling, and let $\operatorname{sep}(f_{\mathrm{pois}})$ denote its separation.
Then the following dichotomy holds:
\[
\boxed{
\kappa_{\mathrm{pois}}
\;\ge\;
\Bigl(1-\tfrac{1}{e}\Bigr)\,
\frac{1}{\sqrt{8}}
\left(1 - \frac{1}{\sqrt{4\pi\ln M}}\right)\quad\textbf{OR}\quad \sigma \;\ge\; \frac{1}{\sqrt{2\ln M}}
}
\]
\end{theorem}

\begin{lemma}[Transferring shuffling bounds to Poisson subsampling under zero-out adjacency]
\label{lem:shuf-bound-to-pois}
Consider DP-SGD trained for one epoch under \emph{Poisson subsampling} with
$M \ge 1$ rounds and per-round sampling probability $q\in(0,1)$, and assume
\emph{zero-out adjacency} (cf. Definition~\ref{def:zeroout-adjacency}).
Let $f_{\mathrm{pois}}$ denote the true $f$-DP trade-off curve under Poisson subsampling,
and let $f_{\mathrm{shuf}}$ denote the true $f$-DP trade-off curve under one-epoch random shuffling
into $M$ batches of size $b = N/M$, where $M$ divides $N$ as in
Theorem~\ref{thm:shuffling}, so that the shuffling procedure discards no records.
Assume that
\[
f_{\mathrm{shuf}}(\alpha)\le f'_{\mathrm{shuf}}(\alpha)
\qquad \text{for all }\alpha\in[0,1].
\]
Let $p=(1-q)^M$ be the probability that the differing record is never sampled during the epoch.
Then, for all $\alpha\in[0,1]$,
\begin{equation}
\label{eq:poisson-mixture-bound}
f_{\mathrm{pois}}(\alpha)
\;\le\;
p(1-\alpha) + (1-p)\,f'_{\mathrm{shuf}}(\alpha).
\end{equation}
In particular, if the expected batch size is $b=qN$ and number of rounds is $M=N/b$ and $q=b/N=1/M$, then for typical batch sizes $b \ll N$ we have
\[
p=(1-q)^M=\Bigl(1-\tfrac1M\Bigr)^M \approx e^{-1}.
\]
\end{lemma}

Proof of Lemma~\ref{lem:shuf-bound-to-pois} can be found in \appref{app:proof-poisson-lemma}{D}. We now combine this mixture bound with the separation lower bound established for shuffled DP-SGD to complete the proof of Theorem~\ref{thm:poisson}.
\begin{proof}[Proof of Theorem~\ref{thm:poisson}]
Under Poisson subsampling, each sample is independently included in each round
with probability $q=1/M$ over $M$ rounds (one epoch), so that
$p=(1-q)^M$. Let $f_{\mathrm{shuf}}$ denote the true $f$-DP trade-off curve under one-epoch
random shuffling. By Lemma~\ref{lem:shuf-bound-to-pois} (applied with
$f'_{\mathrm{shuf}}:=f_{\mathrm{shuf}}$), for all $\alpha\in[0,1]$,
\[
f_{\mathrm{pois}}(\alpha)
\;\le\;
p(1-\alpha)+(1-p)\,f_{\mathrm{shuf}}(\alpha).
\]
Geometrically, this shows that $f_{\mathrm{pois}}$ lies below a convex
combination of the random-guessing line (which has zero separation) and the
shuffled trade-off curve. Hence,
\[
\kappa_{\mathrm{pois}}
\;\ge\;
(1-p)\,\kappa_{\mathrm{shuf}}.
\]
Finally, using $p\approx e^{-1}$ from
Lemma~\ref{lem:shuf-bound-to-pois} and using Theorem~\ref{thm:shuffling} proves the claim.
\end{proof}
Shuffled and Poisson DP-SGD obey the \emph{same} worst-case limitation up to a constant factor. In particular, Poisson-sampled DP-SGD cannot make its trade-off curve arbitrarily close to random guessing under the standard worst-case DP adversary and the one-epoch regime, unless $\sigma$ exceeds the same threshold ($\sigma \ge 1/\sqrt{2\ln M}$) as in the shuffled case. Together, Theorems~\ref{thm:shuffling} and~\ref{thm:poisson} establish a unified worst-case limitation. Under the standard worst-case DP adversary, neither shuffling nor Poisson subsampling can simultaneously achieve a noise multiplier below the $\sigma$ threshold and a separation below the corresponding lower bounds. 

\section{Empirical Analysis}
\label{sec:experiments}
\paragraph{Sampling, normalization, and microbatching.}
We consider both Poisson subsampling and random shuffling.
For Poisson subsampling, each example is included independently with probability $q$ at each step, giving a random number of participating examples. To have the DP-SGD update consistent with common practice, we normalize gradients and calibrate Gaussian noise using the \emph{expected} batch size $b = qN$, treating non-sampled examples as contributing zero gradient. This yields one DP-SGD update per Poisson-sampled step with a fixed noise scale, while preserving the independence structure assumed in our theoretical analysis.
For random shuffling, we use fixed-size batches throughout by shuffling the dataset at each epoch and discarding any incomplete remainder batch which coincides with the standard shuffled DP-SGD mechanism (Definition~\ref{def:random-shuffling}). All experiments are implemented using the \texttt{JAX Privacy} DP-SGD framework~\ifarxiv\cite{jax-privacy2022github,jax2018github}\fi. We use microbatching with fixed size 32 only to reduce memory usage. Gradients are clipped per example and Gaussian noise is added once per logical batch, so microbatching does not affect the optimization or privacy mechanism.
We emphasize that all experiments run the standard DP-SGD mechanism of
Algorithm~\ref{alg:batchsgd} without modification. Zero-out adjacency enters only the
privacy \emph{analysis} and does not alter the training algorithm
(cf.\ Section~\ref{subsec:fdp}). Under substitution adjacency the same conclusions hold
with the noise scale doubled, so the utility figures we report are, if anything,
optimistic for the substitution-based reading.

\paragraph{Datasets and Model Architectures.}
We evaluate our bounds on image (CIFAR-10/100~\cite{Krizhevsky09learningmultiple}, SVHN~\cite{37648}) and text (AG News~\cite{DBLP:conf/nips/ZhangZL15}) benchmarks\ifarxiv~\cite{lhoest-etal-2021-datasets}\fi. We employ standard architectures including ResNet~\cite{DBLP:conf/cvpr/HeZRS16}, Vision Transformers (ViT)~\cite{DBLP:conf/iclr/DosovitskiyB0WZ21}, and encoder-only Transformer architecture~\cite{DBLP:conf/nips/VaswaniSPUJGKP17}. To analyze how model capacity interacts with privacy noise, we evaluate variants ranging from \textbf{Tiny} to \textbf{Base} by increasing the \emph{embedding dimension} (the width of the internal vector representations) and depth. Detailed architectural specifications for all models are provided in \appref{app:implementation}{B}.

\paragraph{Experimental protocol and hyperparameter selection.}
For each epoch budget $E \in \{1,10,25\}$, we first fix a model architecture that is compatible with DP training (e.g., replacing non-DP-friendly components such as LayerNorm where necessary). Using this fixed model, we perform a hyperparameter search under \emph{clean training, i.e. $\sigma = 0$} to identify the best-performing optimizer configuration for each epoch budget.
These hyperparameters are then held fixed across all subsequent experiments. The clipping constant $C$ is selected by maximizing utility under a fixed noise multiplier set to our theoretical lower bound (i.e., $\sigma=1/\sqrt{2\ln M}$). We then evaluate two training regimes: (i) $\sigma=0$ without clipping (clean training) and (ii) full DP-SGD with clipping and Gaussian noise with $\sigma$ set to our theoretical lower bound. The purpose of these experiments is not to achieve peak accuracy, but to demonstrate that the noise levels implied by our lower bound are already severe in practice. Additional implementation and tuning details are provided in \appref{app:implementation}{B}.

\paragraph{Multi-view interpretation.}
Although our bounds are derived for a single epoch, this setting models federated learning scenarios where each client contributes one local epoch and the server observes many such client updates across rounds. Accordingly, we repeat the one-epoch mechanism over $E$ independent views,  corresponding to observing $E$ independent client updates. Our aim is not to optimize final accuracy, but to empirically illustrate the severity of the lower bound as the number of observed views increases.

\paragraph{Results and interpretation.}
Our experiments report test accuracy under shuffling and Poisson subsampling across batch
sizes and epoch budgets. Introducing Gaussian noise at our lower bound induces a
substantial and persistent utility gap relative to clean training. On CIFAR-10 with
ResNet-18 at batch size $128$, the noise floor $\sigma = 0.29$ reduces test accuracy from
$80.1\%$ to $40.4\%$ at $10$ epochs and from $82.2\%$ to $43.4\%$ at $25$
(Table~\ref{tab:cifar10-resnet18-merged}). On AG~News with a $136$M-parameter Transformer
at the same batch size, accuracy drops from $91.4\%$ to $82.6\%$ at $25$ epochs
(Table~\ref{tab:agnews-tx-small-bert-base-256-merged}). The gap is already pronounced at small epoch budgets and does not
close with more epochs, so the noise floor imposes a ceiling on optimization that more
training cannot lift. Shuffling and Poisson subsampling suffer comparable losses, consistent
with Theorems~\ref{thm:shuffling} and~\ref{thm:poisson}. Such degradation is well
documented~\cite{de2022unlockinghighaccuracydifferentiallyprivate,DBLP:conf/icml/SanderSS23,DBLP:conf/nips/BagdasaryanPS19,DBLP:journals/corr/abs-2012-07828}.
Our contribution is the explicit noise floor that provably induces it. Extended experimental results on additional model scales (Tiny, Small, and Base Transformers, ViT-Tiny/Small/Base, ResNet-34, and WideResNet-28x10) and datasets (AG News, CIFAR-10, CIFAR-100, SVHN) can be
found in \appref{app:extended-results}{C}, which also gives the complete batch-size sweeps for every table reported in this section.

\paragraph{Privacy accounting at the noise floor.}
Our theoretical results certify the limitation in $f$-DP terms. We also verify the
$(\varepsilon,\delta)$ implication with the independent shuffled-DP-SGD accountants of Chua
et al.~\cite{DBLP:conf/nips/ChuaGK0MSZ24}. Their shuffle accountant certifies a \emph{lower}
bound on $\varepsilon$ for the released update transcript and their deterministic-batching
accountant a valid \emph{upper} bound, and at our threshold the two agree to within $0.01$
for batch size $128$. For CIFAR-10 ($N = 50{,}000$, $\delta = 1/N = 2\times 10^{-5}$) at
batch size $128$ ($M = 390$) with noise at exactly our threshold
$\sigma = 1/\sqrt{2\ln M} \approx 0.29$, the accountant certifies $\varepsilon \ge 19.5$
after one epoch, $\varepsilon \ge 103.7$ after $10$, and $\varepsilon \ge 219.2$ after $25$.
Table~\ref{tab:accountant} reports these values across batch sizes and extends the same
computation to corpora as large as ImageNet-21k ($N = 14{,}197{,}122$), where the bound
tightens further, since larger $N$ gives more rounds per epoch at a fixed batch
size and a smaller $\delta = 1/N$. Every configuration at the threshold yields
$\varepsilon \gg 10$, reaching $\varepsilon \ge 28.8$ after a single epoch on ImageNet-21k,
far outside any regime commonly regarded as a meaningful worst-case
guarantee~\cite{DBLP:journals/jair/PonomarevaHKXDMVCT23}. Below our noise floor, then, the DP
framework cannot certify meaningful privacy in standard $(\varepsilon,\delta)$ language,
complementing the utility results above.

\begin{table}[tb]
\centering
\caption{\textbf{$(\varepsilon,\delta)$-DP at the noise floor, across dataset scales.}
Tight lower bounds on $\varepsilon$ from the shuffle accountant of Chua et
al.~\cite{DBLP:conf/nips/ChuaGK0MSZ24}, at one representative batch size per corpus, with
$\delta = 1/N$, $M = \lfloor N/\mathrm{BS}\rfloor$ and $\sigma = 1/\sqrt{2\ln M}$. $N$ counts records.}
\label{tab:accountant}
\small
\setlength{\tabcolsep}{2.8pt}
\begin{tabular}{lrrrrrr}
\toprule
Dataset & $N$ & BS & $M$ & $\sigma$ & $\varepsilon$ (1 ep) & $\varepsilon$ (25 ep) \\
\midrule
AG News~\cite{DBLP:conf/nips/ZhangZL15} & $120{,}000$ & $128$ & $937$ & $0.27$ & $22.1$ & $249.8$ \\
SVHN~\cite{37648} & $604{,}388$ & $128$ & $4{,}721$ & $0.24$ & $26.9$ & $306.2$ \\
ImageNet-1k & $1{,}281{,}167$ & $1{,}024$ & $1{,}251$ & $0.26$ & $24.6$ & $268.1$ \\
Places365~\cite{DBLP:journals/pami/ZhouLKO018} & $1{,}803{,}460$ & $1{,}024$ & $1{,}761$ & $0.26$ & $25.7$ & $280.1$ \\
ImageNet-21k~\cite{DBLP:conf/nips/RidnikBNZ21} & $14{,}197{,}122$ & $4{,}096$ & $3{,}466$ & $0.25$ & $28.8$ & $309.1$ \\
\bottomrule
\end{tabular}
\end{table}

\begin{table}[tb]
\centering
\small
\caption{CIFAR-10 / ResNet-18. Rows correspond to batch size (BS) and epoch; columns compare Random Shuffling vs.\ Poisson Subsampling under $\sigma{=}0$ and DP-SGD at the shown $\sigma$.}
\label{tab:cifar10-resnet18-merged}
\begin{tabular}{llrrrr}
\toprule
\multirow{2}{*}{BS} & \multirow{2}{*}{Epoch} & \multicolumn{2}{c}{shuffling} & \multicolumn{2}{c}{poisson} \\
 &  & $\sigma=0$ & DP ($\sigma$) & $\sigma=0$ & DP ($\sigma$) \\
\midrule
\multirow{3}{*}{128} & 1 & 48.0 & 38.0 (0.29) & 47.9 & 38.9 (0.29) \\
 & 10 & 80.1 & 40.4 (0.29) & 79.4 & 40.5 (0.29) \\
 & 25 & 82.2 & 43.4 (0.29) & 81.8 & 45.7 (0.29) \\
\midrule
\multirow{3}{*}{512} & 1 & 39.9 & 35.3 (0.33) & 42.6 & 36.2 (0.33) \\
 & 10 & 70.5 & 59.8 (0.33) & 71.9 & 60.3 (0.33) \\
 & 25 & 79.5 & 61.9 (0.33) & 79.0 & 62.7 (0.33) \\
\midrule
\multirow{3}{*}{4096} & 1 & 15.6 & 14.4 (0.45) & 15.1 & 14.5 (0.45) \\
 & 10 & 45.9 & 38.3 (0.45) & 46.9 & 40.9 (0.45) \\
 & 25 & 54.5 & 48.3 (0.45) & 57.6 & 50.0 (0.45) \\
\bottomrule
\end{tabular}
\end{table}

\begin{table}[tb]
\centering
\small
\caption{AG News / Transformer-Base ($\approx$136M parameters). Rows correspond to batch size (BS) and epoch; columns compare Random Shuffling vs.\ Poisson Subsampling under $\sigma{=}0$ and DP-SGD at the shown $\sigma$.}
\label{tab:agnews-tx-small-bert-base-256-merged}
\begin{tabular}{llrrrr}
\toprule
\multirow{2}{*}{BS} & \multirow{2}{*}{Epoch} & \multicolumn{2}{c}{shuffling} & \multicolumn{2}{c}{poisson} \\
 &  & $\sigma=0$ & DP ($\sigma$) & $\sigma=0$ & DP ($\sigma$) \\
\midrule
\multirow{3}{*}{128} & 1 & 86.8 & 74.7 (0.27) & 87.6 & 74.1 (0.27) \\
 & 10 & 90.7 & 53.9 (0.27) & 89.9 & 56.8 (0.27) \\
 & 25 & 91.4 & 82.6 (0.27) & 91.1 & 82.7 (0.27) \\
\midrule
\multirow{3}{*}{512} & 1 & 84.8 & 76.6 (0.30) & 82.9 & 75.2 (0.30) \\
 & 10 & 89.1 & 84.1 (0.30) & 89.6 & 83.6 (0.30) \\
 & 25 & 90.2 & 82.7 (0.30) & 90.3 & 82.3 (0.30) \\
\bottomrule
\end{tabular}
\end{table}

\section{Discussion and Future Directions}
\label{sec:discussion}

Our results expose a structural tension in DP-SGD under the standard worst-case adversarial model. This section discusses the implications of this limitation, how it relates to existing privacy frameworks, and which directions appear promising for overcoming or reinterpreting these constraints.

\paragraph{Is the adversary too strong?}
The lower bounds we derive are with respect to the standard worst-case DP adversary, who is allowed arbitrary side information and observes all noisy updates. From this perspective, our results may be interpreted less as an indictment of DP-SGD itself and more as a limitation of the adversarial model. This raises the question of whether different privacy notions such as instance-based PAC-privacy~\cite{DBLP:conf/crypto/XiaoD23} can meaningfully relax the trade-off without abandoning rigorous protection. An appealing direction is to understand how separation behaves when privacy is required only with high probability over the data distribution, rather than uniformly over all neighboring datasets. Separation provides a concrete yardstick for such comparisons. If an instance-based or distributional privacy notion yields a smaller separation $\kappa$ at the same noise level $\sigma$, the difference directly quantifies the gain obtained by relaxing the worst-case adversarial model.

\paragraph{Algorithmic alternatives to noise scaling.}
A key takeaway from both our theory and experiments is that increasing the noise multiplier alone is a blunt instrument. Improving privacy--utility trade-offs may require algorithmic changes rather than purely stronger noise calibration. Promising directions include modifying how gradients are clipped and aggregated, reducing the effective dimensionality or sparsity of updates before noise injection, and reconsidering training schedules such as stopping early or running only partial epochs. An open question is whether such interventions can qualitatively weaken the adversary and yield different privacy--utility behavior rather than merely shifting constants within the same worst-case regime. Recent results substantiate parts of this agenda. Random model and data partitioning yields provable amplification~\cite{DBLP:conf/icml/DongCO25}, and structured samplers that eliminate participation variance match or exceed Poisson subsampling~\cite{DBLP:conf/aistats/ChuaGHKKLMSZ25,DBLP:conf/nips/ShenfeldF25,lessrandommoreprivate}.

\paragraph{Multi-epoch behavior.}
Although our main lower bound is derived for a single epoch, multi-epoch training is the norm in practice. Existing asymptotic analyses within the $\mu$-GDP framework~\cite{DBLP:journals/corr/abs-1905-02383} suggest that privacy loss converges to a Gaussian limit under composition. However, such results control only the asymptotic error, leave the non-asymptotic constants open, and do not directly capture the separation metric that governs worst-case distinguishability (see \appref{sec:app-asymptotic}{F}). Understanding how separation evolves under repeated composition in the multi-epoch setting remains an open problem, and extending our separation bounds to this regime is an important direction for future work.

\paragraph{Implications for practice and the broader picture.}
Empirically, we observe that the noise levels implied by our lower bounds already induce substantial accuracy degradation at realistic batch sizes and model scales. Importantly, this does not imply that private learning is infeasible in practice; rather, it clarifies the concrete costs imposed by insisting on strong worst-case privacy guarantees within standard SGD-based training pipelines. More broadly, our results highlight that the combination of worst-case adversarial definitions and conventional optimization dynamics imposes real and sometimes underappreciated constraints on achievable privacy--utility trade-offs. A central open question is whether future progress will primarily come from new algorithmic designs, weaker but still meaningful privacy notions, or a principled combination of both. We view this work as a step toward understanding the boundary of what DP-SGD can and cannot provide under worst-case semantics, and toward motivating alternative frameworks.

\section{Conclusion}
\label{sec:conclusion}
We analyzed DP-SGD under realistic shuffling and Poisson subsampling through the $f$-DP hypothesis-testing framework and established a fundamental limitation under the standard worst-case adversarial model. The noise multiplier and worst-case distinguishability cannot be simultaneously driven below explicit lower bounds. This geometric limitation applies to both sampling schemes with empirical results confirming that the implied noise levels already cause substantial utility degradation at practical scales. Importantly, these findings should not be read as a negative statement about private learning per se; rather, they describe the boundary of what DP-SGD can achieve under strong worst-case definitions, and motivate the exploration of alternative algorithms, adaptive training mechanisms, or weaker but still meaningful privacy notions when this boundary proves too restrictive.

\section*{Ethics Considerations}
This is a theoretical paper regarding DP-SGD, so it introduces no attack and weakens no existing guarantee, and all experiments use public benchmarks with no human subjects or personal data.

\section*{Open Science}
All code, runs, checkpoints, and the scripts are archived on
Zenodo at \url{https://doi.org/10.5281/zenodo.21768773}.

\section*{Acknowledgments}
The contribution of the authors is part of the project CiCS of the research program Gravitation which is (partly) financed by the Dutch Research Council (NWO) under the grant 024.006.037. We acknowledge the use of the DAS-6 HPC cluster at VU Amsterdam~\cite{bal2016medium}.
ChatGPT and Claude were used for language revision.

\bibliographystyle{plainnat}
\bibliography{bib}
\newpage

\appendix
\ifarxiv
  
\section{Open Science Appendix}
\label{app:open-science-full}
Even though our core contribution is theoretical, the experiments provide supporting evidence. To facilitate independent verification, we provide a complete artifact to reproduce all empirical results reported in Section~\ref{sec:experiments} (including Table~\ref{tab:accountant}) and Appendices~\ref{app:extended-results} and~\ref{sec:app-asymptotic}\footnote{\url{https://doi.org/10.5281/zenodo.21768773}}. Our implementation is based on open-source frameworks and public datasets, and all experiments are run with Python~3.11 (the required packages are provided in the artifact). \href{https://doi.org/10.5281/zenodo.21768773}{Link to the Artifact}.

\subsection{Main Empirical Results (Section~\ref{sec:experiments} and Appendix~\ref{app:extended-results})}
All datasets used in our experiments are public and are downloaded automatically when running the code. Specifically, CIFAR-10, CIFAR-100, and SVHN are loaded via TensorFlow Datasets (TFDS).\footnote{\url{https://www.tensorflow.org/datasets}} AG News is loaded via \texttt{torchtext} if available, and otherwise via HuggingFace \texttt{datasets} (\texttt{ag\_news})~\cite{lhoest-etal-2021-datasets}. None of these require authentication. Detailed documentation can also be found in the \texttt{README.md} file in the artifact. One can reproduce our results in two ways:

\paragraph{Reproduce from existing runs (recommended).}
For each model/dataset pair and batch size, we provide a corresponding script in folder \texttt{reproduce\_scripts/} that calls \texttt{reproduce.py} on the appropriate run directories under \texttt{results/}. This script reads the \emph{saved} experimental specification from each run directory, including:
(i) the per-epoch hyperparameters used in that run (recorded per row in \texttt{utility\_sigma\_th.csv}),
(ii) the selected clipping constant $C^{\dagger}$,
(iii) the noise level $\sigma$ (derived from our bound),
and (iv) the deterministic seed plan in \texttt{seeds.json}.
It then re-runs training and evaluation and writes under the run directory: \emph{reproduce/}. 

\paragraph{Run from scratch (optional).}
To generate new run directories, reviewers can execute \texttt{run\_dp.py}, which (a) selects $C^{\dagger}$ using a 1-epoch clip-only procedure, and (b) evaluates two training conditions (clean training and DP-SGD) for each epoch budget in \{1, 10, 25\}. Each epoch is an independent run (trained from the same initialization scheme but with its own derived seed, and hyperparameters), and these runs are executed sequentially.

\paragraph{Hardware and determinism.}
Depending on dataset size, model architecture, and batch size, GPU memory requirements vary (approximately 16--24GB for some configurations and up to $\sim$80GB for others). While we aim for determinism via fixed seeds, minor numerical differences across GPU/driver stacks and mixed precision behavior may lead to negligible deviations in reproduced metrics.

\subsection{Figures in Appendix~\ref{sec:app-asymptotic}}
To reproduce the plots in Appendix~\ref{sec:app-asymptotic}, Python 3 with the \texttt{matplotlib}, \texttt{numpy}, and \texttt{scipy} libraries is required. The figures can be generated by running the script \texttt{appendix\_f.py} provided in the artifact above.

\section{Implementation Details}
\label{app:implementation}

This appendix provides a detailed description of our experimental pipeline, including the DP-SGD implementation, sampling mechanisms, microbatching strategy, and the neural network architectures used in our experiments. All implementations are designed to closely reflect practical DP-SGD deployments while preserving alignment with the theoretical assumptions used in the main body of the paper. One can find our artifact at \href{https://doi.org/10.5281/zenodo.21768773}{Artifact Link} (DOI \texttt{10.5281/zenodo.21768773}).

\subsection{Experimental Pipeline and DP-SGD Implementation}
\label{app:dpsgd}

All experiments are implemented in \texttt{JAX} using the official \emph{DP-SGD} primitives from the \emph{\texttt{jax-privacy}} library.\footnote{For DP-SGD computing, we use the \texttt{jax-privacy} implementation at version 1.0.0. For microbatching, we use the latest branch.}
We rely on the library’s reference implementation of DP-SGD without modifying its core components. We evaluate our bounds on standard image and text benchmarks, including CIFAR-10/100~\cite{Krizhevsky09learningmultiple} (50k training examples), %
SVHN~\cite{37648} ($\sim$604k training examples, the \texttt{train+extra} split), and the AG News corpus~\cite{DBLP:conf/nips/ZhangZL15} (120k training examples). 

\subsection{Microbatching}
\label{app:microbatching}

To efficiently compute per-example clipped gradients, we use the \emph{unchanged microbatching mechanism} provided by \texttt{jax-privacy}. Microbatching splits a logical batch into smaller microbatches, computes gradients for each microbatch, and accumulates clipped gradients incrementally.

Crucially, this microbatching strategy is an \emph{implementation optimization only}: it does not alter the effective privacy mechanism. Padding examples are explicitly handled by inserting dummy indices (marked by $-1$) whose gradients are zeroed out, ensuring that padding has no effect on either optimization or privacy.

Throughout all experiments, microbatching is used solely to reduce memory usage and compilation overhead, and does not change the noise scale, clipping bound, or privacy accounting.

\subsection{Sampling Mechanisms}
\label{app:sampling}
We explicitly control the batch formation mechanism to instantiate the sampling regimes studied in the main paper, up to standard practical considerations required for efficient accelerator execution.

Algorithms~\ref{alg:poisson} and~\ref{alg:randshuffle} give the two sampling mechanisms in
pseudocode. They follow Definitions~\ref{def:poisson-subsampling}
and~\ref{def:random-shuffling} of the main text step by step and add no further detail, so the
main text carries only the definitions.

\renewcommand{\thealgorithm}{$\mathcal{P}$}
\begin{algorithm}[tb]
\caption{Poisson Subsampling}
\label{alg:poisson}
\begin{algorithmic}[1]
\State \textbf{Input:} Dataset $d$ of size $N$, sampling rate $q \in (0,1)$, number of rounds $M$
\For{$j = 1, \ldots, M$}
    \State Initialize $S_j \gets \emptyset$
    \For{each index $i \in [N]$}
        \State Include $i$ in $S_j$ independently with probability $q$
    \EndFor
\EndFor
\State \textbf{Output:} Sequence of index sets $(S_1, \ldots, S_M)$
\end{algorithmic}
\end{algorithm}

\renewcommand{\thealgorithm}{$\mathcal{S}$}
\begin{algorithm}[tb]
\caption{Random Shuffling}
\label{alg:randshuffle}
\begin{algorithmic}[1]
\State \textbf{Input:} Dataset $d$ of size $N$, number of rounds $M$
\State Sample a random permutation $\pi$ over $[N]$
\State Let $b \gets \lfloor N/M \rfloor$
\For{$j = 1, \ldots, M$}
    \State Define index set $S_j \gets \{\pi_{(j-1)b+1}, \ldots, \pi_{jb}\}$
\EndFor
\State \textbf{Output:} Sequence of index sets $(S_1, \ldots, S_M)$
\Comment{The last $N - Mb = N \bmod M$ indices of $\pi$ are discarded.}
\end{algorithmic}
\end{algorithm}

\paragraph{Fixed-size shuffling.}
In the shuffled regime, at the beginning of each epoch we draw a fresh uniform random permutation of the $N$ training examples and iterate through it in contiguous fixed-size batches of size $b$.
To keep step shapes constant, we discard the final incomplete remainder batch and thus use only the first $\lfloor N/b\rfloor \cdot b$ examples per epoch, yielding exactly $\lfloor N/b\rfloor$ updates per epoch.
This corresponds to the standard shuffled DP-SGD mechanism with equal-sized, disjoint batches, applied to an effective dataset of size $\lfloor N/b\rfloor \cdot b$.

\paragraph{Poisson subsampling.}
In the Poisson regime, at each step we include each example independently with probability $q$, yielding a randomly sized sampled index set.
For JIT compatibility, the resulting variable-length index sets are padded \emph{after sampling} with dummy indices (marked by $-1$) to a fixed compiled shape (a multiple of the microbatch size).
Padded indices are mapped to a safe index for data loading and are accompanied by an \texttt{is\_padding\_example} mask, ensuring that their contributions to the gradient are exactly zero.
This padding procedure also handles the corner case of empty sampled batches by producing a non-empty dummy batch of fixed shape.

\paragraph{Normalization under Poisson sampling.}
In our DP-SGD implementation, we normalize gradients and calibrate Gaussian noise using the \emph{expected} batch size $b = qN$ under Poisson subsampling, treating non-sampled examples as contributing zero gradient.
For shuffling-based training, each update operates on exactly $b$ examples, and normalization uses the realized (non-padding) batch size.

\subsection{Vision Transformer Architectures}
\label{app:vit}

For image classification experiments, we use a minimal Vision Transformer (ViT) adapted to $32 \times 32$ images (e.g., CIFAR-style datasets). Images are split into non-overlapping patches using a convolutional patch embedding layer with kernel size and stride equal to the patch size.

The embedded patches are flattened, prepended with a learnable \texttt{[CLS]} token, and combined with learned positional embeddings. The resulting sequence is processed by a stack of Transformer encoder blocks using a \emph{pre-layer normalization} design. Each block consists of multi-head self-attention followed by a feed-forward MLP with GELU activations and residual connections.

To simplify experiments, all dropout rates are set to zero. The final representation of the \texttt{[CLS]} token is passed through a linear classification head. We evaluate multiple ViT configurations (Tiny, Small, Base) that vary in depth, width, and number of attention heads, while keeping the architectural structure fixed.

\subsubsection{Vision Transformer Variants}
\label{app:vit-variants}

We evaluate multiple Vision Transformer (ViT) variants that share the same architectural template but differ in width, depth, and attention capacity. All variants operate on $32 \times 32$ images and use a fixed patch size of $4 \times 4$, resulting in $8 \times 8 = 64$ image patches per input.

\paragraph{Common Architecture.}
All ViT variants follow the same high-level structure:
\begin{itemize}
    \item Patch embedding via a convolutional layer with kernel size and stride equal to the patch size.
    \item A learnable \texttt{[CLS]} token prepended to the patch sequence.
    \item Learned absolute positional embeddings.
    \item A stack of Transformer encoder blocks with pre-layer normalization.
    \item Classification via the final \texttt{[CLS]} representation.
\end{itemize}

\paragraph{ViT-Tiny.}
The \texttt{vit\_tiny\_cifar} variant is designed to be lightweight and DP-friendly. It uses a narrow embedding dimension and few attention heads, reducing both parameter count and gradient variance.

\begin{itemize}
    \item Embedding dimension: $d_{\text{model}} = 192$
    \item Number of layers: $12$
    \item Attention heads: $3$
    \item MLP expansion ratio: $4$
\end{itemize}

\paragraph{ViT-Small.}
The \texttt{vit\_small\_cifar} variant increases representational capacity while retaining the same depth.

\begin{itemize}
    \item Embedding dimension: $d_{\text{model}} = 384$
    \item Number of layers: $12$
    \item Attention heads: $6$
    \item MLP expansion ratio: $4$
\end{itemize}

\paragraph{ViT-Base.}
The \texttt{vit\_base\_cifar} variant mirrors standard ViT-Base scaling laws adapted to small images.

\begin{itemize}
    \item Embedding dimension: $d_{\text{model}} = 768$
    \item Number of layers: $12$
    \item Attention heads: $12$
    \item MLP expansion ratio: $4$
\end{itemize}

\paragraph{ViT Architecture Schematic.}
For clarity, the shared architecture of all ViT variants is summarized below:

\begin{verbatim}
Input Image (32x32xC)
   ↓ Patch Conv (4x4, stride 4)
64 Patch Embeddings (d_model)
   ↓
[CLS] Token + Positional Embedding
   ↓
Transformer Encoder × L layers
   ↓
LayerNorm
   ↓
CLS Token
   ↓
Linear Classifier
\end{verbatim}

\subsection{Text Transformer Architectures}
\label{app:text}

For text classification tasks, we use encoder-only Transformer models with RMS normalization and SwiGLU feed-forward layers. Input tokens are embedded using learned token and positional embeddings and processed by a stack of Transformer blocks.

Each block employs RMSNorm~\cite{DBLP:conf/nips/ZhangS19a} instead of LayerNorm to match common practice in modern Transformer implementations. The feed-forward sublayer uses a SwiGLU activation, which empirically improves optimization stability under noisy gradients. As with the vision models, dropout is disabled throughout.

Attention masks are used to prevent padded tokens from contributing to attention scores. Classification is performed by pooling the representation of the first token and applying a linear output head. We evaluate multiple model scales by varying depth, hidden dimension, and sequence length.

\subsubsection{Tokenization and Vocabulary Construction}
\label{app:tokenization}
To strictly isolate the impact of privacy noise on the optimization process, we train all text models from scratch with random initialization rather than fine-tuning pretrained language models (e.g., BERT~\cite{DBLP:conf/naacl/DevlinCLT19}). Hence, we employ a standalone tokenization and vocabulary construction pipeline instead of relying on pretrained tokenizers.

We use a simple frequency-based tokenizer with a fixed vocabulary size of $V = 30{,}000$. Text is first lowercased and tokenized via whitespace splitting. The vocabulary is then constructed by selecting the $V$ most frequent tokens in the training corpus. All remaining tokens are mapped to a generic \texttt{[UNK]} symbol, and sequences are padded with a dedicated \texttt{[PAD]} token. A special \texttt{[CLS]} token is prepended to each sequence for classification.

\subsubsection{Text Transformer Variants}
\label{app:text-variants}

For text classification, we evaluate a family of Transformer encoder models that differ in sequence length, width, and depth. All models share the same architectural backbone and are trained from scratch.

\paragraph{Common Architecture.}
All text models consist of:
\begin{itemize}
    \item Token and positional embeddings.
    \item A stack of Transformer blocks with pre-normalization using RMSNorm.
    \item Multi-head self-attention followed by a SwiGLU feed-forward network.
    \item Classification via the first (CLS) token.
\end{itemize}

\paragraph{Tiny-128.}
The \texttt{tiny\_128} model is a compact Transformer.

\begin{itemize}
    \item Maximum sequence length: $128$
    \item Embedding dimension: $256$
    \item Number of layers: $4$
    \item Attention heads: $4$
    \item Feed-forward dimension: $1024$
\end{itemize}

\paragraph{Small-128.}
The \texttt{small\_128} variant increases depth and width while keeping the same sequence length.

\begin{itemize}
    \item Maximum sequence length: $128$
    \item Embedding dimension: $512$
    \item Number of layers: $6$
    \item Attention heads: $8$
    \item Feed-forward dimension: $2048$
\end{itemize}

\paragraph{Base-256}
\begin{itemize}
    \item Embedding dimension: $768$
    \item Number of layers: $12$
    \item Attention heads: $12$
    \item Feed-forward dimension: $3072$
    \item Maximum sequence length: $256$
\end{itemize}

Increasing sequence length substantially raises computational cost and attention sensitivity.

\paragraph{Text Transformer Architecture Schematic.}
The shared structure of all text Transformer variants is summarized below:

\begin{verbatim}
Input Tokens (B x T)
   ↓ Token + Positional Embedding
   ↓
[RMSNorm → Self-Attention → Residual]
[RMSNorm → SwiGLU FFN → Residual]
        × L layers
   ↓
RMSNorm
   ↓
CLS Token
   ↓
Linear Classifier
\end{verbatim}

\subsection{Convolutional Architectures}
\label{app:cnn}

For convolutional image classification experiments, we use ResNet-style and Wide ResNet architectures adapted to small-resolution images (e.g., CIFAR-10, CIFAR-100, SVHN). All convolutional models are trained from scratch and use normalization layers that do not rely on batch-dependent statistics, making them compatible with DP-SGD.

\paragraph{Normalization.}
By default, we replace Batch Normalization with Group Normalization. Specifically, we use an adaptive GroupNorm variant that selects the largest valid number of groups not exceeding $32$ based on the channel dimension. This choice avoids batch-dependent running statistics while remaining close in spirit to standard GroupNorm. BatchNorm is not used in DP experiments.

\paragraph{CIFAR-style stem.}
All convolutional models use a CIFAR-style stem consisting of a single $3\times3$ convolution with stride $1$ and no max-pooling, followed by normalization and ReLU. 

\subsubsection{ResNet Architectures}
\label{app:resnet}

We evaluate ResNet-18 and ResNet-34 architectures adapted to CIFAR-style inputs. Both models follow the standard residual block structure with two $3\times3$ convolutions per block and identity skip connections, with projection shortcuts used when the spatial resolution or channel dimension changes.

\paragraph{ResNet-18.}
ResNet-18 consists of four stages with $\{2,2,2,2\}$ residual blocks and channel widths $\{64,128,256,512\}$. Spatial downsampling is performed by strided convolutions at the beginning of each stage (except the first). Global average pooling is applied before a linear classification head.

\paragraph{ResNet-34.}
ResNet-34 increases depth using $\{3,4,6,3\}$ residual blocks while keeping the same stage-wise channel configuration. This model is used primarily on CIFAR-100 to evaluate the effect of increased depth under DP-SGD.

\paragraph{Activation and regularization.}
All ResNet variants use ReLU activations. Dropout is disabled throughout to avoid introducing additional sources of randomness beyond DP noise.

\subsubsection{Wide ResNet Architectures}
\label{app:wideresnet}

To study the effect of increased width, we evaluate a Wide ResNet with depth $28$ and widening factor $10$ (WideResNet-28$\times$10). This architecture follows the standard Wide ResNet design, consisting of $6n+4$ layers grouped into three stages with increasing channel widths.

Each residual block contains two $3\times3$ convolutions with normalization and ReLU activations, and projection shortcuts are used when dimensions mismatch. As with standard Wide ResNets, widening increases representational capacity without increasing depth.

WideResNet-28$\times$10 is evaluated on SVHN, where wider architectures are known to perform well, allowing us to examine the interaction between model width and DP-SGD noise.

\paragraph{Architectural schematic.}
The common structure of the convolutional models is summarized below:

\begin{verbatim}
Input Image (32x32xC)
   ↓
3x3 Conv + Norm + ReLU
   ↓
Residual Blocks × stages
   ↓
Global Average Pooling
   ↓
Linear Classifier
\end{verbatim}

\subsection{Reproducibility and Determinism}
\label{app:reproducibility}

All experiments use explicitly seeded random number generators for sampling, shuffling, and noise generation. Batch indices, padding behavior, and microbatch ordering are deterministic given the seed. This ensures reproducibility across runs and enables controlled comparisons between Poisson subsampling and shuffling-based regimes.

\section{Extended Experimental Results}
\label{app:extended-results}
Tables~\ref{tab:agnews-tx-small-tiny-128-merged}--\ref{tab:svhn-wideresnet28x10-merged} present the complete evaluation across all architecture variants (Transformer-Tiny/Small, ViT-Tiny/Small/Base, WideResNet-28$\times$10, and ResNet-34) and datasets (SVHN, AG News, CIFAR-10, and CIFAR-100) omitted from the main text.
In nearly all settings, we observe the same behavior as discussed in Section~\ref{sec:experiments}, and it is visible even for large-capacity models such as WideResNet-28$\times$10 and ViT-Base. The exceptions are a small number of one-epoch configurations in which the non-private baseline is itself close to chance, together with the largest-budget AG~News / Transformer-Tiny setting, where the gap closes. Moreover, random shuffling and Poisson subsampling consistently yield comparable utility under DP, indicating that the observed degradation is not an artifact of the sampling scheme but rather a consequence of the underlying privacy constraints. These extended results reinforce our central claim that the privacy--utility trade-off identified by our analysis manifests broadly across architectures and datasets.

Tables~\ref{tab:cifar10-resnet18-full} and~\ref{tab:agnews-tx-base-full} give the complete
batch-size sweeps for the two configurations of Section~\ref{sec:experiments}, where only
representative batch sizes are reported.

\begin{table}[htbp]
\centering
\small
\caption{CIFAR-10 / ResNet-18, complete batch-size sweep. Conventions as in the main text.}
\label{tab:cifar10-resnet18-full}
\begin{tabular}{llrrrr}
\toprule
\multirow{2}{*}{BS} & \multirow{2}{*}{Epoch} & \multicolumn{2}{c}{shuffling} & \multicolumn{2}{c}{poisson} \\
 &  & $\sigma=0$ & DP ($\sigma$) & $\sigma=0$ & DP ($\sigma$) \\
\midrule
\multirow{3}{*}{128} & 1 & 48.0 & 38.0 (0.29) & 47.9 & 38.9 (0.29) \\
 & 10 & 80.1 & 40.4 (0.29) & 79.4 & 40.5 (0.29) \\
 & 25 & 82.2 & 43.4 (0.29) & 81.8 & 45.7 (0.29) \\
\midrule
\multirow{3}{*}{256} & 1 & 44.9 & 40.8 (0.31) & 47.6 & 39.7 (0.31) \\
 & 10 & 68.6 & 50.0 (0.31) & 77.8 & 46.4 (0.31) \\
 & 25 & 78.7 & 41.7 (0.31) & 79.2 & 46.4 (0.31) \\
\midrule
\multirow{3}{*}{512} & 1 & 39.9 & 35.3 (0.33) & 42.6 & 36.2 (0.33) \\
 & 10 & 70.5 & 59.8 (0.33) & 71.9 & 60.3 (0.33) \\
 & 25 & 79.5 & 61.9 (0.33) & 79.0 & 62.7 (0.33) \\
\midrule
\multirow{3}{*}{1024} & 1 & 29.5 & 28.6 (0.36) & 32.6 & 28.7 (0.36) \\
 & 10 & 60.4 & 56.7 (0.36) & 61.2 & 57.5 (0.36) \\
 & 25 & 77.9 & 67.0 (0.36) & 72.8 & 65.7 (0.36) \\
\midrule
\multirow{3}{*}{2048} & 1 & 27.8 & 23.2 (0.40) & 29.0 & 23.5 (0.40) \\
 & 10 & 53.7 & 47.6 (0.40) & 55.3 & 46.9 (0.40) \\
 & 25 & 71.4 & 61.9 (0.40) & 69.2 & 61.9 (0.40) \\
\midrule
\multirow{3}{*}{4096} & 1 & 15.6 & 14.4 (0.45) & 15.1 & 14.5 (0.45) \\
 & 10 & 45.9 & 38.3 (0.45) & 46.9 & 40.9 (0.45) \\
 & 25 & 54.5 & 48.3 (0.45) & 57.6 & 50.0 (0.45) \\
\bottomrule
\end{tabular}
\end{table}

\begin{table}[htbp]
\centering
\small
\caption{AG News / Transformer-Base ($\approx$136M parameters), complete batch-size sweep. Conventions as in the main text.}
\label{tab:agnews-tx-base-full}
\begin{tabular}{llrrrr}
\toprule
\multirow{2}{*}{BS} & \multirow{2}{*}{Epoch} & \multicolumn{2}{c}{shuffling} & \multicolumn{2}{c}{poisson} \\
 &  & $\sigma=0$ & DP ($\sigma$) & $\sigma=0$ & DP ($\sigma$) \\
\midrule
\multirow{3}{*}{128} & 1 & 86.8 & 74.7 (0.27) & 87.6 & 74.1 (0.27) \\
 & 10 & 90.7 & 53.9 (0.27) & 89.9 & 56.8 (0.27) \\
 & 25 & 91.4 & 82.6 (0.27) & 91.1 & 82.7 (0.27) \\
\midrule
\multirow{3}{*}{256} & 1 & 86.8 & 76.1 (0.29) & 85.9 & 75.9 (0.29) \\
 & 10 & 91.0 & 27.1 (0.29) & 90.4 & 25.9 (0.29) \\
 & 25 & 90.8 & 76.3 (0.29) & 89.5 & 73.4 (0.29) \\
\midrule
\multirow{3}{*}{512} & 1 & 84.8 & 76.6 (0.30) & 82.9 & 75.2 (0.30) \\
 & 10 & 89.1 & 84.1 (0.30) & 89.6 & 83.6 (0.30) \\
 & 25 & 90.2 & 82.7 (0.30) & 90.3 & 82.3 (0.30) \\
\bottomrule
\end{tabular}
\end{table}

\begin{table}[htbp]
\centering
\small
\caption{AG News / Transformer-Tiny ($\approx$12M parameters). Rows correspond to batch size (BS) and epoch budget; columns compare Random Shuffling vs.\ Poisson Subsampling under $\sigma{=}0$ and DP-SGD at the shown $\sigma$.}
\label{tab:agnews-tx-small-tiny-128-merged}
\begin{tabular}{llrrrr}
\toprule
\multirow{2}{*}{BS} & \multirow{2}{*}{Epoch} & \multicolumn{2}{c}{shuffling} & \multicolumn{2}{c}{poisson} \\
 &  & $\sigma=0$ & DP ($\sigma$) & $\sigma=0$ & DP ($\sigma$) \\
\midrule
\multirow{3}{*}{128} & 1 & 87.5 & 77.3 (0.27) & 86.5 & 76.1 (0.27) \\
 & 10 & 90.9 & 83.3 (0.27) & 90.2 & 83.6 (0.27) \\
 & 25 & 91.2 & 81.1 (0.27) & 90.3 & 79.3 (0.27) \\
\midrule
\multirow{3}{*}{256} & 1 & 87.0 & 79.1 (0.29) & 87.3 & 78.0 (0.29) \\
 & 10 & 91.2 & 83.6 (0.29) & 90.3 & 80.6 (0.29) \\
 & 25 & 90.2 & 85.8 (0.29) & 91.4 & 85.2 (0.29) \\
\midrule
\multirow{3}{*}{512} & 1 & 85.3 & 80.1 (0.30) & 84.1 & 78.5 (0.30) \\
 & 10 & 90.3 & 84.2 (0.30) & 87.6 & 84.3 (0.30) \\
 & 25 & 91.1 & 85.2 (0.30) & 90.0 & 85.4 (0.30) \\
\midrule
\multirow{3}{*}{1024} & 1 & 80.5 & 75.1 (0.32) & 81.9 & 74.8 (0.32) \\
 & 10 & 87.3 & 86.4 (0.32) & 86.7 & 85.8 (0.32) \\
 & 25 & 90.4 & 86.3 (0.32) & 90.1 & 87.0 (0.32) \\
\bottomrule
\end{tabular}
\end{table}

\begin{table}[htbp]
\centering
\small
\caption{AG News / Transformer-Small ($\approx$40M parameters). Rows correspond to batch size (BS) and epoch budget; columns compare Random Shuffling vs.\ Poisson Subsampling under $\sigma{=}0$ and DP-SGD at the shown $\sigma$.}
\label{tab:agnews-tx-small-small-128-merged}
\begin{tabular}{llrrrr}
\toprule
\multirow{2}{*}{BS} & \multirow{2}{*}{Epoch} & \multicolumn{2}{c}{shuffling} & \multicolumn{2}{c}{poisson} \\
 &  & $\sigma=0$ & DP ($\sigma$) & $\sigma=0$ & DP ($\sigma$) \\
\midrule
\multirow{3}{*}{128} & 1 & 87.9 & 76.5 (0.27) & 87.4 & 75.3 (0.27) \\
 & 10 & 91.3 & 71.4 (0.27) & 90.4 & 70.3 (0.27) \\
 & 25 & 90.8 & 79.0 (0.27) & 91.0 & 77.0 (0.27) \\
\midrule
\multirow{3}{*}{256} & 1 & 87.5 & 77.3 (0.29) & 86.6 & 75.7 (0.29) \\
 & 10 & 90.9 & 81.8 (0.29) & 90.0 & 81.4 (0.29) \\
 & 25 & 90.0 & 66.4 (0.29) & 90.3 & 64.2 (0.29) \\
\midrule
\multirow{3}{*}{512} & 1 & 86.2 & 77.8 (0.30) & 84.6 & 76.0 (0.30) \\
 & 10 & 90.3 & 83.7 (0.30) & 90.4 & 83.2 (0.30) \\
 & 25 & 90.9 & 83.0 (0.30) & 90.7 & 83.6 (0.30) \\
\midrule
\multirow{3}{*}{1024} & 1 & 82.1 & 75.1 (0.32) & 81.2 & 75.0 (0.32) \\
 & 10 & 89.5 & 86.7 (0.32) & 89.7 & 86.8 (0.32) \\
 & 25 & 87.5 & 87.8 (0.32) & 87.5 & 87.6 (0.32) \\
\bottomrule
\end{tabular}
\end{table}

\begin{table}[htbp]
\centering
\small
\caption{CIFAR-10 / ViT-Tiny ($\approx$5M parameters). Rows correspond to batch size (BS) and epoch budget; columns compare Random Shuffling vs.\ Poisson Subsampling under $\sigma{=}0$ and DP-SGD at the shown $\sigma$.}
\label{tab:cifar10-vit-tiny-cifar-merged}
\begin{tabular}{llrrrr}
\toprule
\multirow{2}{*}{BS} & \multirow{2}{*}{Epoch} & \multicolumn{2}{c}{shuffling} & \multicolumn{2}{c}{poisson} \\
 &  & $\sigma=0$ & DP ($\sigma$) & $\sigma=0$ & DP ($\sigma$) \\
\midrule
\multirow{3}{*}{128} & 1 & 38.8 & 38.9 (0.29) & 43.8 & 38.1 (0.29) \\
 & 10 & 57.8 & 41.8 (0.29) & 58.3 & 40.8 (0.29) \\
 & 25 & 67.6 & 48.8 (0.29) & 66.1 & 48.5 (0.29) \\
\midrule
\multirow{3}{*}{256} & 1 & 38.5 & 39.3 (0.31) & 42.6 & 40.7 (0.31) \\
 & 10 & 58.4 & 31.9 (0.31) & 57.1 & 34.8 (0.31) \\
 & 25 & 68.1 & 29.4 (0.31) & 65.4 & 37.1 (0.31) \\
\midrule
\multirow{3}{*}{512} & 1 & 39.5 & 38.4 (0.33) & 40.2 & 39.3 (0.33) \\
 & 10 & 57.8 & 46.6 (0.33) & 56.5 & 45.4 (0.33) \\
 & 25 & 62.7 & 52.8 (0.33) & 59.6 & 52.0 (0.33) \\
\midrule
\multirow{3}{*}{1024} & 1 & 35.3 & 35.3 (0.36) & 35.7 & 35.8 (0.36) \\
 & 10 & 55.6 & 50.5 (0.36) & 54.4 & 50.5 (0.36) \\
 & 25 & 61.8 & 56.5 (0.36) & 60.9 & 55.0 (0.36) \\
\bottomrule
\end{tabular}
\end{table}

\begin{table}[htbp]
\centering
\small
\caption{CIFAR-10 / ViT-Small ($\approx$21M parameters). Rows correspond to batch size (BS) and epoch budget; columns compare Random Shuffling vs.\ Poisson Subsampling under $\sigma{=}0$ and DP-SGD at the shown $\sigma$.}
\label{tab:cifar10-vit-small-cifar-merged}
\begin{tabular}{llrrrr}
\toprule
\multirow{2}{*}{BS} & \multirow{2}{*}{Epoch} & \multicolumn{2}{c}{shuffling} & \multicolumn{2}{c}{poisson} \\
 &  & $\sigma=0$ & DP ($\sigma$) & $\sigma=0$ & DP ($\sigma$) \\
\midrule
\multirow{3}{*}{128} & 1 & 44.7 & 38.1 (0.29) & 45.8 & 39.0 (0.29) \\
 & 10 & 61.0 & 43.5 (0.29) & 59.1 & 36.0 (0.29) \\
 & 25 & 64.0 & 51.1 (0.29) & 62.5 & 43.3 (0.29) \\
\midrule
\multirow{3}{*}{256} & 1 & 41.9 & 40.3 (0.31) & 43.3 & 40.0 (0.31) \\
 & 10 & 57.7 & 36.8 (0.31) & 57.4 & 38.7 (0.31) \\
 & 25 & 63.4 & 27.8 (0.31) & 62.7 & 36.6 (0.31) \\
\midrule
\multirow{3}{*}{512} & 1 & 39.4 & 37.9 (0.33) & 42.3 & 38.2 (0.33) \\
 & 10 & 56.1 & 50.5 (0.33) & 56.9 & 50.3 (0.33) \\
 & 25 & 63.4 & 51.7 (0.33) & 60.9 & 47.0 (0.33) \\
\midrule
\multirow{3}{*}{1024} & 1 & 37.1 & 33.9 (0.36) & 36.8 & 33.5 (0.36) \\
 & 10 & 55.6 & 49.7 (0.36) & 55.1 & 49.3 (0.36) \\
 & 25 & 61.1 & 55.4 (0.36) & 60.7 & 54.4 (0.36) \\
\bottomrule
\end{tabular}
\end{table}

\begin{table}[htbp]
\centering
\small
\caption{CIFAR-10 / ViT-Base ($\approx$85M parameters). Rows correspond to batch size (BS) and epoch budget; columns compare Random Shuffling vs.\ Poisson Subsampling under $\sigma{=}0$ and DP-SGD at the shown $\sigma$.}
\label{tab:cifar10-vit-base-cifar-merged}
\begin{tabular}{llrrrr}
\toprule
\multirow{2}{*}{BS} & \multirow{2}{*}{Epoch} & \multicolumn{2}{c}{shuffling} & \multicolumn{2}{c}{poisson} \\
 &  & $\sigma=0$ & DP ($\sigma$) & $\sigma=0$ & DP ($\sigma$) \\
\midrule
\multirow{3}{*}{128} & 1 & 42.4 & 39.1 (0.29) & 43.9 & 38.6 (0.29) \\
 & 10 & 60.3 & 36.5 (0.29) & 59.0 & 33.5 (0.29) \\
 & 25 & 61.9 & 9.8 (0.29) & 61.1 & 10.3 (0.29) \\
\midrule
\multirow{3}{*}{256} & 1 & 37.9 & 40.2 (0.31) & 41.5 & 40.2 (0.31) \\
 & 10 & 57.5 & 33.4 (0.31) & 55.9 & 31.6 (0.31) \\
 & 25 & 61.8 & 31.5 (0.31) & 63.2 & 31.8 (0.31) \\
\midrule
\multirow{3}{*}{512} & 1 & 42.2 & 38.8 (0.33) & 40.2 & 39.4 (0.33) \\
 & 10 & 56.1 & 48.9 (0.33) & 54.9 & 49.7 (0.33) \\
 & 25 & 56.0 & 50.2 (0.33) & 56.3 & 47.1 (0.33) \\
\midrule
\multirow{3}{*}{1024} & 1 & 36.8 & 34.7 (0.36) & 36.1 & 34.6 (0.36) \\
 & 10 & 55.5 & 50.8 (0.36) & 53.4 & 49.9 (0.36) \\
 & 25 & 56.5 & 52.0 (0.36) & 56.3 & 52.3 (0.36) \\
\bottomrule
\end{tabular}
\end{table}

\begin{table}[htbp]
\centering
\small
\caption{CIFAR-100 / ResNet-34. Rows correspond to batch size (BS) and epoch budget; columns compare Random Shuffling vs.\ Poisson Subsampling under $\sigma{=}0$ and DP-SGD at the shown $\sigma$.}
\label{tab:cifar100-resnet34-merged}
\begin{tabular}{llrrrr}
\toprule
\multirow{2}{*}{BS} & \multirow{2}{*}{Epoch} & \multicolumn{2}{c}{shuffling} & \multicolumn{2}{c}{poisson} \\
 &  & $\sigma=0$ & DP ($\sigma$) & $\sigma=0$ & DP ($\sigma$) \\
\midrule
\multirow{3}{*}{128} & 1 & 8.7 & 2.6 (0.29) & 8.7 & 2.5 (0.29) \\
 & 10 & 40.9 & 7.6 (0.29) & 42.8 & 7.8 (0.29) \\
 & 25 & 48.4 & 9.9 (0.29) & 48.5 & 10.6 (0.29) \\
\midrule
\multirow{3}{*}{256} & 1 & 7.1 & 3.2 (0.31) & 6.0 & 3.0 (0.31) \\
 & 10 & 37.6 & 13.7 (0.31) & 40.5 & 13.0 (0.31) \\
 & 25 & 48.1 & 14.4 (0.31) & 44.8 & 13.1 (0.31) \\
\midrule
\multirow{3}{*}{512} & 1 & 4.0 & 3.0 (0.33) & 4.6 & 3.1 (0.33) \\
 & 10 & 36.2 & 16.1 (0.33) & 37.2 & 16.7 (0.33) \\
 & 25 & 46.9 & 18.3 (0.33) & 46.7 & 20.0 (0.33) \\
\midrule
\multirow{3}{*}{1024} & 1 & 3.1 & 1.7 (0.36) & 2.6 & 2.1 (0.36) \\
 & 10 & 28.3 & 15.3 (0.36) & 28.5 & 15.6 (0.36) \\
 & 25 & 32.9 & 24.2 (0.36) & 35.7 & 21.9 (0.36) \\
\bottomrule
\end{tabular}
\end{table}

\begin{table}[htbp]
\centering
\small
\caption{SVHN / WideResNet-28x10 ($\approx$36M parameters). Rows correspond to batch size (BS) and epoch budget; columns compare Random Shuffling vs.\ Poisson Subsampling under $\sigma{=}0$ and DP-SGD at the shown $\sigma$.}
\label{tab:svhn-wideresnet28x10-merged}
\begin{tabular}{llrrrr}
\toprule
\multirow{2}{*}{BS} & \multirow{2}{*}{Epoch} & \multicolumn{2}{c}{shuffling} & \multicolumn{2}{c}{poisson} \\
 &  & $\sigma=0$ & DP ($\sigma$) & $\sigma=0$ & DP ($\sigma$) \\
\midrule
\multirow{3}{*}{128} & 1 & 94.9 & 85.9 (0.24) & 94.6 & 83.7 (0.24) \\
 & 10 & 97.7 & 89.0 (0.24) & 97.7 & 89.0 (0.24) \\
 & 25 & 97.7 & 89.8 (0.24) & 97.6 & 89.5 (0.24) \\
\midrule
\multirow{3}{*}{256} & 1 & 94.8 & 87.4 (0.25) & 95.0 & 85.4 (0.25) \\
 & 10 & 97.5 & 91.4 (0.25) & 97.4 & 91.4 (0.25) \\
 & 25 & 97.4 & 91.7 (0.25) & 97.3 & 92.3 (0.25) \\
\midrule
\multirow{3}{*}{512} & 1 & 92.3 & 88.7 (0.27) & 91.9 & 87.3 (0.27) \\
 & 10 & 97.2 & 92.0 (0.27) & 97.4 & 92.3 (0.27) \\
 & 25 & 97.4 & 93.2 (0.27) & 96.9 & 93.1 (0.27) \\
\bottomrule
\end{tabular}
\end{table}

\subsection{Per-Dataset \texorpdfstring{$(\varepsilon,\delta)$}{(eps,delta)}-DP Sweeps at the Noise Floor}
\label{app:accountant-sweeps}

Table~\ref{tab:accountant} in the main text reports one representative batch size per
corpus. Table~\ref{tab:accountant-full} gives the complete batch-size sweep for each
dataset, including CIFAR-10, which is omitted from the main table. Throughout,
$M = \lfloor N/\mathrm{BS} \rfloor$ is the number of rounds in one epoch,
$\sigma = 1/\sqrt{2\ln M}$ is the noise floor of Theorem~\ref{thm:shuffling}, and
$\delta = 1/N$. Values are \emph{lower} bounds on $\varepsilon$ from the shuffle
accountant of Chua et al.~\cite{DBLP:conf/nips/ChuaGK0MSZ24}. The matching
deterministic-batching upper bound agrees within $0.16$ in every cell of every block, and
within $0.03$ for $\mathrm{BS} \le 512$. For the five representative rows of
Table~\ref{tab:accountant} the two bounds agree within $0.03$ throughout, so the values
quoted there are tight.

$N$ counts records. For the first five corpora it is the size of the training split.
ImageNet-21k has no official train--validation
split~\cite{DBLP:conf/nips/RidnikBNZ21}, so its $N$ is the size of the whole corpus, which
is also the right count for $\delta = 1/N$.

Two patterns hold in every block. Increasing the batch size raises $\sigma$, because fewer
rounds per epoch means a weaker max-statistic, and therefore lowers $\varepsilon$. Increasing
the dataset size at a fixed batch size does the opposite. No configuration anywhere
in the grid reaches $\varepsilon \le 11$.

\begin{table}[htbp]
\centering
\caption{\textbf{$(\varepsilon,\delta)$-DP at the noise floor, complete batch-size sweeps.}
One block per dataset, in increasing order of $N$. Conventions as in
Table~\ref{tab:accountant}.}
\label{tab:accountant-full}
\begin{tabular}{rrrrrr}
\toprule
BS & $M$ & $\sigma$ & \multicolumn{3}{c}{$\varepsilon$}\\
\cmidrule(lr){4-6}
 & & & 1 epoch & 10 epochs & 25 epochs \\
\midrule
\multicolumn{6}{l}{\emph{CIFAR-10}~\cite{Krizhevsky09learningmultiple}, $N = 50{,}000$, $\delta = 2.0\times10^{-5}$} \\
$128$ & $390$ & $0.29$ & $19.5$ & $103.7$ & $219.2$ \\
$256$ & $195$ & $0.31$ & $18.0$ & $94.1$ & $197.6$ \\
$512$ & $97$ & $0.33$ & $16.4$ & $84.2$ & $175.6$ \\
$1{,}024$ & $48$ & $0.36$ & $14.7$ & $74.0$ & $153.0$ \\
$2{,}048$ & $24$ & $0.40$ & $12.9$ & $63.7$ & $130.3$ \\
$4{,}096$ & $12$ & $0.45$ & $11.1$ & $53.0$ & $106.9$ \\
\midrule
\multicolumn{6}{l}{\emph{AG News}~\cite{DBLP:conf/nips/ZhangZL15}, $N = 120{,}000$, $\delta = 8.3\times10^{-6}$} \\
$128$ & $937$ & $0.27$ & $22.1$ & $117.9$ & $249.8$ \\
$256$ & $468$ & $0.29$ & $20.6$ & $108.4$ & $228.3$ \\
$512$ & $234$ & $0.30$ & $19.0$ & $98.7$ & $206.6$ \\
$1{,}024$ & $117$ & $0.32$ & $17.4$ & $88.8$ & $184.6$ \\
\midrule
\multicolumn{6}{l}{\emph{SVHN} (train+extra)~\cite{37648}, $N = 604{,}388$, $\delta = 1.7\times10^{-6}$} \\
$128$ & $4{,}721$ & $0.24$ & $26.9$ & $144.2$ & $306.2$ \\
$256$ & $2{,}360$ & $0.25$ & $25.4$ & $134.8$ & $284.9$ \\
$512$ & $1{,}180$ & $0.27$ & $23.9$ & $125.2$ & $263.4$ \\
\midrule
\multicolumn{6}{l}{\emph{ImageNet-1k}~\cite{DBLP:conf/cvpr/DengDSLL009}, $N = 1{,}281{,}167$, $\delta = 7.8\times10^{-7}$} \\
$256$ & $5{,}004$ & $0.24$ & $27.7$ & $147.0$ & $311.2$ \\
$1{,}024$ & $1{,}251$ & $0.26$ & $24.6$ & $127.8$ & $268.1$ \\
$4{,}096$ & $312$ & $0.30$ & $21.4$ & $108.1$ & $224.1$ \\
$16{,}384$ & $78$ & $0.34$ & $17.9$ & $87.6$ & $178.9$ \\
\midrule
\multicolumn{6}{l}{\emph{Places365}~\cite{DBLP:journals/pami/ZhouLKO018}, $N = 1{,}803{,}460$, $\delta = 5.5\times10^{-7}$} \\
$256$ & $7{,}044$ & $0.24$ & $28.7$ & $152.6$ & $323.1$ \\
$1{,}024$ & $1{,}761$ & $0.26$ & $25.7$ & $133.4$ & $280.1$ \\
$4{,}096$ & $440$ & $0.29$ & $22.5$ & $113.8$ & $236.3$ \\
$16{,}384$ & $110$ & $0.33$ & $19.0$ & $93.4$ & $191.3$ \\
\midrule
\multicolumn{6}{l}{\emph{ImageNet-21k}~\cite{DBLP:conf/nips/RidnikBNZ21}, $N = 14{,}197{,}122$, $\delta = 7.0\times10^{-8}$} \\
$1{,}024$ & $13{,}864$ & $0.23$ & $31.9$ & $167.2$ & $352.5$ \\
$4{,}096$ & $3{,}466$ & $0.25$ & $28.8$ & $147.9$ & $309.1$ \\
$16{,}384$ & $866$ & $0.27$ & $25.5$ & $128.0$ & $265.0$ \\
$65{,}536$ & $216$ & $0.30$ & $22.0$ & $107.5$ & $219.8$ \\
\bottomrule
\end{tabular}
\end{table}

\section{Proof of Lemma~\ref{lem:shuf-bound-to-pois}}
\label{app:proof-poisson-lemma}
\label{app:proof-l64}
\begin{proof}[Proof of Lemma~\ref{lem:shuf-bound-to-pois}]
We upper bound the Poisson-subsampled trade-off curve by conditioning on the number of times the
distinguishing record is sampled and reducing the resulting experiment to a two-branch mixture.
First, recall from Section~\ref{sec:np-lemma} that for any (possibly randomized) test $\phi$,
\[
\alpha(\phi) = \mathbf{E}[\phi \mid d],
\qquad
\beta(\phi) = 1 - \mathbf{E}[\phi \mid d'],
\]
and the trade-off function is
\[
f(a) = \inf_{\phi:\,\alpha(\phi)\le a} \beta(\phi).
\]
\paragraph{Step 1: Conditioning on $K$ and coarsening $\{K\ge 1\}$.}
Let $K \sim \mathrm{Binomial}(M,q)$ denote the number of rounds in which the differentiating record
is included during the epoch, and let $p:=\Pr[K=0]=(1-q)^M$.

We write the Poisson experiment as a mixture over $K$.
On the event $\{K\ge 1\}$, the record is included at least once; to obtain an upper bound it is
convenient to \emph{coarsen} this event by replacing the conditional law given $\{K\ge 1\}$ by the
conditional law given $\{K=1\}$. Intuitively, this discards information because observing the record multiple times can only help
the adversary.
Thus, the true Poisson trade-off curve is upper bounded by the trade-off curve of the
two-branch mixture experiment with
\[
K=0 \text{ with probability } p,
\qquad
K=1 \text{ with probability } 1-p,
\]
where on the second branch we use the conditional law of the original mechanism given $K=1$.
We denote this mixture trade-off curve by $f_{\mathrm{mix}}$; then
\begin{equation}
\label{eq:pois-le-mix}
f_{\mathrm{pois}}(\alpha)\ \le\ f_{\mathrm{mix}}(\alpha)
\qquad \text{for all }\alpha\in[0,1].
\end{equation}

\paragraph{Step 2: Analysis of the $K=0$ branch (never sampled).}
We analyze the hypothesis-testing problem conditional on $K=0$, i.e., the event that the
distinguishing record is never included in any sampled batch over the $M$ rounds.

Under zero-out adjacency, the neighboring datasets satisfy
\[
d=(d\cap d')\cup\{\bot\},
\qquad
d'=(d\cap d')\cup\{\xi_N\},
\]
where $\bot$ is the ghost record with identically zero gradient contribution and $\xi_N$ is the real differing record. Conditioned on $K=0$, the distinguishing location (the $N$th record position) is never sampled. Hence, on the event $K=0$, the gradients actually computed in all rounds depend only on the common subset $d\cap d'$.

As a result, under both $H_0\equiv d$ and $H_1\equiv d'$, the joint distribution of the observed
outputs $(x_1,\dots,x_M)$ depends only on $d\cap d'$ and is identical under the two hypotheses.
Moreover, this distribution coincides with the null-hypothesis likelihood in the shuffling
experiment, namely the Gaussian likelihood given in Eq.~\eqref{eq:likelihood-h0}. Then, conditioned on $K=0$, the two hypotheses are statistically indistinguishable,
and the adversary can do no better than random guessing.

Then, for any (possibly randomized) rejection rule $\phi$, the conditional errors satisfy
\begin{align}
\alpha_0(\phi)
=
\mathbb{E}\!\left[\phi \,\middle|\, H_0,\ K=0\right]
=
\mathbb{E}\!\left[\phi \,\middle|\, d\cap d'\right],
\\
\beta_0(\phi)
=
\mathbb{E}\!\left[1-\phi \,\middle|\, H_1,\ K=0\right]
=
1-\mathbb{E}\!\left[\phi \,\middle|\, d\cap d'\right].
\end{align}
In particular, $\beta_0(\phi)=1-\alpha_0(\phi)$.

Finally, by zero-out adjacency, the presence or absence of the $N$th record does not affect the
shuffling experiment: shuffling on the full dataset $d$ induces the same distribution of
$(x_1,\dots,x_M)$ as shuffling on $d\cap d'$.
Therefore,
\[
\mathbb{E}\!\left[\phi \,\middle|\, d\cap d'\right]
=
\mathbb{E}\!\left[\phi \,\middle|\, d\right]
=
\alpha_{\mathrm{shuf}}(\phi),
\]
and we may identify the $K=0$-branch errors as
\[
\alpha_0(\phi)=\alpha_{\mathrm{shuf}}(\phi),
\qquad
\beta_0(\phi)=1-\alpha_{\mathrm{shuf}}(\phi).
\]

\paragraph{Step 3: Analysis of the $K=1$ branch (sampled exactly once).}
Now condition on $K=1$.
Then the distinguishing record is included in exactly one of the $M$ rounds.
By symmetry of Poisson subsampling across rounds, the (random) round index in which it is included is uniform over
$\{1,\dots,M\}$.

Consequently, conditional on $K=1$, the induced observation model is exactly the same as in the one-epoch shuffling
analysis where the distinguishing record $\xi_N$ is included in exactly one of the $M$ rounds: under $H_0$ no round contains a shifted contribution, while under $H_1$ exactly one round contains a shift of magnitude $1/\sigma$, uniformly over the $M$ possible rounds. Equivalently, the conditional likelihoods of the observed outputs $(x_1,\dots,x_M)$ coincide with the shuffling likelihoods in Eqs.~\eqref{eq:likelihood-h0}--\eqref{eq:likelihood-h1}.

Therefore, for every (possibly randomized) rejection rule $\phi$, the conditional type-I/type-II errors in the
$K=1$ branch equal the corresponding shuffling errors:
\begin{align}
\alpha_1(\phi)
:=
\Pr[\phi=1 \mid H_0,\ K=1]
=
\alpha_{\mathrm{shuf}}(\phi),
\\
\beta_1(\phi)
:=
\Pr[\phi=0 \mid H_1,\ K=1]
=
\beta_{\mathrm{shuf}}(\phi).
\end{align}

\paragraph{Step 4: Mixture errors expressed via shuffling quantities.}
Fix any (possibly randomized) rejection rule $\phi$ applied to the observed outputs $(x_1,\dots,x_M)$.
From Steps~2 and~3 we have the exact identifications
\[
\alpha_0(\phi)=\alpha_{\mathrm{shuf}}(\phi),
\qquad
\beta_0(\phi)=1-\alpha_{\mathrm{shuf}}(\phi),
\]
and
\[
\alpha_1(\phi)=\alpha_{\mathrm{shuf}}(\phi),
\qquad
\beta_1(\phi)=\beta_{\mathrm{shuf}}(\phi).
\]
Therefore, the mixture type-I and type-II errors satisfy
\[
\alpha_{\mathrm{mix}}(\phi)
=
p\,\alpha_{\mathrm{shuf}}(\phi)+(1-p)\,\alpha_{\mathrm{shuf}}(\phi)
=
\alpha_{\mathrm{shuf}}(\phi),
\]
and
\[
\beta_{\mathrm{mix}}(\phi)
=
p\bigl(1-\alpha_{\mathrm{shuf}}(\phi)\bigr)+(1-p)\,\beta_{\mathrm{shuf}}(\phi).
\]

By definition of the trade-off function,
\begin{align}
f_{\mathrm{mix}}(\alpha)
&=
\inf_{\phi:\ \alpha_{\mathrm{mix}}(\phi)\le \alpha}\ \beta_{\mathrm{mix}}(\phi)\nonumber\\
&=
\inf_{\phi:\ \alpha_{\mathrm{shuf}}(\phi)\le \alpha}
\Bigl[
p\bigl(1-\alpha_{\mathrm{shuf}}(\phi)\bigr)
+
(1-p)\,\beta_{\mathrm{shuf}}(\phi)
\Bigr].
\label{eq:f-mix-start}
\end{align}
\paragraph{Step 5: Reduction to the shuffling trade-off}
When minimizing $\beta_{\mathrm{shuf}}$ subject to the constraint
$\alpha_{\mathrm{shuf}}(\phi)\le\alpha$, the infimum is attained (or approached)
by rejection rules satisfying $\alpha_{\mathrm{shuf}}(\phi)=\alpha$. Hence both components of the objective
\[
p\bigl(1-\alpha_{\mathrm{shuf}}(\phi)\bigr)+(1-p)\beta_{\mathrm{shuf}}(\phi)
\]
are minimized by the same boundary choice $\alpha_{\mathrm{shuf}}(\phi)=\alpha$.
Therefore,
\begin{align*}
f_{\mathrm{mix}}(\alpha)
&=
\inf_{\phi:\ \alpha_{\mathrm{shuf}}(\phi)= \alpha}
\Bigl[p\bigl(1-\alpha_{\mathrm{shuf}}(\phi)\bigr)+(1-p)\beta_{\mathrm{shuf}}(\phi)\Bigr]\\
&=
p(1-\alpha)
+
(1-p)\inf_{\phi:\ \alpha_{\mathrm{shuf}}(\phi)= \alpha}\beta_{\mathrm{shuf}}(\phi).
\end{align*}

By the definition of the shuffling trade-off curve,
\[
f_{\mathrm{shuf}}(\alpha)
=
\inf_{\phi:\ \alpha_{\mathrm{shuf}}(\phi)= \alpha}\beta_{\mathrm{shuf}}(\phi),
\]
and therefore
\[
f_{\mathrm{mix}}(\alpha)=p(1-\alpha)+(1-p)f_{\mathrm{shuf}}(\alpha).
\]

\paragraph{Step 6: Concluding for Poisson.}
Finally, using the assumed bound
$f_{\mathrm{shuf}}(\alpha)\le f'_{\mathrm{shuf}}(\alpha)$, we obtain
\[
f_{\mathrm{mix}}(\alpha)
\le
p(1-\alpha)+(1-p)\,f'_{\mathrm{shuf}}(\alpha).
\]
Combining this with \eqref{eq:pois-le-mix} proves \eqref{eq:poisson-mixture-bound}.
\end{proof}

\section{From Separation to \texorpdfstring{$(\varepsilon,\delta)$}{(eps,delta)}-DP Guarantees}
\label{sec:app-kappa-to-eps-delta}

\begin{lemma}[$\mu$-GDP induces Gaussian separation]
\label{lem:kappa-of-mu}
If a mechanism is $\mu$-GDP~\cite{DBLP:journals/corr/abs-1905-02383}, i.e., its $f$-DP trade-off curve equals the Gaussian trade-off curve $G_\mu$,
\[
G_\mu(\alpha)=\Phi\!\bigl(\Phi^{-1}(1-\alpha)-\mu\bigr),
\]
then its separation satisfies
\begin{equation}
\operatorname{sep}(G_\mu)
=
\frac{1-2\alpha^*}{\sqrt2}
=
\frac{2\Phi(\mu/2)-1}{\sqrt2},
\label{eq:sep-of-mu}
\end{equation}
where $\alpha^*=\Phi(-\mu/2)$ is the unique fixed point of $G_\mu$.
In particular, the separation is monotone increasing in $\mu$.
\end{lemma}
\begin{proof}[Proof of Lemma~\ref{lem:kappa-of-mu}]
The fixed point condition $G_\mu(\alpha^*)=\alpha^*$ is
\[
\Phi^{-1}(\alpha^*)=\Phi^{-1}(1-\alpha^*)-\mu.
\]
Using symmetry $\Phi^{-1}(1-\alpha)=-\Phi^{-1}(\alpha)$, we obtain
\[
\mu=-2\,\Phi^{-1}(\alpha^*),
\qquad\text{hence}\qquad
\alpha^*=\Phi(-\mu/2).
\]
By Lemma~\ref{lem:sep-fixed-point}, $\operatorname{sep}(G_\mu)=\frac{1-2\alpha^*}{\sqrt2}$, which yields
\[
\operatorname{sep}(G_\mu)=\frac{1-2\Phi(-\mu/2)}{\sqrt2}
=\frac{2\Phi(\mu/2)-1}{\sqrt2}.
\]
Monotonicity follows since $\Phi$ is increasing.
\end{proof}

We now show how a lower bound on the separation 
$\kappa_{\mathrm{shuf}}$ 
of the one-epoch shuffled DP-SGD mechanism yields a corresponding
lower bound on any $(\varepsilon,\delta)$-DP guarantee assigned to the
same mechanism.  
Our argument proceeds in three steps:
(i) recall the canonical $(\varepsilon,\delta)$-DP trade-off curve 
$f_{\varepsilon,\delta}$; (ii) compute its separation $\operatorname{sep}(f_{\varepsilon,\delta})$ in
closed form; and  (iii) compare this separation to $\kappa_{\mathrm{shuf}}$.
Under the standard convention $\delta = 1/N$, this comparison produces a
\emph{minimum admissible privacy parameter}~$\varepsilon$ or equivalently, implies that the noise multiplier $\sigma$ must be larger than our bound.

\paragraph{The $(\varepsilon,\delta)$-DP trade-off $f_{\varepsilon,\delta}$.}
Following~\cite{DBLP:journals/corr/abs-1905-02383} (Proposition 2.4-2.5), any
$(\varepsilon,\delta)$-DP guarantee corresponds to the
piecewise-linear trade-off function
\begin{equation}
\label{eq:f-eps-delta-def}
f_{\varepsilon,\delta}(\alpha)
=
\max\Bigl\{
0,\,
1-\delta - e^{\varepsilon}\alpha,\,
e^{-\varepsilon}(1-\delta-\alpha)
\Bigr\},
\qquad \alpha\in[0,1].
\end{equation}
This curve is the \emph{tightest} possible $f$-DP trade-off compatible
with $(\varepsilon,\delta)$-DP: if a mechanism is $(\varepsilon,\delta)$-DP,
then its true trade-off curve under that mechanism must satisfy
\begin{equation}
\label{eq:f-shuf-above-fepsdelta}
f_{\mathrm{shuf}}(\alpha)
\;\ge\;
f_{\varepsilon,\delta}(\alpha)
\qquad\forall\,\alpha\in[0,1],
\end{equation}
since every hypothesis test available to the adversary must obey the
constraints induced by the $(\varepsilon,\delta)$-DP definition.

\begin{lemma}[Separation of the $(\varepsilon,\delta)$-DP trade-off]
\label{lem:sep-fepsdelta}
Let $f_{\varepsilon,\delta}$ be defined as in
\eqref{eq:f-eps-delta-def}. Then
\begin{equation}
\label{eq:kappa-eps-delta-final}
\kappa_{\varepsilon,\delta}
:=
sep(f_{\varepsilon,\delta})
=
\frac{e^{\varepsilon}-1+2\delta}{(1+e^{\varepsilon})\sqrt{2}}.
\end{equation}
\begin{proof}
Recall that the $(\varepsilon,\delta)$-DP trade-off function
$f_{\varepsilon,\delta}$ is symmetric and piecewise linear, with two linear
branches meeting at a unique fixed point. Explicitly,
\[
f_{\varepsilon,\delta}(\alpha)
=
\max\bigl\{
1-\delta-e^{\varepsilon}\alpha,\;
e^{-\varepsilon}(1-\delta-\alpha)
\bigr\}.
\]
We solve for a fixed point $\hat a$ satisfying
$f_{\varepsilon,\delta}(\hat a)=\hat a$.
On the left linear branch,
$f_{\varepsilon,\delta}(\alpha) = 1-\delta - e^{\varepsilon}\alpha$.
The fixed-point equation becomes
\[
\alpha = 1-\delta - e^{\varepsilon}\alpha,
\qquad
\hat a
= \frac{1-\delta}{1+e^{\varepsilon}}.
\]
Using the fixed-point formula
$sep(f)=(1-2\hat a)/\sqrt{2}$, we obtain
\begin{align*}
1-2\hat a
&= 1 - \frac{2(1-\delta)}{1+e^{\varepsilon}}
= \frac{e^{\varepsilon}-1+2\delta}{1+e^{\varepsilon}},
\end{align*}
yielding~\eqref{eq:kappa-eps-delta-final}. Notice that the same value of $\hat a$ is obtained when taking the right linear branch of $f_{\varepsilon,\delta}$, so the resulting separation is unchanged. We remark that this consistency is expected, as every symmetric convex trade-off function has exactly one fixed point.

\end{proof}
\end{lemma}

\paragraph{Relation to the $(\varepsilon,\delta)$ privacy region of Kairouz et al.~\cite{DBLP:conf/icml/KairouzOV15}.}
The quantity $\kappa_{\varepsilon,\delta}$ is consistent with the geometric picture of the
$(\varepsilon,\delta)$ privacy region in~\cite[Fig.~1]{DBLP:conf/icml/KairouzOV15}. The fixed point
$\hat a = (1-\delta)/(1+e^{\varepsilon})$ of $f_{\varepsilon,\delta}$ is exactly the corner
point of that region, and $\kappa_{\varepsilon,\delta}$
in~\eqref{eq:kappa-eps-delta-final} is its Euclidean distance to the random-guessing
diagonal. The two viewpoints differ in scope. The privacy region
of~\cite{DBLP:conf/icml/KairouzOV15} captures the geometry of a single fixed
$(\varepsilon,\delta)$ pair, whereas the separation $\kappa$ operates on the full
trade-off curve, which encodes all $(\varepsilon,\delta)$ pairs
simultaneously~\cite{DBLP:journals/corr/abs-1905-02383}. The mapping between them is
explicit. For a mechanism whose trade-off curve has separation $\kappa$, the fixed point
lies at $\bigl(\tfrac12 - \tfrac{\kappa}{\sqrt2},\, \tfrac12 -
\tfrac{\kappa}{\sqrt2}\bigr)$, and for each $\delta \in [0,\, \kappa\sqrt2]$ the smallest
admissible $\varepsilon(\delta)$ in Lemma~\ref{lem:eps-lower-from-kappa} is determined by
the slope $-e^{-\varepsilon(\delta)}$ of the line through $(1-\delta,\, 0)$ and this fixed
point. For $\delta > \kappa\sqrt2$ the constraint in
Lemma~\ref{lem:eps-lower-from-kappa} becomes vacuous.

\paragraph{Ordering of separations.}
Since the trade-off of the shuffled mechanism satisfies \eqref{eq:f-shuf-above-fepsdelta},
its pointwise separation from the random-guessing line $\beta = 1-\alpha$ must satisfy
\[
\operatorname{sep}_{f_{\mathrm{shuf}}}(\alpha)
\;\le\;
\operatorname{sep}_{f_{\varepsilon,\delta}}(\alpha)
\qquad\forall\,\alpha.
\]
Taking maxima over $\alpha$ and applying Lemma~\ref{lem:sep-fixed-point} gives
\begin{equation}
\label{eq:kappa-chain-main}
\kappa_{\mathrm{shuf}}
\;\le\;
\kappa_{\varepsilon,\delta}.
\end{equation}

On the other hand, Theorem~\ref{thm:shuffling} provides an explicit
lower bound on $\kappa_{\mathrm{shuf}}$: for one-epoch shuffling with  $\sigma < 1/\sqrt{2\ln M}$,
\[
\kappa_{\mathrm{shuf}}
\;\ge\;
\frac{1}{\sqrt{8}}\left(1 - \frac{1}{\sqrt{4\pi\ln M}}\right).
\]
Combining this with~\eqref{eq:kappa-chain-main}, any
$(\varepsilon,\delta)$-DP pair consistent with the shuffled mechanism
must satisfy
\[
\kappa_{\varepsilon,\delta}
\;\ge\;
\kappa_{\mathrm{shuf}}
\;\ge\;
\frac{1}{\sqrt{8}}
\!\left(1 - \frac{1}{\sqrt{4\pi\ln M}}\right).
\]

\begin{lemma}[Lower bound on $\varepsilon$ from separation]
\label{lem:eps-lower-from-kappa}
Let $\kappa_{\mathrm{shuf}} = \operatorname{sep}(f_{\mathrm{shuf}})$ denote 
the separation of the shuffled mechanism’s $f$-DP trade-off curve.  
If the same mechanism is also reported as $(\varepsilon,\delta)$-DP, then
$\varepsilon$ must satisfy
\begin{equation}
\label{eq:eps-lower-final}
\varepsilon
\;\ge\;
\ln\!\left(
\frac{1 + \kappa_{\mathrm{shuf}}\sqrt{2} - 2\delta}
     {1 - \kappa_{\mathrm{shuf}}\sqrt{2}}
\right).
\end{equation}
\end{lemma}

\begin{proof}
For any $(\varepsilon,\delta)$-DP guarantee, the induced trade-off curve
$f_{\varepsilon,\delta}$ has separation
$\kappa_{\varepsilon,\delta}$ as given in Lemma~\ref{lem:sep-fepsdelta}:
\[
\kappa_{\varepsilon,\delta}
= \frac{e^{\varepsilon}-1+2\delta}{(1+e^{\varepsilon})\sqrt{2}}.
\]
Since $f_{\mathrm{shuf}}(\alpha) \ge f_{\varepsilon,\delta}(\alpha)$ for
all $\alpha$, the separations must satisfy
$\kappa_{\varepsilon,\delta} \ge \kappa_{\mathrm{shuf}}$.
Thus
\[
\frac{e^{\varepsilon}-1+2\delta}{(1+e^{\varepsilon})\sqrt{2}}
\;\ge\;
\kappa_{\mathrm{shuf}}.
\]
Rearranging,
\[
e^{\varepsilon}(1-\kappa_{\mathrm{shuf}}\sqrt{2})
\;\ge\;
1 + \kappa_{\mathrm{shuf}}\sqrt{2} - 2\delta.
\]
Since $\kappa_{\mathrm{shuf}} < 1/\sqrt{2}$ ensures the coefficient of
$e^{\varepsilon}$ is positive, division yields
\[
e^{\varepsilon}
\;\ge\;
\frac{1 + \kappa_{\mathrm{shuf}}\sqrt{2} - 2\delta}
     {1 - \kappa_{\mathrm{shuf}}\sqrt{2}},
\]
and taking logarithms proves~\eqref{eq:eps-lower-final}.
\end{proof}

\emph{Remark.} The formula in~\eqref{eq:eps-lower-final} is well-defined because Lemma~\ref{lem:sep-fixed-point} ensures that $\kappa_{\mathrm{shuf}} < 1/\sqrt{2}$.

\paragraph{Interpretation and numerical illustration.}
Lemma~\ref{lem:eps-lower-from-kappa} shows that the geometric separation
$\kappa_{\mathrm{shuf}}$ imposes a hard constraint on any
$(\varepsilon,\delta)$-DP description of the same mechanism.  In typical
DP-SGD practice one fixes $\delta = 1/N$, so for a given separation
$\kappa$ the smallest admissible $\varepsilon$ is
\[
\varepsilon_{\min}(\kappa;N)
=
\ln\!\left(
\frac{1 + \kappa\sqrt{2} - 2/N}
     {1 - \kappa\sqrt{2}}
\right).
\]
Theorem~\ref{thm:shuffling} provides a \emph{lower bound} on the
shuffled separation $\kappa_{\mathrm{shuf}}$. Since $\varepsilon_{\min}(\kappa;N)$ is monotonically increasing in
$\kappa$, plugging in this lower bound yields the \emph{smallest}
$\varepsilon$ compatible with our one-epoch worst-case analysis; any
larger true separation would force a strictly larger $\varepsilon$. We note that while a smaller $\varepsilon$ is possible, it requires a larger $\sigma$ (violating the assumption of our lower bound), effectively hurting utility.

For Poisson subsampling, Theorem~\ref{thm:poisson} gives $\kappa_{\mathrm{pois}} \ge (1-1/e)\,\kappa_{\mathrm{shuf}}$, so the same formula with $\kappa = \kappa_{\mathrm{pois}}$ yields the corresponding $\varepsilon_{\min}$ for the Poisson regime.

Table~\ref{tab:kappa-eps} reports these \emph{best-case} values for a
representative choice $\delta = 1/N$ with $N = 10^8$.  Even at this
optimistic one-epoch baseline, shuffled DP-SGD already forces
$\varepsilon \approx 1$, while the Poisson regime inherits a comparable,
though slightly weaker, constraint. In other words, the worst-case adversary possesses a substantial advantage in distinguishing datasets, which makes the formal privacy guarantees insufficient for practical deployment. To bypass this limitation, future work must either move away from the standard worst-case adversarial framework or introduce fundamental algorithmic modifications.

\begin{table}[t]
\centering
\caption{
\textbf{Minimum $\varepsilon$ implied by separation bounds (one epoch, $\delta=1/N$, $N=10^8$).}
For each number of rounds $M$, we plug the lower bound
$\kappa_{\mathrm{shuf}}$ from Theorem~\ref{thm:shuffling} and the
induced Poisson bound
$\kappa_{\mathrm{pois}} \ge (1-1/e)\,\kappa_{\mathrm{shuf}}$
into Lemma~\ref{lem:eps-lower-from-kappa} to obtain the smallest
$\varepsilon$ consistent with our worst-case $f$-DP analysis.
Larger true separations would only increase these $\varepsilon$ values.
}
\label{tab:kappa-eps}
\begin{tabular}{rcccc}
\toprule
$M$
 & $\kappa_{\mathrm{shuf}}(M)$
 & $\varepsilon_{\min}^{\mathrm{shuf}}$
 & $\kappa_{\mathrm{pois}}(M)$
 & $\varepsilon_{\min}^{\mathrm{pois}}$
\\
\midrule
$1{,}000$ & $0.316$ & $0.96$ & $0.200$ & $0.58$ \\
$3{,}000$ & $0.318$ & $0.97$ & $0.201$ & $0.59$ \\
$10{,}000$ & $0.321$ & $0.98$ & $0.203$ & $0.59$ \\
$30{,}000$ & $0.322$ & $0.98$ & $0.204$ & $0.59$ \\
$100{,}000$ & $0.324$ & $0.99$ & $0.205$ & $0.60$ \\
$300{,}000$ & $0.325$ & $1.00$ & $0.206$ & $0.60$ \\
$1{,}000{,}000$ & $0.327$ & $1.00$ & $0.207$ & $0.60$ \\
$3{,}000{,}000$ & $0.328$ & $1.00$ & $0.207$ & $0.60$ \\
$5{,}000{,}000$ & $0.328$ & $1.01$ & $0.207$ & $0.60$ \\
\bottomrule
\end{tabular}
\end{table}

\section{Asymptotic Separation Bound for Poisson Subsampling}
\label{sec:app-asymptotic}
Before turning to our main results, we first examine what the asymptotic $\mu$-GDP theory~\cite{DBLP:journals/corr/abs-1905-02383} predicts for the separation metric under Poisson subsampling. The purpose of this section is not to derive new guarantees, but rather to relate existing $\mu$-GDP results to our geometric notion of separation, to analyze their implications in the proportional-growth regime, and to motivate the form of our subsequent bounds.

Edgeworth-type expansions sharpen these CLT-based approximations with improved asymptotic
error rates~\cite{DBLP:conf/icml/ZhengDLS20}. They too control only the asymptotic error and
supply no non-asymptotic constants, and neither line of work addresses separation directly.

Concretely, we translate the resulting $\mu$-GDP characterization into an explicit prediction for separation. As a simple illustration, we will later specialize this analysis to a fixed choice of $\sigma^{-1}$ corresponding to our theoretical lower bound, in order to illustrate what the asymptotic theory suggests at that noise level.

\begin{lemma}[Asymptotic $\mu$-GDP separation for Poisson Subsampling]
\label{lem:asymptotic-muGDP-separation}
Fix $\sigma^{-1}>0$ and let $(E_M)_{M\ge 1}$ be a sequence of epoch budgets,
where $E_M$ denotes the total number of \emph{epochs} performed when the
effective dataset size parameter is $M$.
Assume that this sequence satisfies the scaling
\begin{equation}
\sqrt{\frac{E_M}{M}} \;\longrightarrow\; c,
\label{eq:epoch-scaling}
\end{equation}
for some constant $c>0$ as $M\to\infty$. 
Let $\mu(M,E_M,\sigma^{-1})$ denote the asymptotic $\mu$-GDP parameter given
by Corollary~5.4 of~\cite{DBLP:journals/corr/abs-1905-02383} under Poisson
subsampling with noise scale $\sigma$, and define the corresponding
$\mu$-GDP--predicted separation
\[
\kappa_{\mu\text{-GDP}}(M,E_M,\sigma^{-1})
:=
\operatorname{sep}\!\left(G_{\mu(M,E_M,\sigma^{-1})}\right).
\]
Then, along any sequence $(E_M)_{M\ge 1}$ satisfying~\eqref{eq:epoch-scaling},
the $\mu$-GDP--predicted separation converges as
\begin{equation}
\kappa_{\mu\text{-GDP}}(M,E_M,\sigma^{-1})
\;\longrightarrow\;
\operatorname{sep}\!\left(G_{\mu_\infty}\right),
\label{eq:kappa-muGDP-limit}
\end{equation}
where
\begin{equation}
\mu_\infty
=
\sqrt{2}\,c\,
\sqrt{
e^{\sigma^{-2}}\Phi(1.5\sigma^{-1})
+3\Phi(-0.5\sigma^{-1})
-2
}.
\label{eq:mu-limit}
\end{equation}
Moreover, for each $M$,
\begin{equation}
\kappa_{\mu\text{-GDP}}(M,E_M,\sigma^{-1})
=
\frac{2\Phi\!\left(\mu(M,E_M,\sigma^{-1})/2\right)-1}{\sqrt{2}}.
\label{eq:kappa-muGDP-explicit}
\end{equation}

Finally, for any $\mu>0$, the separation of a $\mu$-GDP mechanism satisfies
the tail bound
\begin{equation}
\operatorname{sep}(G_\mu)
\;\ge\;
\frac{1}{\sqrt{2}}
-
\frac{2}{\sqrt{\pi}}\cdot \frac{\exp(-\mu^2/8)}{\mu}.
\label{eq:kappa-tail-one-shot}
\end{equation}
In particular, substituting $\mu=\mu(M,E_M,\sigma^{-1})$ yields an explicit
lower bound on the limiting $\mu$-GDP separation.
\end{lemma}
\begin{proof}

Corollary~5.4 of~\cite{DBLP:journals/corr/abs-1905-02383} establishes the privacy guarantees for \emph{NoisySGD}, which denotes the standard stochastic gradient descent algorithm with gradient clipping and Gaussian noise injection. In the context of our work, NoisySGD corresponds to the DP-SGD mechanism (Algorithm~\ref{alg:batchsgd}) when instantiated with Poisson subsampling.

Specifically, the result states that if $m\sqrt{T}/N$ converges to a finite constant $c$, where $m$ is the expected batch size and $T$ is the total number of gradient steps, then the mechanism is asymptotically $\mu$-GDP with parameter
\begin{equation}
\label{eq:mugdp-full}
\mu
=
\sqrt{2}\,\frac{m\sqrt{T}}{N}\,
\sqrt{
e^{\sigma^{-2}}\Phi(1.5\sigma^{-1})
+3\Phi(-0.5\sigma^{-1})
-2
}.
\end{equation}
Using the epoch reparameterization
$E_M := Tm/N$ and $M := N/m$, we obtain
\[
\frac{m\sqrt{T}}{N}=\sqrt{\frac{E_M}{M}}.
\]
Hence, under the scaling assumption~\eqref{eq:epoch-scaling},
$\mu(M,E_M,\sigma^{-1})$ converges to the limit $\mu_\infty$ given
in~\eqref{eq:mu-limit}.

By Lemma~\ref{lem:kappa-of-mu}, the separation induced by $\mu$-GDP is obtained
by evaluating the Gaussian trade-off curve at its unique fixed point, yielding
$\operatorname{sep}(G_\mu)=\frac{2\Phi(\mu/2)-1}{\sqrt{2}}$.
Applying this expression to $\mu(M,E_M,\sigma^{-1})$ and assuming
$M\to\infty$ gives~\eqref{eq:kappa-muGDP-limit}.
For the tail bound, Lemma~\ref{lem:kappa-of-mu} implies
\[
\frac{1}{\sqrt2}-\operatorname{sep}(G_\mu)=\sqrt2\,(1-\Phi(\mu/2)).
\]
Applying the Gaussian tail inequality $1-\Phi(t)\le \varphi(t)/t$ for $t>0$ at
$t=\mu/2$ yields
\[
\frac{1}{\sqrt2}-\operatorname{sep}(G_\mu)
\le
\sqrt2\cdot \frac{\varphi(\mu/2)}{\mu/2}.
\]
Since $\varphi(\mu/2)=\frac{1}{\sqrt{2\pi}}\exp(-\mu^2/8)$, rearranging gives
\eqref{eq:kappa-tail-one-shot}.
\end{proof}

\paragraph{Interpretation.}
Lemma~\ref{lem:asymptotic-muGDP-separation} characterizes the behavior of the
$\mu$-GDP approximation under a specific joint asymptotic regime in which the
epoch budget and the number of rounds grow proportionally.
Fix $\sigma^{-1}>0$, and let $(E_M)_{M\ge 1}$ be any sequence of epoch budgets such
that
\[
\sqrt{\frac{E_M}{M}} \;\longrightarrow\; c
\qquad\text{as } M\to\infty,
\]
for some constant $c>0$.
Define the $\mu$-GDP–predicted separation
\[
\kappa_{\mu\text{-GDP}}(M,E_M,\sigma^{-1})
:= \operatorname{sep}\!\left(G_{\mu(M,E_M,\sigma^{-1})}\right),
\]
and write
\[
\varepsilon(M,E_M,\sigma^{-1})
:= \bigl|\kappa_{\mu\text{-GDP}}(M,E_M,\sigma^{-1})
- \operatorname{sep}\!\left(G_{\mu_\infty}\right)
\bigr|
\]
for the approximation error relative to the limiting GDP separation.
Then Lemma~\ref{lem:asymptotic-muGDP-separation} implies that
\[
\varepsilon(M,E_M,\sigma^{-1}) \;\longrightarrow\; 0
\qquad\text{as } M\to\infty,
\]
whenever $\sqrt{E_M/M}\to c$.

In particular, fix a reference pair $(M^\star,E^\star)$ and define
$c := \sqrt{E^\star/M^\star}$.
For any sequence $(E_M)_{M\ge 1}$ satisfying $E_{M^\star}=E^\star$ and
$\sqrt{E_M/M}\to c$, the quantity $\sqrt{E_M/M}$ necessarily oscillates around
$c$, with oscillations vanishing as $M$ grows.
Up to this point, the statement is fully rigorous and follows directly from the
asymptotic characterization of $\mu(M,E,\sigma^{-1})$.

At a heuristic level, this picture suggests the following intuition.
If $M^\star$ is sufficiently large, then the regime $(M^\star,E^\star)$ already
lies close to the asymptotic scaling curve $\sqrt{E_M/M}=c$.
Consequently, one expects that the $\mu$-GDP prediction
$\kappa_{\mu\text{-GDP}}(M^\star,E^\star,\sigma^{-1})$ is already close to its
asymptotic value, and hence that $\varepsilon(M^\star,E^\star,\sigma^{-1})$ is small. This intuition, however, does not by itself constitute a formal guarantee. The $f$-DP framework establishes asymptotic convergence to a Gaussian trade-off curve, but does not provide convergence rate results. As a result, closeness to the asymptotic regime alone cannot be translated into a formal privacy guarantee.

\subsection{Illustrative Asymptotic Instantiation}

\begin{figure}[ht]
    \centering
    \includegraphics[width=\linewidth]{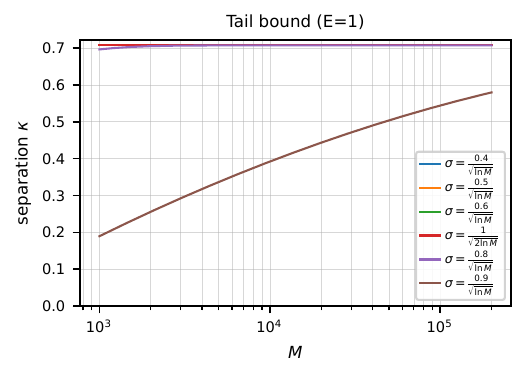}
    \caption{
    Explicit lower bound on the separation $\operatorname{sep}\!\left(G_{\mu(M,1,\sigma^{-1})}\right)$ as a function of the number of rounds per epoch $M$ under noise schedules of the form $\sigma = s/\sqrt{\ln M}$, with $E=1$.
    }
    \label{fig:kappa-tailLB-vs-M}
\end{figure}
\begin{figure}[ht]
    \centering
    \includegraphics[width=\linewidth]{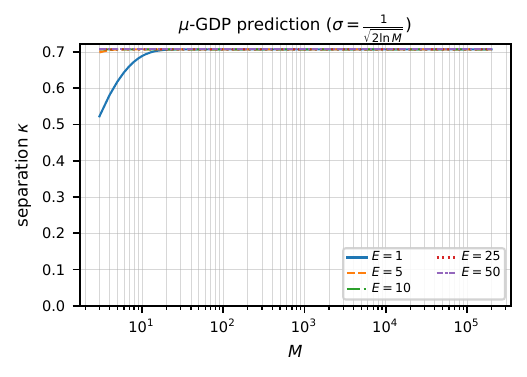}
    \caption{
    GDP-predicted separation $\kappa_{\mu\text{-GDP}}(M,E,\sigma^{-1})$ as a function of the number of rounds per epoch $M$
    under the noise schedule $\sigma = 1/\sqrt{2\ln M}$, for several epoch counts $E$.
    }
    \label{fig:kappa-gdp-vs-M}
\end{figure}

We now carry out a simple asymptotic computation, to see what the $\mu$-GDP
picture behind Lemma~\ref{lem:asymptotic-muGDP-separation} suggests when the noise level is set
according to our theoretical lower bound, derived later in Theorem~\ref{thm:shuffling}.
Our setting departs from the lemma's in two ways. We fix the one-epoch setting $E=1$, so that
$\sqrt{E_M/M}=1/\sqrt{M}\to 0$, whereas the lemma (and Corollary~5.4
of~\cite{DBLP:journals/corr/abs-1905-02383} behind it) takes this ratio to a constant $c>0$;
and we let $\sigma$ vary with $M$, where the lemma holds it fixed as $M\to\infty$. The
computation below is therefore indicative of the scaling, and we do not present it as a formal
consequence of the asymptotic theory.
Let
\[
\sigma \;=\; \frac{s}{\sqrt{\ln M}},
\qquad\text{equivalently}\qquad
\sigma^{-1} \;=\; \frac{\sqrt{\ln M}}{s},
\]
for a fixed constant $s>0$, and assume $s<1$, so that the exponent
$\tfrac{1}{2s^2}-\tfrac12$ in~\eqref{eq:sep-explicit-exercise-clean} is positive. This covers the
regime of interest here, namely $s \le 1/\sqrt{2}$, i.e.\ $\sigma \le 1/\sqrt{2\ln M}$.

Substituting this expression into the asymptotic $\mu$-GDP formula of
Corollary~5.4 in~\cite{DBLP:journals/corr/abs-1905-02383}, we obtain
\begin{align}
\mu(M,1,\sigma^{-1})
&=
\sqrt{2}\,\sqrt{\frac{1}{M}}\,
\sqrt{
\exp(\sigma^{-2})\,\Phi(1.5\sigma^{-1})
+3\Phi(-0.5\sigma^{-1})
-2
}
\notag\\
&=
\sqrt{2}\,\frac{1}{\sqrt{M}}\,
\sqrt{
\exp\!\Bigl(\tfrac{\ln M}{s^2}\Bigr)\,
\Phi\!\Bigl(\tfrac{1.5\sqrt{\ln M}}{s}\Bigr)
+3\Phi\!\Bigl(-\tfrac{0.5\sqrt{\ln M}}{s}\Bigr)
-2
}.
\label{eq:mu-plugged-E1}
\end{align}

To obtain a simple and explicit lower bound, we bound the terms inside the
square root in a conservative manner.
First, since $\Phi(x)\ge \tfrac12$ for all $x\ge 0$, we have
\[
\Phi\!\Bigl(\tfrac{1.5\sqrt{\ln M}}{s}\Bigr)
\;\ge\;
\frac12.
\]
Second, since $\Phi(x)\ge 0$ for all $x\in\mathbb{R}$, we may drop the negative-tail
term entirely:
\[
3\Phi\!\Bigl(-\tfrac{0.5\sqrt{\ln M}}{s}\Bigr) \;\ge\; 0.
\]
Applying these bounds to~\eqref{eq:mu-plugged-E1} yields
\begin{align}
\mu(M,1,\sigma^{-1})
&\;\ge\;
\sqrt{2}\,\frac{1}{\sqrt{M}}\,
\sqrt{
\frac12\,\exp\!\Bigl(\tfrac{\ln M}{s^2}\Bigr)
-2
}
\notag\\
&=
\sqrt{2}\,\frac{1}{\sqrt{M}}\,
\sqrt{
\frac12\,M^{1/s^2}-2
}
\;=\;
\frac{1}{\sqrt{M}}\,
\sqrt{M^{1/s^2}-4},
\label{eq:mu-lower-E1}
\end{align}
where the last equality holds whenever $M^{1/s^2}\ge 4$.

Finally, we simplify the expression further by observing that for sufficiently
large $M$,
\[
M^{1/s^2}-4 \;\ge\; \tfrac12\,M^{1/s^2}.
\]
Applying this bound to~\eqref{eq:mu-lower-E1} yields
\[
\mu(M,1,\sigma^{-1})
\;\ge\;
\frac{1}{\sqrt{M}}\,
\sqrt{\tfrac12\,M^{1/s^2}}
\;=\;
\frac{1}{\sqrt{2}}\,
M^{\frac{1}{2s^2}-\frac12}.
\]

Substituting this expression into the tail bound
\eqref{eq:kappa-tail-one-shot} gives the fully explicit separation bound
\begin{equation}
\operatorname{sep}\!\left(G_{\mu(M,1,\sigma^{-1})}\right)
\;\ge\;
\frac{1}{\sqrt{2}}
-
\frac{2}{\sqrt{\pi}}\,
\frac{
\exp\!\Bigl(
-\frac{1}{16}\,M^{\frac{1}{s^2}-1}
\Bigr)
}{
\frac{1}{\sqrt{2}}\,M^{\frac{1}{2s^2}-\frac12}
}.
\label{eq:sep-explicit-exercise-clean}
\end{equation}

The computation shows that, when the noise multiplier is set at the level
$\sigma = s/\sqrt{\ln M}$ with $s<1$ (and $E=1$), the $\mu$-GDP expression extrapolates to a
separation that increases with $M$ and approaches its maximal value $1/\sqrt{2}$. This remains a
heuristic reading of the formula outside the regime in which the asymptotic
theory is stated, and the proofs of
Theorems~\ref{thm:shuffling} and~\ref{thm:poisson} do not rely on it.
Figures~\ref{fig:kappa-tailLB-vs-M} and~\ref{fig:kappa-gdp-vs-M} visualize the implications
of Lemma~\ref{lem:asymptotic-muGDP-separation} and its asymptotic instantiation for the
separation metric as the number of rounds per epoch $M$ increases.
Figure~\ref{fig:kappa-tailLB-vs-M} plots the fully explicit lower bound obtained by
combining the Gaussian tail bound in Eq.~\eqref{eq:kappa-tail-one-shot} with the explicit
lower bound on the $\mu$-GDP parameter derived in
Eqs.~\eqref{eq:mu-lower-E1}–\eqref{eq:sep-explicit-exercise-clean} under the noise schedule
$\sigma = s/\sqrt{\ln M}$ (with $E=1$).
Figure~\ref{fig:kappa-gdp-vs-M} shows the corresponding $\mu$-GDP–predicted separation
$\kappa_{\mu\text{-GDP}}(M,E,\sigma^{-1})$ defined in
Eq.~\eqref{eq:kappa-muGDP-explicit}, where $\mu(M,E,\sigma^{-1})$ is given by the asymptotic
characterization in Eq.~\eqref{eq:mugdp-full} with epoch scaling~\eqref{eq:epoch-scaling}.
Both figures demonstrate that, under noise schedules of order
$\sigma=\Theta(1/\sqrt{\ln M})$, the separation increases rapidly with $M$ and approaches
its maximal value $1/\sqrt{2}$.

We emphasize that, unlike the asymptotic $\mu$-GDP approximation discussed above,
our main separation bound is non-asymptotic. In particular, it holds for every finite $M$ in the single-epoch setting, without requiring any limiting regime or asymptotic approximation.

\section{Origin of the \texorpdfstring{$\sqrt{2\ln M}$}{sqrt(2 ln M)} Scale}
\label{app:sqrt-2lnM}

The noise floor $\sigma \ge 1/\sqrt{2\ln M}$ appearing in Theorem~\ref{thm:shuffling} comes
from the observation model itself, and it survives any change in how privacy is measured.
We make that precise below.

The threshold $1/\sqrt{2\ln M}$ reflects Gaussian extreme-value behavior. For
$g_1,\dots,g_M$ i.i.d.\ standard normal, the sub-Gaussian maximal inequality gives
$\mathbb{E}[\max_i g_i] \le \sqrt{2\ln M}$ for every $M$~\cite[Thm.~2.5]{blm2013}, and this
is asymptotically sharp, since $\mathbb{E}[\max_i g_i] = \sqrt{2\ln M}\,(1+o(1))$ as
$M \to \infty$~\cite[Rem.~2.7.7 and Ex.~2.38]{vershynin2026hdp}. The maximum also stays within $O(1)$ of its
mean by Gaussian concentration~\cite[Thm.~5.5]{blm2013}. A single coordinate
shifted by $1/\sigma$ can therefore remain hidden among $M$ unshifted coordinates only if
$1/\sigma \lesssim \sqrt{2\ln M}$. In particular, the lower bound is \emph{not} an artifact
of the separation metric. It is a structural property of the underlying product
distributions, and $\kappa$ merely records it geometrically. Our contribution is to make
this limitation explicit and provable for DP-SGD, via the max-statistic converse of
Section~\ref{subsec:suboptimal-test}, the resulting separation lower bound that holds
uniformly over all $f$-DP families, and the mixture reduction that transfers the bound to
Poisson subsampling (Section~\ref{subsec:poisson-mixture}).

\else

\fi
\end{document}